%% file: main.tex
\def\year{2022}\relax
\documentclass[letterpaper]{article} 
\usepackage{aaai22}  
\usepackage{times}  
\usepackage{helvet}  
\usepackage{courier}  
\usepackage[hyphens]{url}  
\usepackage{graphicx} 
\usepackage{natbib}  
\usepackage{caption} 
\DeclareCaptionStyle{ruled}{labelfont=normalfont,labelsep=colon,strut=off} 
\usepackage{amsmath,amsfonts,amsthm}
\usepackage[linesnumbered,ruled]{algorithm2e}
\usepackage{array}
\usepackage{multirow}
\newtheorem{theorem}{Theorem}
\newtheorem{lemma}{Lemma}

\usepackage{newfloat}
\usepackage{listings}
\title{Partial Model Averaging in Federated Learning: Performance Guarantees and Benefits}
\author{
    Sunwoo Lee\textsuperscript{\rm 1}, Anit Kumar Sahu\textsuperscript{\rm 2}, Chaoyang He\textsuperscript{\rm 1}, Salman Avestimehr\textsuperscript{\rm 1}
}
\affiliations{
    \textsuperscript{\rm 1}University of Southern California\\
    \textsuperscript{\rm 2}Amazon Alexa AI\\

    \textsuperscript{\rm 1}\{sunwool,chaoyang.he,avestime\}@usc.edu\\
    \textsuperscript{\rm 2}anitsah@amazon.com
}

\usepackage{bibentry}

\begin{document}

\maketitle

\begin{abstract}
  Local Stochastic Gradient Descent (SGD) with periodic model averaging (FedAvg) is a foundational algorithm in Federated Learning. The algorithm independently runs SGD on multiple workers and periodically averages the model across all the workers. When local SGD runs with many workers, however, the periodic averaging causes a significant model discrepancy across the workers making the global loss converge slowly. While recent advanced optimization methods tackle the issue focused on non-IID settings, there still exists the model discrepancy issue due to the underlying periodic model averaging. We propose a partial model averaging framework that mitigates the model discrepancy issue in Federated Learning. The partial averaging encourages the local models to stay close to each other on parameter space, and it enables to more effectively minimize the global loss. Given a fixed number of iterations and a large number of workers (128), the partial averaging achieves up to $2.2\%$ higher validation accuracy than the periodic full averaging.
\end{abstract}

\input{1_intro}
\input{2_background}
\input{3_design}
\input{4_analysis}
\input{5_experiments}

\input{6_related}

\input{7_conclusion}

\vspace{.2em}

\bibliography{mybib.bib}

\appendix
\onecolumn
\input{supplement_1_preliminary}
\input{supplement_2_iid}
\input{supplement_3_noniid}
\input{supplement_4_exp}

\end{document}

%% file: 1_intro.tex
\section {Introduction} \label{sec:intro}

Local Stochastic Gradient Descent (Local SGD) with periodic model averaging has been recently shown to be a promising alternative to vanilla synchronous SGD \cite{robbins1951stochastic}. The algorithm runs SGD on multiple workers independently and averages the model parameters across all the workers periodically. FedAvg \cite{mcmahan2017communication} is built around local SGD and has been shown to be effective in Federated Learning to solve problems involving non-Independent and Identically Distributed (non-IID) data. Several studies have shown that local SGD achieves linear speedup with respect to number of workers for convex and non-convex problems \cite{stich2018local,yu2019parallel,yu2019linear,wang2018cooperative,haddadpour2019local}.

While the periodic model averaging dramatically reduces the communication cost in distributed training, it causes model discrepancy across all the workers.
Due to variance of stochastic gradients and data heterogeneity, the independent local training steps can disperse the models over a wide region in the parameter space.
Averaging a large number of such different local models can significantly distract the convergence of global loss as compared to synchronous SGD that only has one global model.
The model discrepancy can adversely affect the convergence both in IID and non-IID settings.
To scale up the training to hundreds, thousands, or even millions of workers in Federated Learning, it is crucial to address this issue.

Many researchers have put much effort into addressing the model discrepancy issue in non-IID settings.
Variance Reduced Local-SGD (VRL-SGD) \cite{liang2019variance} and SCAFFOLD \cite{karimireddy2020scaffold} make use of extra control variates to accelerate the convergence by reducing variance of stochastic gradients.
FedProx \cite{li2018federated} adds a proximal term to each local loss to suppress the distance among the local models.
FedNova \cite{wang2020tackling} normalizes the magnitude of local updates across the workers so that the model averaging less distracts the global loss.
All these algorithms employ the periodic model averaging as a backbone of the model aggregation.
Thus, although they mitigate the model discrepancy caused by the data heterogeneity, the issue still exists due to the underlying periodic model averaging scheme.

Breaking the convention of periodic full model averaging, we propose a \textit{partial model averaging} framework to tackle the model discrepancy issue in Federated Learning.
Instead of allowing the workers independently update the full model parameters within each communication round, our framework synchronizes a distinct subset of the model parameters every iteration.
Such frequent synchronizations encourage all the local models to stay close to each other on parameter space, and thus the global loss is not strongly distracted when averaging many local models.
Our empirical study shows that the partial model averaging effectively suppresses the degree of model discrepancy during the training, and it results in making the global loss converge faster than the periodic averaging.
Within a fixed iteration budget, the faster convergence of the loss most likely results in achieving a higher validation accuracy in Federated Learning.
We also theoretically analyze the convergence property of the proposed algorithm for smooth and non-convex problems considering both IID and non-IID data settings.

We focus on how the partial model averaging affects the classification performance when it replaces the underlying periodic model averaging scheme in Federated Learning.
We evaluate the performance of the proposed framework across a variety of computer vision and natural language processing tasks.
Given a fixed number of iterations and a large number of workers (128), the partial averaging shows a faster convergence and achieves up to $2.2\%$ higher validation accuracy than the periodic averaging.
In addition, the partial averaging consistently accelerates the convergence across various degrees of the data heterogeneity.
These results demonstrate that the partial averaging effectively mitigates the adverse impact of the model discrepancy on the federated neural network training.
The partial averaging method has the same communication cost as the periodic averaging and does not require extra computations.

\textbf{Contributions} -- We highlight our contributions below.
\begin{enumerate}
    \item {We propose a novel partial model averaging framework for large-scale Federated Learning. The framework tackles the model discrepancy in a foundational model averaging level. Our theoretical analysis provides a convergence guarantee for non-convex problems, achieving linear speedup with respect to the number of workers.}
    \item {We explore benefits of the proposed partial averaging framework. Our empirical study demonstrates that the global loss is not strongly distracted when partially averaging the local models, which results in a faster convergence. We also report extensive experimental results across various benchmark datasets and models.}
    \item {The partial averaging framework is readily applicable to any Federated Learning algorithms. Our study introduces promising future works regarding how to harmonize the layer-wise model aggregation scheme with many Federated Learning algorithms such as FedProx, FedNova, SCAFFOLD, and adaptive averaging interval methods.}
\end{enumerate}

%% file: 2_background.tex
\section {Background} \label{sec:back}

\textbf{Local SGD with Periodic Model Averaging} --
We consider federated optimization problems of the form
\begin{align}
    \underset{\mathbf{x} \in \mathbb{R}^d}{\min}\left[F(\mathbf{x}) := \sum_{i=1}^{m} p_i F_i(\mathbf{x}) \right] \label{eq:cost},
\end{align}
where $p_i = n_i / n$ is the ratio of local data to the total dataset, and $F_i(\mathbf{x}) = \frac{1}{n_i} \sum_{\xi \in \mathcal{D}} f_i(\mathbf{x}, \xi)$
is the local objective function of client $i$.
$n$ is the global dataset size and $n_i$ is the local dataset size.


The model averaging can be expressed as follows:
\begin{equation}
\label{eq:average}
    \mathbf{u}_k = \sum_{i=1}^{m} p_i \mathbf{x}_k^i,
\end{equation}
where $m$ is the number of workers (local models), $\mathbf{x}_k^i$ is the local model of worker $i$ at iteration $k$, and $\mathbf{u}_k$ is the averaged model.
Note that, $p_i = 1 / m$ when the data is IID.
The parameter update rule of local SGD with periodic averaging (FedAvg) is as follows.
\begin{equation}
    \mathbf{x}_{k+1}^i = 
    \begin{cases}
    \sum_{i=1}^{m} p_i [\mathbf{x}_k^i - \mu g(\mathbf{x}_k^i)], & k \textrm{ mod } \tau \textrm{ is } 0\\
    \mathbf{x}_k^i -\mu g(\mathbf{x}_k^i), & \textrm{otherwise}
    \end{cases}
\end{equation}
where $\tau$ is the model averaging interval and $g(\cdot)$ is a stochastic gradient computed from a random training sample $\boldsymbol{\xi}$.
This update rule allows all the workers to independently update their own models for every $\tau$ iterations.

\textbf{Model Discrepancy} --
Assuming the local optimizers are stochastic optimization methods, the most typical training algorithm for neural network training, all $m$ local models can move toward different directions on parameter space due to the variance of the stochastic gradients.
In Federated Learning, the data heterogeneity also makes such an effect more significant.
We call the difference between the local models and the global model \textit{model discrepancy}.
If the degree of model discrepancy is large, the local models are more likely attracted to different minima adversely affecting the convergence of global loss.
Note that synchronous SGD does not have such an issue since it guarantees all the workers always view the same model parameters.

%% file: 3_design.tex
\section {Partial Model Averaging Framework}

\begin{algorithm}[t]
    \caption{Local SGD with partial model avg.}\label{alg:proposed}
    \SetKwInOut{Input}{Input}
    \Input{Initial parameters $\mathbf{x}_0$, learning rate $\eta$, and model averaging interval $\tau$}
    \For{$k=1$ to $K$}{
        A local SGD step: $\mathbf{x}_{k}^i=\mathbf{x}_{k-1}^i-\mu g(\mathbf{x}_{k-1}^i)$\;
        $j \leftarrow k \text{ mod } \tau$\;
        Average $j^{th}$ subset of the model across all $m$ workers:
        $\mathbf{u}_{(j,k)} = \frac{1}{m}\sum_{i=1}^{m} \mathbf{x}_{(j,k)}^i$\;
        Each worker updates $j^{th}$ subset of the local model: $\mathbf{x}_{(j,k)}^i = \mathbf{u}_{(j,k)}$ \;
    }
    Return $\mathbf{u}_K = \frac{1}{m}\sum_{i=1}^m \mathbf{x}_K^i$
\end{algorithm}


Algorithm \ref{alg:proposed} presents local SGD with partial model averaging.
Each worker independently runs SGD until the stop condition is satisfied ($K$ iterations).
After every SGD step, the algorithm averages a distinct subset of model parameters across all $m$ workers.
Each subset consists of $\frac{d}{\tau}$ parameters, where $d$ is the total number of model parameters and $\tau$ is the model averaging interval.
In this setting, each subset is averaged after every $\tau$ iterations.
At the end of the training, Algorithm \ref{alg:proposed} returns the fully-averaged model $\mathbf{u}_K$.

Figure \ref{fig:schematic} shows schematic illustrations of the periodic averaging (\textbf{a}) and the partial averaging (\textbf{b}).
They show the expected movement of two local models on the parameter space within one communication round ($\tau=3$).
While the periodic averaging allows fully-independent local updates, the partial averaging frequently synchronizes a part of the model parameters suppressing the model discrepancy.

In this work, we use mini-batch SGD as a local solver for simplicity.
The framework can be applied to any advanced optimizers by simply changing the parameter update rule at line 2.
For instance, FedProx~\cite{li2018federated} can be applied by replacing the $g(\mathbf{x}_k^i)$ term with the gradient computed from the FedProx loss function.

Note that Algorithm \ref{alg:proposed} does not specify how to partition the model parameters to $\tau$ subsets.
As long as the entire parameters are synchronized at least once within $\tau$ iterations, it is theoretically guaranteed to have the same maximum bound of the convergence rate.
We discuss the impact of the model partitioning on the training results in Appendix.

\begin{figure}[t]
\centering
\includegraphics[width=\columnwidth]{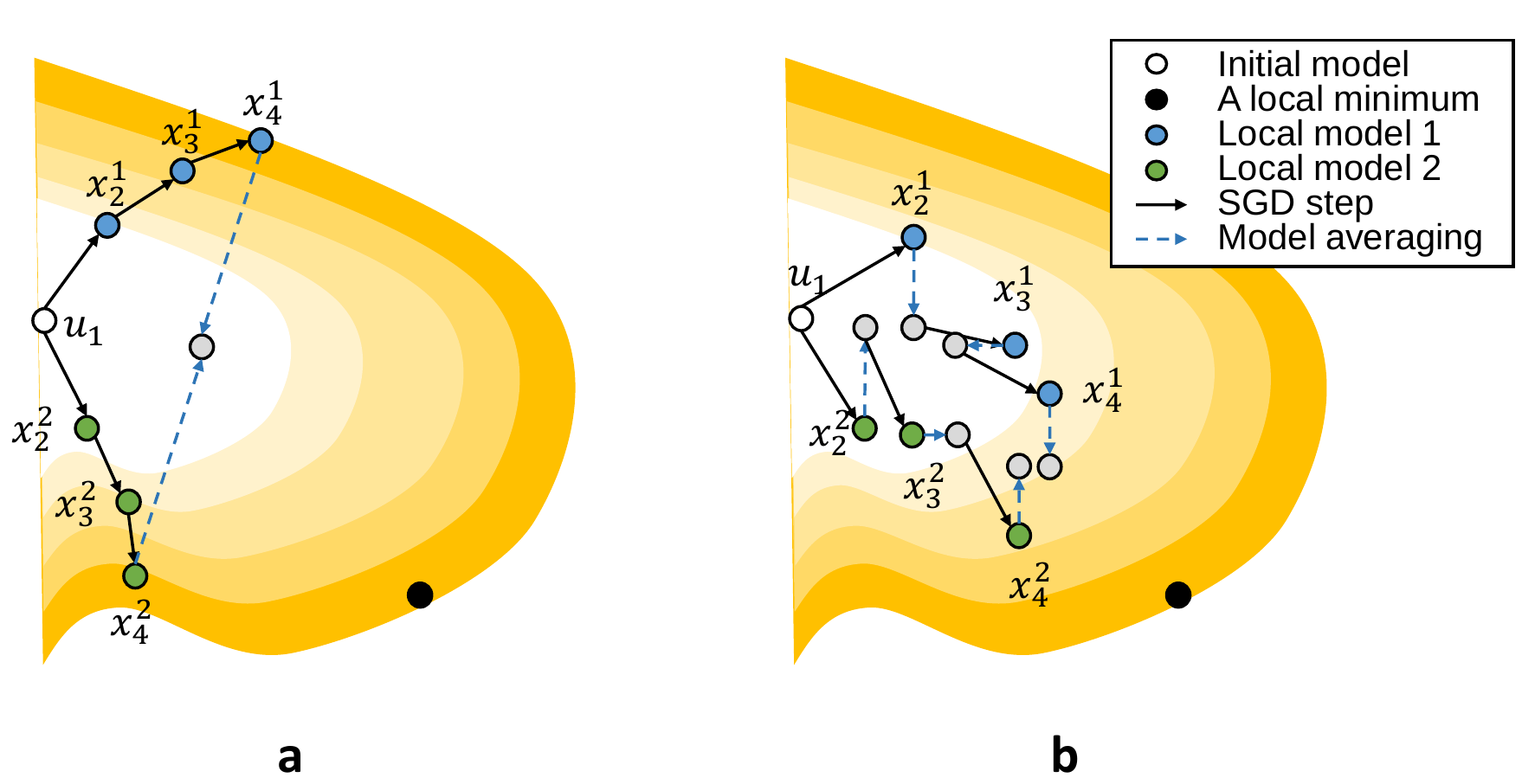}
\caption{
    Example illustrations of \textbf{a)}: periodic averaging and \textbf{b)}: partial averaging with two workers ($\tau=3$).
    While the periodic averaging allows fully-independent local updates, the partial averaging frequently synchronize a part of model suppressing the model discrepancy.
}
\label{fig:schematic}
\end{figure}

%% file: 4_analysis.tex
\section {Convergence Analysis}

\subsection {Preliminaries}

\textbf{Notations} -- 
To consider the model partitions in the convergence analysis, we borrow partition-wise notations and assumptions from \cite{you2019large}.
All vectors in this paper are column vectors.
$\mathbf{x} \in \mathbb{R}^d$ denotes the parameters of one local model and $m$ is the number of workers.
The model is partitioned into $\tau$ subsets such that $\mathbf{x}_{j} \in \mathbb{R}^{d_j}$ for $j \in \{ 1, \cdots, \tau \}$, where $\sum_{j=1}^{\tau} d_j = d$.
We use $g_j(\mathbf{x, \xi})$ to denote the gradient of $f(\cdot)$ with respect to $\mathbf{x}_{j}$, where $\xi$ is a single training sample.
For convenience, we use $g_j(\mathbf{x})$ instead.
The gradient computed from the whole training samples with respect to $\mathbf{x}_{j}$ is denoted by $\nabla_j F(\mathbf{x})$.
$L_j$ is Lipschitz constant of $f(\cdot)$ with respect to $\mathbf{x}_j$.
$L_{\max}$ indicates the maximum Lipschitz constant among all $\tau$ model partitions: $\textrm{max}(L_j), j \in \{1, \cdots, \tau \}$.
Likewise, $\sigma^2 = \sum_{j = 1}^{\tau} \sigma_j^2$.
We provide all the proofs in Appendix.

\subsection {Convergence Analysis for IID Data}

\textbf{Assumptions} --
We analyze the convergence rate of Algorithm \ref{alg:proposed} under the following assumptions.
\begin{enumerate}
    \item Smoothness: $f(\cdot)$ is $L_j$-smooth for all $\mathbf{x}_j$;
    \item Unbiased gradient: $\mathop{\mathbb{E}}_{\xi}[g_j(x)] = \nabla_j F(x)$;
    \item Bounded variance: $\mathop{\mathbb{E}}_{\xi}[ \| g_j(x) - \nabla_j F(x)) \|^2 ] \leq \sigma_j^2$, where $\sigma_j^2$ is a positive constant;
\end{enumerate}

\begin{theorem}
\label{theorem:iid}
Suppose all $m$ local models are initialized to the same point $\mathbf{u}_1$. Under Assumption $1 \sim 3$, if Algorithm \ref{alg:proposed} runs for $K$ iterations using the learning rate $\eta$ that satisfies $L_{\max}^2 \eta^2 \tau (\tau - 1) + \eta L_{\max} \leq 1$, then the average-squared gradient norm of $\mathbf{u}_k$ is bounded as follows
\begin{align}
    \mathop{\mathbb{E}}\left[\frac{1}{K}\sum_{i=1}^{K} \| \nabla F(\mathbf{u}_k) \|^2\right] & \leq \frac{2}{\eta K} \mathop{\mathbb{E}} \left[ F(\mathbf{u}_1) - F(\mathbf{u}_{inf}) \right] \nonumber \\
    & \quad\quad + \frac{\eta}{m} \sum_{j = 1}^{\tau} L_j \sigma_j^2 \label{eq:theorem1} \\
    & \quad\quad + \eta^2 (\tau - 1) \sum_{j = 1}^{\tau} L_j^2 \sigma_j^2 \nonumber
\end{align}
\end{theorem}

\textbf{Remark 1.} For non-convex smooth objective functions and IID data, local SGD with the partial model averaging ensures the convergence of the model to a stationary point.
Particularly, the convergence rate is not dependent on the partition size or the synchronization order across the partitions. 
That is, as long as the entire model parameters are covered at least once in $\tau$ iterations, Algorithm \ref{alg:proposed} guarantees the convergence. 

\textbf{Remark 2.}
(linear speedup) If the learning rate $\eta = \frac{\sqrt{m}}{\sqrt{K}}$, the complexity of (\ref{eq:theorem1}) becomes
\begin{align}
    \mathcal{O} \left( \frac{1}{\sqrt{mK}} \right) + \mathcal{O} \left( \frac{m}{K} \right) \nonumber,
\end{align}
where all the constants are removed by $\mathcal{O}$.
Thus, if $K > m^3$, the first term dominates the second term achieving linear speedup.
Note that the partial averaging has the same complexity of the convergence rate as the periodic averaging method \cite{wang2018cooperative}.

\subsection {Convergence Analysis for Non-IID Data}
For non-IID convergence analysis, we use an assumption on the data heterogeneity that is presented in \cite{wang2020tackling}.

\textbf{Assumptions} -- Our analysis is based on the following assumptions.
\begin{enumerate}
    \item Smoothness: $f(\cdot)$ is $L_j$-smooth for all $\mathbf{x}_j$;
    \item Unbiased gradient: $\mathop{\mathbb{E}}_{s}[g_{(i,j)}(x)] = \nabla_j F_i(x)$;
    \item Bounded variance: $\mathop{\mathbb{E}}_{s}[ \| g_{(i,j)}(x) - \nabla_j F_i(x)) \|^2 ] \leq \sigma_j^2$, where $\sigma_j^2$ is a positive constant;
    \item {Bounded Dissimilarity: For any sets of weights $\{p_i \geq 0\}_{i=1}^{m}, \sum_{i=1}^{m} p_i = 1$, there exist constants $\beta^2 \geq 1$ and $\kappa^2 \geq 0$ such that $\sum_{i=1}^{m} p_i \| \nabla F_i(\mathbf{x}) \|^2 \leq \beta^2 \| \sum_{i=1}^{m} p_i \nabla F_i (\mathbf{x}) \|^2 + \kappa^2$;}
\end{enumerate}

\begin{theorem}
\label{theorem:niid}
Suppose all $m$ local models are initialized to the same point $\mathbf{u}_1$. Under Assumption $1 \sim 4$, if Algorithm \ref{alg:proposed} runs for $K$ iterations and the learning rate satisfies $\eta \leq \frac{1}{L_{\max} } \min \left\{\frac{1}{2}, \frac{1}{ \sqrt{2\tau (\tau - 1) (2\beta^2 + 1)}} \right\}$, the average-squared gradient norm of $\mathbf{u}_k$ is bounded as follows
\begin{align}
    \mathop{\mathbb{E}}\left[\frac{1}{K}\sum_{i=1}^{K} \| \nabla F(\mathbf{u}_k) \|^2\right] & \leq \frac{4}{\eta K}\left( \mathop{\mathbb{E}}\left[ F(\mathbf{u}_{1}) - F(\mathbf{u}_{inf}) \right] \right) \nonumber \\
    & \quad\quad + 4\eta \sum_{i=1}^{m} p_i^2 \sum_{j = 1}^{\tau} L_j \sigma_j^2 \nonumber \\
    & \quad\quad + 3 \eta^2 (\tau - 1) \sum_{j = 1}^{\tau} L_j^2 \sigma_j^2 \nonumber \\
    & \quad\quad + 6 \eta^2 \tau (\tau - 1) \sum_{j = 1}^{\tau} L_j^2 \kappa_j^2. \nonumber
\end{align}
\end{theorem}

\textbf{Remark 3.} For non-convex smooth objective functions and non-IID data, local SGD with the partial model averaging ensures the convergence to a stationary point.
Likely to IID data, the partition size or the synchronization order across the partitions do not affect the bound. 

\textbf{Remark 4.}
(linear speedup) If the learning rate $\eta = \frac{\sqrt{m}}{\sqrt{K}}$ and $p_i = \frac{1}{m}, \forall i \in \{1, \cdots, m\}$, the complexity of the above maximum bound becomes
 \begin{align}
     \mathcal{O} \left( \frac{1}{\sqrt{mK}} \right) + \mathcal{O} \left( \frac{m}{K} \right), \nonumber
 \end{align}
 where all the constants are removed by $\mathcal{O}$.
 Thus, if $K^3 > m$, the first first term becomes dominant, and it achieves linear speedup.
 Although the exact bounds cannot be directly compared due to the different assumptions, our analysis shows that the partial averaging method has the same complexity of the convergence rate as the periodic averaging method \cite{wang2020tackling}.

\subsection{Impact of Partial Model Averaging on Local Models}

We empirically analyze the impact of the partial averaging on the statistical efficiency of local SGD.
Figure \ref{fig:disc} shows the squared distance between the global model $\mathbf{u}_k$ and the local model $\mathbf{x}_k^i$ averaged across all $m$ workers.
The distance is collected from CIFAR-10 (ResNet20) training with $m=128$ workers.
The left chart shows the distance comparison between the periodic averaging and the partial averaging at the first $500$ iterations and the right chart shows the comparison in the middle of training (iteration $3000 \sim 3500$).
It is clearly observed that the partial averaging effectively suppresses the maximum degree of model discrepancy.
While the periodic averaging has a wide spectrum of the distance within each communication round, the partial averaging shows a stable distance across the iterations.

When analyzing the convergence stochastic optimization methods, the difference between the local gradients and the global gradients is usually bounded by the distance between the corresponding model parameters under a smoothness assumption on the objective function.
The shorter distance among the models bounds the gradient difference more tightly, and it makes the loss more efficiently converge.
We verify such an effect by comparing the local loss and the global loss curves.
We collect full-batch training loss of all individual local models (local loss) and compare it to the loss of the global model (global loss) at the end of each communication round.
Figure \ref{fig:lossdiff}.\textbf{a} and \ref{fig:lossdiff}.\textbf{b} show the loss curves of the periodic averaging and the partial averaging, respectively.
While the periodic averaging makes the global loss frequently spikes, the partial averaging shows the global loss that goes down more smoothly along with the minimum local loss.
Figure \ref{fig:loss} shows the loss curves of four different datasets.
The partial averaging achieves a faster convergence than the periodic averaging in all the experiments.
This empirical analysis demonstrates that the partial averaging accelerates the convergence of the global loss by mitigating the degree of model discrepancy.

\begin{figure}[t]
\centering
\includegraphics[width=\columnwidth]{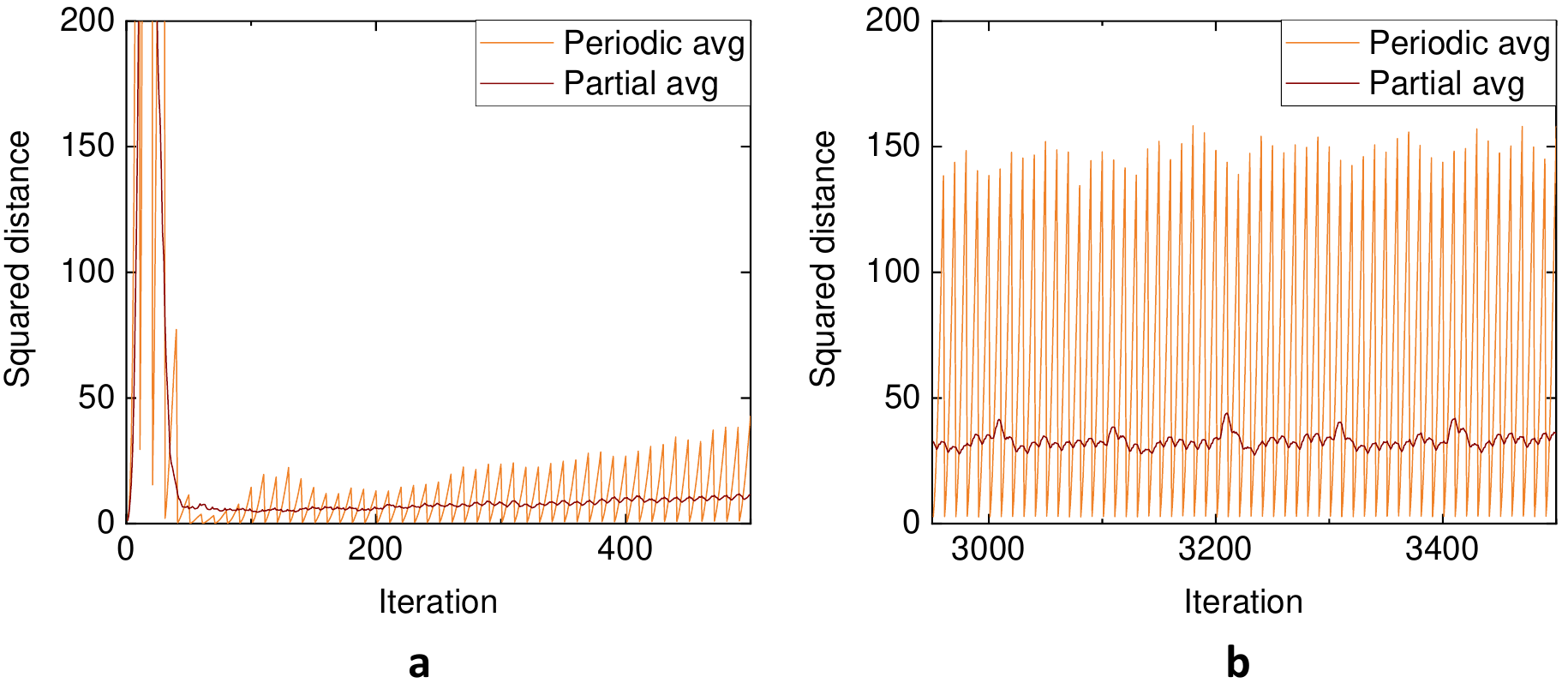}
\caption{
    The model discrepancy (the squared distance between the global model $\textbf{u}_k$ and the local model $\textbf{x}_k^i$ averaged across all $m$ workers) comparison between the periodic averaging and the partial averaging.
    The curves are collected from ResNet20 (CIFAR-10) training. \textbf{a)}: the curves for the first 500 iterations.
    \textbf{b)}: the curves for the iteration $3000 \sim 3500$.
}
\label{fig:disc}
\end{figure}

\begin{figure}[t]
\centering
\includegraphics[width=0.95\columnwidth]{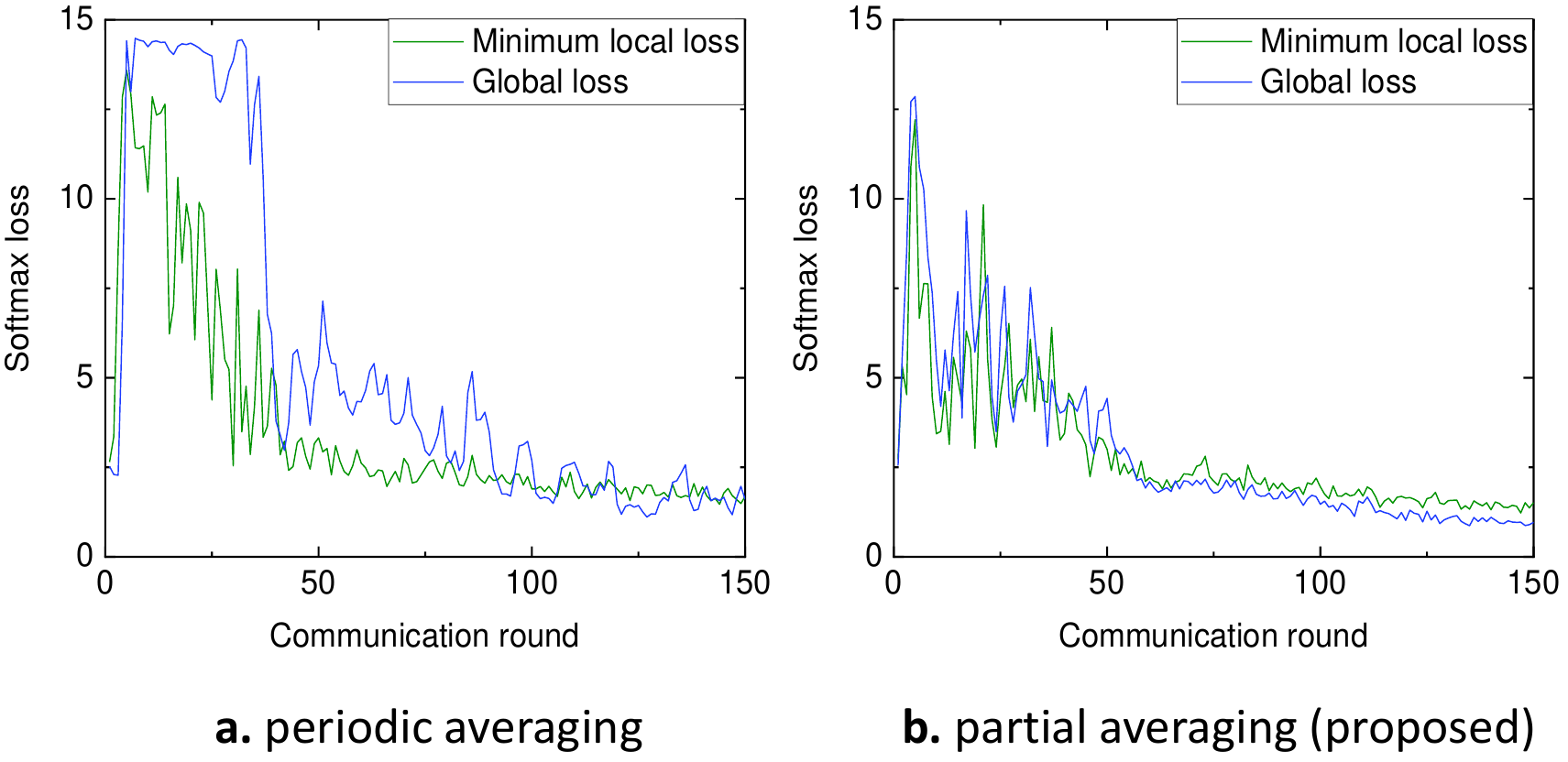}
\caption{
    The comparison between the minimum local loss among all the workers and the global loss.
    The curves are collected from ResNet20 (CIFAR-10) training for 150 communication rounds. \textbf{a)}: The full-batch training loss curves of the periodic averaging. We compare the minimum local loss and the global loss curves.
    \textbf{b)}: The same curves of the partial averaging.
}
\label{fig:lossdiff}
\end{figure}

\subsection {Communication Cost} \label{sec:comm}

For large-scale deep learning applications on High-Performance Computing (HPC) systems, the fully-distributed communication model is typically used.
The most popular communication pattern for model averaging is \textit{allreduce} operation.
In Federated Learning, server-client communication model is more commonly used.
Considering the independent and heterogeneous client-side compute nodes, the individual communication pattern (\textit{send} and \textit{receive} operations) can be a better fit for model averaging.
Regardless of the communication patterns, the periodic averaging and the partial averaging have the same total communication cost.
Given the local model size $d$, the periodic averaging method requires one inter-process communication of all the $d$ parameters after every $\tau$ iterations.
The proposed partial averaging performs one communication at every iteration, but only aggregates $\frac{d}{\tau}$ parameters at once.
Thus, if $\tau$ is the same, the two averaging methods have the same total communication cost.

One potential drawback of the proposed method is the increased number of inter-process communications.
While having the same total data size to be transferred, the partial averaging method requires more frequent communications than the periodic full averaging method, and it results in increasing the total latency cost.
One may consider adjusting the number of model partitions to reduce the latency cost while degrading the expected classification performance.
We consider making a practical trade-off between the latency cost and the statistical efficiency as an important future work.

\begin{figure*}[t]
\centering
\includegraphics[width=2\columnwidth]{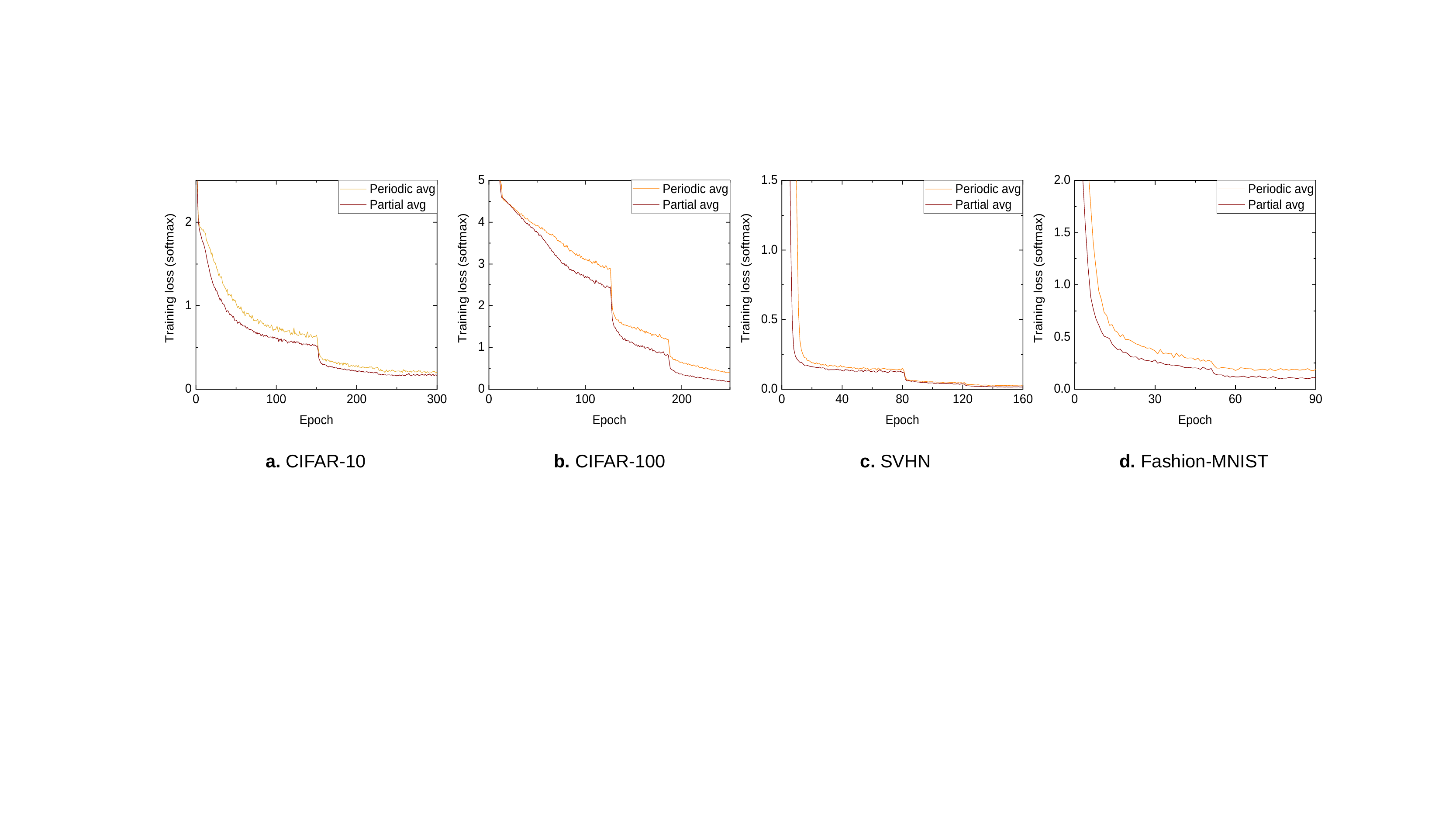}
\caption{
    (IID data) The training loss comparison between the periodic averaging and the partial averaging across four different datasets: \textbf{a.} CIFAR-10 (ResNet-20), \textbf{b.} CIFAR-100 (WideResNet-28-10), \textbf{c.} SVHN (WideResNet-16-8), and \textbf{d.} Fashion-MNIST (VGG-11).
    128 workers are used for training.
}
\label{fig:loss}
\end{figure*}

%% file: 5_experiments.tex
\section {Experiments}\label{sec:exp}

In this section, we present key experimental results that demonstrate efficacy of the partial averaging framework for Federated Learning.
Additional experimental results can be found in Appendix.

\begin{table*}[t]
\scriptsize
\centering
\caption{
    The classification performance comparisons using IID data.
    The learning rate is fine-tuned based on a grid search for all individual settings.
    \vspace{0.1cm}
}
\begin{tabular}{|c|c|c|c|c|c||c|c|} \hline
dataset & model & batch size (LR) & workers & epochs & avg interval & periodic avg & partial avg \\ \hline \hline
\multirow{3}{*}{CIFAR-10} & \multirow{3}{*}{ResNet20} & \multirow{3}{*}{32 (1.2)} & \multirow{15}{*}{128} & \multirow{3}{*}{300} & 2 & $91.19\pm 0.2\%$ & \textbf{91.89} $\pm 0.1\%$ \\
 & & & & & 4 & $89.80\pm 0.2\%$ & \textbf{90.58} $\pm 0.2\%$ \\ 
 & & & & & 8 & $85.70\pm 0.3\%$ & \textbf{88.17} $\pm 0.1\%$ \\ \cline{1-3}\cline{5-8}
\multirow{3}{*}{CIFAR-100} & \multirow{3}{*}{WRN28-10} & \multirow{3}{*}{32 (1.2)} & & \multirow{3}{*}{250} & 2 & $77.64\pm 0.2\%$ & \textbf{79.15} $\pm 0.1\%$ \\
 & & & & & 4 & $76.07\pm 0.2\%$ & \textbf{77.03} $\pm 0.2\%$ \\
 & & & & & 8 & $60.82\pm 0.2\%$ & \textbf{62.32} $\pm 0.2\%$ \\ \cline{1-3}\cline{5-8}
\multirow{3}{*}{SVHN} & \multirow{3}{*}{WRN16-8} & \multirow{3}{*}{64 (0.2)} & & \multirow{3}{*}{160} & 4 & $98.15\pm 0.1\%$ & \textbf{98.54}$\pm 0.1\%$ \\
 & & & & & 16 & $98.02\pm 0.2\%$ & \textbf{98.13}$\pm 0.1\%$ \\
 & & & & & 64 & $97.54\pm 0.1\%$ & \textbf{97.78}$\pm 0.1\%$ \\ \cline{1-3}\cline{5-8}
\multirow{3}{*}{Fasion-MNIST} & \multirow{3}{*}{VGG-11} & 32 (0.2) & & \multirow{3}{*}{90} & 2 & $92.33\pm 0.1\%$ & \textbf{94.01}$\pm 0.1\%$ \\ \cline{3-3}
 & & 32 (0.1) & & & 4 & $91.80\pm 0.1\%$ & \textbf{93.03}$\pm 0.1\%$ \\ \cline{3-3}
 & & 32 (0.08) & & & 8 & $90.48\pm 0.1\%$ & \textbf{92.21}$\pm 0.1\%$ \\ \cline{1-3}\cline{5-8}
\multirow{3}{*}{IMDB review} & \multirow{3}{*}{LSTM} & \multirow{3}{*}{10 (0.6)} & & \multirow{3}{*}{90} & 2 & $88.14\pm 0.1\%$ & \textbf{89.22}$\pm 0.1\%$ \\
 & & & & & 4 & $88.78\pm 0.2\%$ & \textbf{89.27}$\pm 0.1\%$ \\
 & & & & & 8 & $88.53\pm 0.2\%$ & \textbf{88.74}$\pm 0.3\%$ \\ \hline
\end{tabular}
\label{tab:classification}
\end{table*}

\subsection {Experimental Settings} \label{sec:settings}

We implemented our experiments using TensorFlow 2.4.0 \cite{tensorflow2015-whitepaper}.
All the experiments were conducted on a GPU cluster that has four compute nodes each of which has two NVIDIA V100 GPUs.
Because of the limited compute resources, we simulate the large-scale local SGD training such that all $m$ local models are distributed to $p$ processes ($m > p$), and each process sequentially trains the given $\frac{m}{p}$ local models.
When averaging the parameters, they are aggregated and summed up across the local models owned by each process first, and then reduced across all the processes using MPI communications.

We perform extensive Computer Vision experiments using popular benchmark datasets: CIFAR-10 and CIFAR-100 \cite{krizhevsky2009learning}, SVHN \cite{netzer2011reading}, Fashion-MNIST \cite{xiao2017fashion}, and Federated Extended MNIST \cite{caldas2018leaf}.
We also run Natural Language Processing (sentiment analysis) experiments using IMDB dataset \cite{maas-EtAl:2011:ACL-HLT2011}.
Due to the limited space, the details about the datasets and the model architectures are provided in Appendix.
We use momentum SGD with a coefficient of $0.9$ and apply gradual warmup \cite{goyal2017accurate} to the first 5 epochs to stabilize the training.
All the reported performance results are average accuracy across three separate runs.
Due to the limited space, we report the final accuracy only and show all the full learning curves in Appendix.

\begin{table*}[h!]
\scriptsize
\centering
\caption{
    CIFAR-10 classification results with extended training epochs.
    The partial averaging accuracy catches up with the sync SGD accuracy ($92.63 \pm 0.2\%$) faster than the periodic averaging.
    \vspace{0.1cm}
}
\begin{tabular}{|c|c|c|c|c||c|c|c|} \hline
dataset & model & \# of workers & avg interval & epochs & periodic avg. & partial avg. (proposed) \\ \hline \hline
\multirow{3}{*}{CIFAR-10} & \multirow{3}{*}{ResNet20} & \multirow{3}{*}{128} & \multirow{3}{*}{4} & 300 & $89.80\pm 0.2\%$ & $\textbf{90.58}\pm 0.2\%$ \\
 &  &  &  & 400 & $90.16\pm 0.1\%$ & $\textbf{91.70}\pm 0.2\%$ \\
 &  &  &  & 500 & $91.19\pm 0.2\%$ & $\textbf{92.20}\pm 0.1\%$ \\ \hline
\end{tabular}
\label{tab:longer}
\end{table*}

\subsection{Experiments with IID Data}

We use the hyper-parameter settings shown in the reference works, and further tune only the learning rate based on a grid search.
Table \ref{tab:classification} presents our experimental results achieved using 128 workers.
The partial averaging achieves a higher validation accuracy than the periodic averaging in all the experiments.
This comparison demonstrates that the partial averaging method effectively accelerates the local SGD for IID data.
We can also see that the accuracy consistently drops in all the experiments as the averaging interval $\tau$ increases.
While the larger interval improves the scaling efficiency by reducing the total communication cost, it can harm the statistical efficiency of local SGD.

We also present CIFAR-10 classification results with extended epochs in Table \ref{tab:longer}.
The partial averaging catches up with the synchronous SGD accuracy ($92.63 \pm 0.2\%$) faster than the periodic averaging.
We ran synchronous SGD using $128$ batch size and $0.1$ learning rate for 300 epochs.
One insight is that the degree of model discrepancy indeed strongly affects the final accuracy.
The synchronous SGD can be considered as a special case where the averaging interval is $1$.
That is, the degree of model discrepancy is always $0$, and thus synchronous SGD achieves a higher accuracy than any local SGD settings.
This indirectly explains why the partial averaging achieves a higher accuracy than the periodic averaging.
The lower the degree of model discrepancy across the workers, the higher the accuracy.

\begin{figure*}[t]
\centering
\includegraphics[width=2\columnwidth]{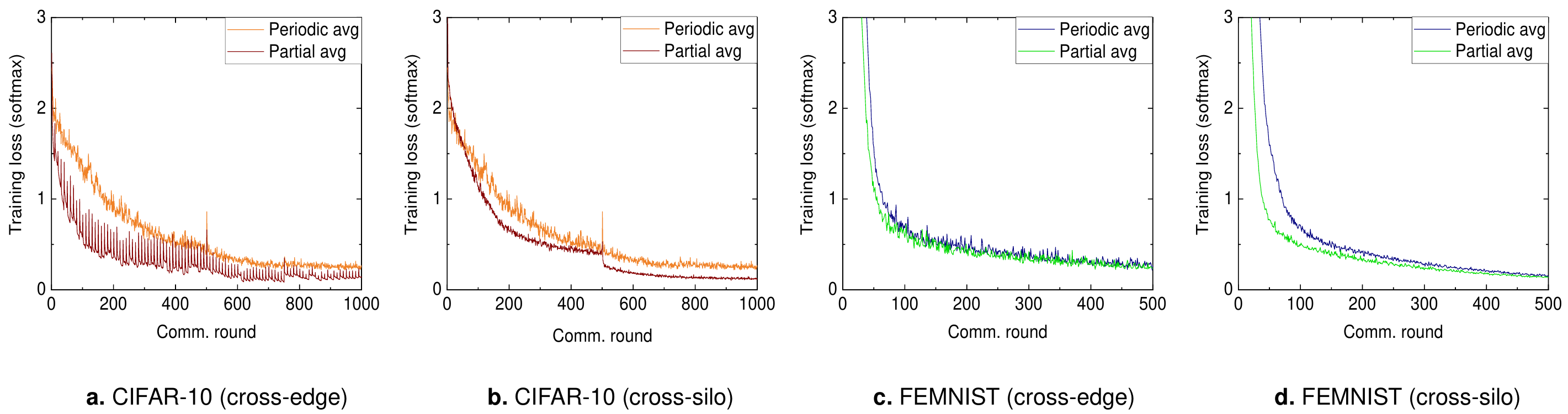}
\caption{
    (Non-IID data) The training loss comparison between the periodic averaging and the partial averaging across two different datasets: \textbf{a.} CIFAR-10 (\textit{cross-edge}), \textbf{b.} CIFAR-10 (\textit{cross-silo}), \textbf{c.} FEMNIST (\textit{cross-edge}), and \textbf{d.} FEMNIST (\textit{cross-silo}).
    In the \textit{cross-edge} settings, $25\%$ of random workers participate in training.
}
\label{fig:loss_niid}
\end{figure*}

\begin{table*}[t]
\scriptsize
\centering
\caption{
    Classification experiments using non-IID data.
    We conduct the experiments with various degrees of data heterogeneity (Dir($\alpha$)) and device selection ratio settings.
    The ResNet20 is trained for $10,000$ iterations.
    The LSTM and CNN are trained for $2,000$ iterations.
    \vspace{0.1cm}
}
\begin{tabular}{|c|c|c|c|c|c||c|c|} \hline
dataset & batch size (LR) & workers & avg interval & active ratio & Dir($\alpha$) & periodic avg & partial avg \\ \hline \hline
\multirow{9}{*}{\shortstack{CIFAR-10\\(ResNet20)}} & \multirow{8}{*}{32 (0.4)} & \multirow{9}{*}{128} & \multirow{9}{*}{10} & \multirow{3}{*}{$100\%$} & 1 & $90.38\pm 0.1\%$ & \textbf{91.54} $\pm 0.1\%$ \\
 & & & & & 0.5 & $90.18\pm 0.1\%$ & \textbf{91.56}$\pm 0.1\%$ \\
 & & & & & 0.1 & $89.92\pm 0.2\%$ & \textbf{91.31} $\pm 0.1\%$ \\
\cline{5-8}
 & & & & \multirow{3}{*}{$50\%$} & 1 & $89.98\pm 0.2\%$ & \textbf{90.61} $\pm 0.2\%$ \\
 & & & & & 0.5 & $89.51\pm 0.3\%$ & \textbf{91.02}$\pm 0.3\%$ \\
 & & & & & 0.1 & $88.99\pm 0.3\%$ & \textbf{90.64} $\pm 0.2\%$ \\
\cline{5-8}
 & & & & \multirow{3}{*}{$25\%$} & 1 & $89.32\pm 0.3\%$ & \textbf{91.00} $\pm 0.2\%$ \\
 & & & & & 0.5 & $88.73\pm 0.4\%$ & \textbf{90.16}$\pm 0.3\%$ \\
\cline{2-2}
 & 32 (0.2) & & & & 0.1 & $87.70\pm 0.4\%$ & \textbf{88.95} $\pm 0.3\%$ \\
\hline
\multirow{6}{*}{\shortstack{IMDB reviews\\(LSTM)}} & \multirow{2}{*}{10 (0.4)} & \multirow{6}{*}{128} & \multirow{6}{*}{10} & \multirow{2}{*}{$100\%$} & 1 & $88.03\pm 0.2\%$ & \textbf{88.68}$\pm 0.2\%$ \\
 & & & & & 0.5 & $87.72\pm 0.2\%$ & \textbf{88.40}$\pm 0.3\%$ \\
\cline{2-2}\cline{5-8}
 & \multirow{2}{*}{10 (0.2)} & & & \multirow{2}{*}{$50\%$} & 1 & $83.79\pm 0.3\%$ & \textbf{85.83}$\pm 0.3\%$ \\
 & & & & & 0.5 & $83.00\pm 0.2\%$ & \textbf{84.82}$\pm 0.2\%$ \\
\cline{2-2}\cline{5-8}
 & \multirow{2}{*}{10 (0.1)} & & & \multirow{2}{*}{$25\%$} & 1 & $81.13\pm 0.3\%$ & \textbf{83.40} $\pm 0.2\%$ \\
 & & & & & 0.5 & $80.02\pm 0.2\%$ & \textbf{82.01}$\pm 0.3\%$ \\
\hline
\multirow{3}{*}{FEMNIST} & \multirow{2}{*}{32 (0.1)} & \multirow{3}{*}{128} & \multirow{3}{*}{4} & $100\%$ & - & $83.93 \pm 0.4\%$ &\textbf{85.34} $\pm 0.3\%$ \\
\cline{5-8}
 & & & & $50\%$ & - & $85.27 \pm 0.3\%$ & \textbf{85.81}$\pm 0.1\%$ \\
\cline{2-2}\cline{5-8}
 & 32 (0.05) & & & $25\%$ & - & $85.73 \pm 0.2\%$ & \textbf{85.90}$\pm 0.1\%$ \\
\hline
\end{tabular}
\label{tab:noniid}
\end{table*}

\subsection {Experiments with Non-IID Data} \label{sec:noniid}
\textbf{Data Heterogeneity Settings} --
To evaluate the performance of the proposed framework in realistic Federated Learning environments, we also run experiments under two settings: non-IID data and partial device participation.
First, we generate synthetic heterogeneous data distributions based on Dirichlet's distribution.
We use concentration coefficients of $0.1$, $0.5$, and $1.0$ to evaluate the proposed framework across different degrees of data heterogeneity.
Second, we use three different device participation ratios, $25\%$ and $50\%$ (\textit{cross-edge}) and $100\%$ (\textit{cross-silo}).
For \textit{cross-edge} Federated Learning settings, we randomly select a subset of the workers for training at every communication round.
Note that, since the partial averaging method synchronizes only a subset of parameters at once, extra communications are required to send out the whole local model parameters to other workers at the end of every communication round.
To make a fair comparison with respect to the communication cost, we use a $10\%$ longer interval for the partial averaging and re-distribute the local models after every $10$ communication rounds.
Under this setting, the two averaging methods have a similar total communication cost while the partial averaging has a slightly higher degree of data heterogeneity.

\textbf{Accuracy Comparison} --
We fix all the factors that affect the training time: the number of workers, the number of training iterations, and the averaging interval, and then we tune the local batch size and learning rate.
Figure \ref{fig:loss_niid} shows the loss curves of CIFAR-10 and FEMNIST training.
Figure \ref{fig:loss_niid}\textbf{.a} and \textbf{c} show the curves for the \textit{cross-edge} settings and \textbf{b} and \textbf{d} show the curves for the \textit{cross-silo} settings.
Regardless of the ratio of participation, the partial averaging effectively accelerates the convergence of the training loss.
Due to the data heterogeneity, Federated Learning usually requires more iterations to converge than the training in centralized environments.
That is, the faster convergence likely results in achieving a higher validation accuracy within a fixed iteration budget.

Table \ref{tab:noniid} shows our best-tuned hyper-parameter settings and the accuracy results of the three different problems (CIFAR-10 and FEMNIST classifications and IMDB sentiment analysis).
Note that we set $\alpha \geq 0.5$ for IMDB because the Dirichlet's concentration coefficient lower than $0.5$ makes some workers not assigned with any training samples.
Given a fixed iteration budget, as expected, the partial averaging achieves a higher accuracy than the periodic averaging in all the experiments.
These results verifies that the partial averaging effectively mitigates the adverse impact of the model discrepancy on the global loss convergence in non-IID settings.

%% file: 6_related.tex
\section {Related Work}
\label{sec:related}

\textbf{Post-local SGD} -- Lin et al. proposed \textit{post-local} SGD in \cite{lin2018don}.
The algorithm begins the training with a single worker and then increases the number of workers once the learning rate is decayed.
This approach makes the model converge much faster than pure local SGD because the training does not suffer from the model discrepancy in the early training epochs.
However, it significantly undermines the degree of parallelism making it less practical.
The authors use up to 16 workers for training and achieve a comparable accuracy to that of synchronous SGD.

\textbf{Variance Reduced Stochastic Methods} -- 
Variance reduced stochastic methods, such as SVRG \cite{johnson2013accelerating} or SAGA \cite{defazio2014saga}, are known to improve the convergence rate of SGD.
Recently, Liang et al. successfully applied the variance reduction technique to local SGD \cite{liang2019variance}.
Despite the faster convergence thanks to the reduced stochastic variance, the variance reduction techniques are known to harm the generalization performance \cite{defazio2018ineffectiveness}.
This limitation is aligned with the fact that the convergence under a low noise condition can adversely affect the generalization performance \cite{li2019towards,lewkowycz2020large}.

We implemented VRL-SGD \cite{liang2019variance} using TensorFlow based on the open-source\footnote{https://github.com/zerolxf/VRL-SGD} of the reference work.
Due to the limited space, we provide the detailed experimental results in Appendix.
While VRL-SGD effectively accelerates the convergence of the training loss, we could observe a non-negligible gap of the validation accuracy between local SGD with the partial averaging and VRL-SGD.
Note that VRL-SGD performs extra computations to obtain average gradient deviations while our proposed model aggregation scheme does not have such computations.

\textbf{Adaptive Model Averaging Interval} -- Some researchers have proposed adaptive model averaging interval methods \cite{wang2018adaptive,haddadpour2019local}.
The common principle behind these works is that the communication cost of model averaging can be reduced by adjusting the averaging frequency based on the training progress at run-time.
The proposed partial averaging method is readily applicable to these adaptive interval methods because the proposed method is not dependent on any interval settings.
One can expect even better scaling efficiency if the largest averaging interval for each part of the model can be found.

%% file: 7_conclusion.tex
\section {Conclusion}
\label{sec:conclusion}

We proposed a partial model averaging framework for Federated Learning.
Our analysis and experimental results demonstrate the efficacy of the partial averaging in large-scale local SGD.
The proposed framework is readily applicable to any Federated Learning applications.
Breaking the conventional assumption of periodic model averaging can considerably broaden potential design options for Federated Learning algorithms.
We consider harmonizing the partial averaging with many existing advanced Federated Learning algorithms such as FedProx, FedNova, and SCAFFOLD as a critical future work.

\textbf{Societal Impacts} --
This research work does not have any potential adverse impacts on society.
Our study aims to improve the statistical efficiency of Federated Learning algorithms making better use of the provided hardware resources.
Consequently, large-scale Federated Learning applications may finish the neural network training faster, and it can results in reducing the CO2 footprint.
The faster training also may reduce the electricity consumption.

%% file: supplement_1_preliminary.tex
\section{Appendix}
\subsection{Preliminaries}

Herein, we provide the proofs of all Theorems and Lemmas presented in the paper.

\textbf{Notations} -- We first define a few notations for our analysis.
\begin{itemize}
    \item $m$: the number of local models (clients)
    \item $K$: the number of total training iterations
    \item $\tau$: the model averaging interval
    \item $\mathbf{x}_{(j,k)}^{i}$: a local model partition $j$ of client $i$ at iteration $k$
    \item $\mathbf{u}_{(j,k)}$: a model partition $j$ averaged across all $m$ clients at iteration $k$
    \item $\mathbf{g}_{(j,k)}^{i}$: a stochastic gradient of client $i$ with respect to the local model partition $j$ at iteration $k$
    \item $\nabla_j F_i(\cdot)$: a local full-batch gradient of client $i$ with respect to the model partition $j$
    \item $L_{max}$: the maximum Lipschitz constant across all $\tau$ model partitions; $L_{max} = max(L_j), j \in \{1, \cdots, \tau \}$
\end{itemize}

\textbf{Vectorization} --
We further define a vectorized form of the local model partition and its gradients as follows.
\begin{align}
    & \mathbf{x}_{(j,k)} = vec(\mathbf{x}_{(j,k)}^1, \mathbf{x}_{(j,k)}^2, \cdots, \mathbf{x}_{(j,k)}^m) \label{eq:vec1}\\
    & \mathbf{g}_{(j,k)} = vec(\mathbf{g}_{(j,k)}^1, \mathbf{g}_{(j,k)}^2, \cdots, \mathbf{g}_{(j,k)}^m) \label{eq:vec2}\\
    & \mathbf{f}_{(j,k)} = vec(\nabla_j F_1(\mathbf{x}_{k}^{1}), \nabla_j F_2(\mathbf{x}_{k}^{2}), \cdots, \nabla_j F_m(\mathbf{x}_{k}^{m})) \label{eq:vec3}
\end{align}

\textbf{Averaging Matrix} --
We first define a full-averaging matrix $\mathbf{J}_j$ for each model partition $j$ as follows.
\begin{equation}
\begin{split}
    & \mathbf{J}_j = \frac{1}{m}\mathbf{1}_{m}\mathbf{1}_{m}^{\top} \otimes \mathbf{I}_{d_j}, \hspace{0.5cm} j \in \{ 1, \cdots, \tau \},\\
\end{split}
\end{equation}
where $\otimes$ indicates Kronecker product, $\mathbf{1}_{m}\in\mathbb{R}^{m}$ is a vector of ones, and $\mathbf{I}_{d_j} \in \mathbb{R}^{d_j \times d_j}$ is an identity matrix.

Then, we define a time-varying partial-averaging matrix for each model partition $\mathbf{P}_{(j,k)} \in \mathbb{R}^{md_j \times md_j}$ for $j \in \{1, \cdots, \tau \}$.
\begin{equation}
\label{eq:matrix}
\mathbf{P}_{(j,k)} =
\begin{cases}
    \mathbf{J}_j &\hspace{0.5cm}\textrm{ if } k \textrm{ mod } \tau \textrm{ is } j \\
    \mathbf{I}_{j} &\hspace{0.5cm}\textrm{ if } k \textrm{ mod } \tau \textrm{ is not } j,
\end{cases}
\end{equation}
where $\mathbf{I}_j$ is an identity matrix of size $md_j \times md_j$.

Here we present an example where $d = 3$, $\tau = 2$, and $m = 2$.
If $d_0 = 2$ and $d_1 = 1$, we have $\mathbf{P}_{(j,k)}$ as follows.
\begin{align}
    &\mathbf{P}_{(0,0)} = 
    \begin{bmatrix}
    \frac{1}{2} & 0 & \frac{1}{2} & 0 \\
    0 & \frac{1}{2} & 0 & \frac{1}{2} \\
    \frac{1}{2} & 0 & \frac{1}{2} & 0 \\
    0 & \frac{1}{2} & 0 & \frac{1}{2} \\
    \end{bmatrix},
    \mathbf{P}_{(0,1)} = 
    \begin{bmatrix}
    1 & 0 & 0 & 0 \\
    0 & 1 & 0 & 0 \\
    0 & 0 & 1 & 0 \\
    0 & 0 & 0 & 1 \\
    \end{bmatrix},\\
    &\mathbf{P}_{(1,0)} = 
    \begin{bmatrix}
    1 & 0 \\
    0 & 1 \\
    \end{bmatrix},
    \mathbf{P}_{(1,1)} = 
    \begin{bmatrix}
    \frac{1}{2} & 0 \\
    0 & \frac{1}{2} \\
    \end{bmatrix}.
\end{align}

For instance, if $j = 0$ and ($k$ mod $\tau$) is $j$, the first partition of the model is averaged by multiplying $\mathbf{P}_{(0,k)}$ by $\mathbf{x}_{(0,k)}$ as follows.
\begin{align}
    \mathbf{P}_{(0,k)}\mathbf{x}_{(0,k)} & =
    \begin{bmatrix}
    \frac{1}{2} & 0 & \frac{1}{2} & 0 \\
    0 & \frac{1}{2} & 0 & \frac{1}{2} \\
    \frac{1}{2} & 0 & \frac{1}{2} & 0 \\
    0 & \frac{1}{2} & 0 & \frac{1}{2} \\
    \end{bmatrix}
    \begin{bmatrix}
    x^{(0,0)} \\
    x^{(0,1)} \\
    x^{(1,0)} \\
    x^{(1,1)} \\
    \end{bmatrix}
    = \begin{bmatrix}
    (x^{(0,0)} + x^{(1,0)}) / 2 \\
    (x^{(0,1)} + x^{(1,1)}) / 2 \\
    (x^{(0,0)} + x^{(1,0)}) / 2 \\
    (x^{(0,1)} + x^{(1,1)}) / 2
    \end{bmatrix},
\end{align}
where $x^{(i,j)}$ indicates the parameter $j$ of worker $i$.

The averaging matrix $\mathbf{P}_{(j,k)}$ has the following properties.
\begin{enumerate}
    \item {$\mathbf{P}_{(j,k)}\mathbf{1}_{md_j} = \mathbf{1}_{md_j}, \forall j \in \{1,\cdots,\tau \}$.}
    \item {$\mathbf{J}_j \mathbf{P}_{(j,k)}$ is $\mathbf{J}_{j}, \forall k \in \{1,\cdots, K \}$.}
    \item {$\mathbf{P}_{(j,k)}\mathbf{P}_{(j,k')}$ is symmetric, $\forall k,k' \in \{1,\cdots, K \}$.}
    \item {$\mathbf{P}_{(j,k)}\mathbf{P}_{(j,k')}$ consists only of $m \times m$ diagonal blocks, $\forall k,k' \in \{1,\cdots, K \}$.}
\end{enumerate}

Using the vectorized form of the parameters (\ref{eq:vec1}) and gradients (\ref{eq:vec2}), the parameter update rule of FedAvg is
\begin{align}
    \mathbf{x}_k^i & = \mathbf{P}_k(\mathbf{x}_{k-1}^i - \eta \mathbf{g}_{k-1}^{i}),
\end{align}
where $\eta$ is the learning rate.

%% file: supplement_2_iid.tex
\section {Convergence Analysis for IID Data}
Now, we provide the proof of main Theorem under IID data settings.

\subsection {Proof of Theorem 1}

\textbf{Theorem 1.} \textit{Suppose all $m$ local models are initialized to the same point $\mathbf{u}_1$. Under Assumption $1 \sim 3$, if Algorithm 1 runs for $K$ iterations using the learning rate $\eta$ that satisfies $L_{\max}^2 \eta^2 \tau (\tau - 1) + \eta L_{\max} \leq 1$, then the average-squared gradient norm of $\mathbf{u}_k$ is bounded as follows}
\begin{align}
    \mathop{\mathbb{E}}\left[\frac{1}{K}\sum_{i=1}^{K} \| \nabla F(\mathbf{u}_k) \|^2\right] & \leq \frac{2}{\eta K} \mathop{\mathbb{E}} \left[ F(\mathbf{u}_1) - F(\mathbf{u}_{inf}) \right] + \frac{\eta}{m} \sum_{j = 1}^{\tau} L_j \sigma_j^2 + \eta^2 (\tau - 1) \sum_{j = 1}^{\tau} L_j^2 \sigma_j^2
\end{align}

\begin{proof}
Based on Lemma \ref{lemma:framework} and \ref{lemma:iid_discrepancy}, we have
\begin{align}
    \mathop{\mathbb{E}}\left[\frac{1}{K}\sum_{i=1}^{K} \| \nabla F(\mathbf{u}_k) \|^2\right] & \leq \frac{2}{\eta K} \mathop{\mathbb{E}} \left[ F(\mathbf{u}_1) - F(\mathbf{u}_{inf}) \right] + \sum_{j = 1}^{\tau} \frac{\eta L_j \sigma_j^2}{m} + \sum_{j = 1}^{\tau} \left( \frac{\eta L_j - 1}{mK} \sum_{i=1}^{m} \sum_{k = 1}^{K} \| \nabla_j F(\mathbf{x}_k^i) \|^2 \right) \nonumber \\
    & \quad + \sum_{j=1}^{\tau} L_j^2 \left( \eta^2 (\tau - 1) \sigma_j^2 + \frac{\eta^2 \tau (\tau - 1)}{mK} \sum_{i=1}^{m} \sum_{k=1}^{K} \mathop{\mathbb{E}}\left[\left\| \nabla_j F(\mathbf{x}_k^i) \right\|^2\right] \right)
\end{align}
After a minor rearrangement, we have
\begin{align}
    \mathop{\mathbb{E}}\left[\frac{1}{K}\sum_{i=1}^{K} \| \nabla F(\mathbf{u}_k) \|^2\right] & \leq \frac{2}{\eta K} \mathop{\mathbb{E}} \left[ F(\mathbf{u}_1) - F(\mathbf{u}_{inf}) \right] + \sum_{j = 1}^{\tau} \frac{\eta L_j \sigma_j^2}{m} + \sum_{j = 1}^{\tau} L_j^2 \eta^2 (\tau - 1) \sigma_j^2 \nonumber \\
    & \quad + \sum_{j = 1}^{\tau} \left( \frac{L_j^2 \eta^2 \tau (\tau - 1) + \eta L_j - 1}{mK} \sum_{i=1}^{m} \sum_{k = 1}^{K} \| \nabla_j F(\mathbf{x}_k^i) \|^2 \right)
\end{align}
If the learning rate $\eta$ satisfies $L_{\max}^2 \eta^2 \tau (\tau - 1) + \eta L_{\max} \leq 1$, we have
\begin{align}
    \mathop{\mathbb{E}}\left[\frac{1}{K}\sum_{i=1}^{K} \| \nabla F(\mathbf{u}_k) \|^2\right] & \leq \frac{2}{\eta K} \mathop{\mathbb{E}} \left[ F(\mathbf{u}_1) - F(\mathbf{u}_{inf}) \right] + \sum_{j = 1}^{\tau} \frac{\eta L_j \sigma_j^2}{m} + \sum_{j = 1}^{\tau} L_j^2 \eta^2 (\tau - 1) \sigma_j^2 \nonumber \\
    & = \frac{2}{\eta K} \mathop{\mathbb{E}} \left[ F(\mathbf{u}_1) - F(\mathbf{u}_{inf}) \right] + \frac{\eta}{m} \sum_{j = 1}^{\tau} L_j \sigma_j^2 + \eta^2 (\tau - 1) \sum_{j = 1}^{\tau} L_j^2 \sigma_j^2
\end{align}
We finish the proof.
\end{proof}

\subsection {Proof of Lemma 1}

\textbf{Lemma 1.}
\textit{(framework) Under Assumption $1 \sim 3$, if $\eta \leq \frac{1}{L}$, Algorithm 1 ensures}
\begin{align}
    \mathop{\mathbb{E}}\left[\frac{1}{K}\sum_{i=1}^{K} \| \nabla F(\mathbf{u}_k) \|^2\right] \leq \frac{2}{\eta K} \mathop{\mathbb{E}} \left[ F(\mathbf{u}_1) - F(\mathbf{u}_{inf}) \right] + \frac{\eta}{m} \sum_{j = 1}^{\tau} L_j \sigma_j^2 + \frac{1}{mK} \sum_{i = 1}^{m} \sum_{k = 1}^{K} \sum_{j = 1}^{\tau} L_j^2 \| \mathbf{u}_{(j,k)} - \mathbf{x}_{(j,k)}^i \|^2
\end{align}
\begin{proof}
Based on Assumption 1, we have
\begin{align}
    \mathop{\mathbb{E}}\left[F(\mathbf{u}_{k+1}) - F(\mathbf{u}_k)\right] & \leq -\eta \sum_{j = 1}^{\tau} \mathop{\mathbb{E}}\left[\langle \nabla_j F(\mathbf{u}_{k}), \frac{1}{m}\sum_{i=1}^{m} \mathbf{g}_{(j,k)}^{i} \rangle\right] + \sum_{j = 1}^{\tau} \frac{\eta^2 L_j}{2} \left( \mathop{\mathbb{E}}\left[\left\| \frac{1}{m}\sum_{i=1}^{m} \mathbf{g}_{(j,k)}^{i} \right\|^2\right] \right) \label{eq:iid_lipschitz}.
\end{align}

The first term on the right-hand side in (\ref{eq:iid_lipschitz}) can be re-written as follows.
\begin{align}
    & -\eta \sum_{j = 1}^{\tau} \left( \mathop{\mathbb{E}}\left[\langle \nabla_j F(\mathbf{u}_{k}), \frac{1}{m}\sum_{i=1}^{m} \mathbf{g}_{(j,k)}^{i} \rangle\right] \right) = -\eta \sum_{j = 1}^{\tau} \left( \frac{1}{m} \sum_{i=1}^{m} \langle \nabla_j F(\mathbf{u}_k), \nabla_j F(\mathbf{x}_k^i) \rangle \right) \\
    & \quad = -\eta \sum_{j = 1}^{\tau} \left( \frac{1}{2m}\sum_{i=1}^{m} \left( \| \nabla_j F(\mathbf{u}_k) \|^2 + \| \nabla_j F(\mathbf{x}_k^i) \|^2 - \| \nabla_j F(\mathbf{u}_k) - \nabla_j F(\mathbf{x}_k^i) \|^2 \right) \right) \label{eq:expand} \\
    & \quad = -\eta \sum_{j = 1}^{\tau} \left( \frac{1}{2} \| \nabla_j F(\mathbf{u}_k) \|^2 + \frac{1}{2m} \sum_{i=1}^{m} \| \nabla_j F(\mathbf{x}_k^i) \|^2 - \frac{1}{2m} \sum_{i=1}^{m} \| \nabla_j F(\mathbf{u}_k) - \nabla_j F(\mathbf{x}_k^i) \|^2 \right), \label{eq:first}
\end{align}
where (\ref{eq:expand}) holds based on a basic equality: $2a^{\top} b = \|a\|^2 + \|b\|^2 - \|a-b\|^2$.

The second term on the right-hand side in (\ref{eq:iid_lipschitz}) is bounded as follows.
\begin{align}
    &\sum_{j = 1}^{\tau} \frac{\eta^2 L_j}{2} \left( \mathop{\mathbb{E}}\left[\left\| \frac{1}{m}\sum_{i=1}^{m} \mathbf{g}_{(j,k)}^{i} \right\|^2\right] \right) \nonumber \\
    & \quad\quad = \sum_{j = 1}^{\tau} \frac{\eta^2 L_j}{2} \left( \mathop{\mathbb{E}}\left[\left\| \frac{1}{m}\sum_{i=1}^{m} \mathbf{g}_{(j,k)}^{i} - \mathop{\mathbb{E}}\left[ \frac{1}{m}\sum_{i=1}^{m} \mathbf{g}_{(j,k)}^{i} \right]\right\|^2\right] + \left\| \mathop{\mathbb{E}} \left[ \frac{1}{m}\sum_{i=1}^{m} \mathbf{g}_{(j,k)}^{i} \right] \right\|^2 \right) \label{eq:iid_split} \\
    & \quad\quad = \sum_{j = 1}^{\tau} \frac{\eta^2 L_j}{2} \left( \mathop{\mathbb{E}}\left[\left\| \frac{1}{m}\sum_{i=1}^{m} \mathbf{g}_{(j,k)}^{i} - \frac{1}{m}\sum_{i=1}^{m} \nabla_j F(\mathbf{x}_k^i) \right\|^2\right] + \left\| \frac{1}{m}\sum_{i=1}^{m} \nabla_j F(\mathbf{x}_k^i) \right\|^2 \right) \nonumber\\
    & \quad\quad = \sum_{j = 1}^{\tau} \frac{\eta^2 L_j}{2} \left( \frac{1}{m^2}\sum_{i=1}^{m} \mathop{\mathbb{E}}\left[\left\| \mathbf{g}_{(j,k)}^{i} - \nabla_j F(\mathbf{x}_k^i) \right\|^2\right] + \left\| \frac{1}{m}\sum_{i=1}^{m} \nabla_j F(\mathbf{x}_k^i) \right\|^2 \right) \label{eq:sigma}\\
    & \quad\quad \leq \sum_{j = 1}^{\tau} \frac{\eta^2 L_j}{2} \left( \frac{1}{m^2}\sum_{i=1}^{m} \mathop{\mathbb{E}}\left[\left\| \mathbf{g}_{(j,k)}^{i} - \nabla_j F(\mathbf{x}_k^i) \right\|^2\right] + \frac{1}{m}\sum_{i=1}^{m} \left\| \nabla_j F(\mathbf{x}_k^i) \right\|^2 \right) \label{eq:jensen}\\
    & \quad\quad \leq \sum_{j = 1}^{\tau} \frac{\eta^2 L_j}{2} \left( \frac{\sigma_j^2}{m} + \frac{1}{m}\sum_{i=1}^{m} \left\| \nabla_j F(\mathbf{x}_k^i) \right\|^2 \right), \label{eq:second}
\end{align}
where (\ref{eq:iid_split}) follows a basic equality: $\mathop{\mathbb{E}}[\| \mathbf{x} \|^2] = \mathop{\mathbb{E}}[\| \mathbf{x} - \mathop{\mathbb{E}}[ \mathbf{x} ] \|^2] + \| \mathop{\mathbb{E}}[ \mathbf{x} ] \|^2$ for any random vector $\mathbf{x}$.
(\ref{eq:sigma}) holds because $\mathbf{g}_{(j,k)}^{i} - \nabla_j F(\mathbf{x}_k^i)$ has $0$ mean and is independent across $i$.
(\ref{eq:jensen}) holds based on the convexity of $\ell_2$ norm and Jensen's inequality.
Then, plugging in (\ref{eq:first}) and (\ref{eq:second}) into (\ref{eq:iid_lipschitz}), we have
\begin{align}
    & \mathop{\mathbb{E}}\left[F(\mathbf{u}_{k+1}) - F(\mathbf{u}_k)\right] \nonumber \\
    & \quad\quad \leq -\eta \sum_{j = 1}^{\tau} \left( \frac{1}{2} \| \nabla_j F(\mathbf{u}_k) \|^2 + \frac{1}{2m} \sum_{i=1}^{m} \| \nabla_j F(\mathbf{x}_k^i) \|^2 - \frac{1}{2m} \sum_{i=1}^{m} \| \nabla_j F(\mathbf{u}_k) - \nabla_j F(\mathbf{x}_k^i) \|^2 \right) \nonumber \\
    & \quad\quad\quad\quad + \sum_{j = 1}^{\tau} \frac{\eta^2 L_j}{2} \left( \frac{\sigma_j^2}{m} + \frac{1}{m}\sum_{i=1}^{m} \left\| \nabla_j F(\mathbf{x}_k^i) \right\|^2 \right) \nonumber \\
    & \quad\quad = \sum_{j = 1}^{\tau} \left( - \frac{\eta}{2} \| \nabla_j F(\mathbf{u}_k \|^2 - \frac{\eta}{2m} \sum_{i=1}^{m} \| \nabla_j F(\mathbf{x}_k^i) \|^2 + \frac{\eta}{2m} \sum_{i=1}^{m} \| \nabla_j F(\mathbf{u}_k) - \nabla_j F(\mathbf{x}_k^i) \|^2 \right) \nonumber \\
    & \quad\quad\quad\quad + \sum_{j = 1}^{\tau} \left( \frac{\eta^2 L_j \sigma_j^2}{2m} + \frac{\eta^2 L_j}{2m} \sum_{i=1}^{m} \| \nabla_j F(\mathbf{x}_k^i) \|^2 \right) \label{eq:long}
\end{align}
After dividing both sides of (\ref{eq:long}) by $\frac{\eta}{2}$ and rearranging, we have
\begin{align}
    \sum_{j = 1}^{\tau} \| \nabla_j F(\mathbf{u}_k) \|^2 & \leq \frac{2}{\eta} \mathop{\mathbb{E}} \left[ F(\mathbf{u}_k) - F(\mathbf{u}_{k+1}) \right] + \sum_{j = 1}^{\tau} \frac{\eta L_j \sigma_j^2}{m} + \sum_{j = 1}^{\tau} \left( \frac{\eta L_j - 1}{m} \sum_{i=1}^{m} \| \nabla_j F(\mathbf{x}_k^i) \|^2 \right) \nonumber \\
    & \quad\quad + \sum_{j = 1}^{\tau} \left( \frac{1}{m} \sum_{i = 1}^{m} \| \nabla_j F(\mathbf{u}_k) - \nabla_j F(\mathbf{x}_k^i) \|^2 \right) \label{eq:rearrange}
\end{align}
Taking expectation on both sides of (\ref{eq:rearrange}) and averaging it over $K$ iterations, we have
\begin{align}
    \mathop{\mathbb{E}} \left[ \frac{1}{K} \sum_{k = 1}^{K} \left( \sum_{j = 1}^{\tau} \| \nabla_j F(\mathbf{u}_k) \|^2 \right) \right] & \leq \frac{2}{\eta K} \mathop{\mathbb{E}} \left[ F(\mathbf{u}_1) - F(\mathbf{u}_{k+1}) \right] + \sum_{j = 1}^{\tau} \frac{\eta L_j \sigma_j^2}{m} \nonumber \\
    & \quad\quad + \sum_{j = 1}^{\tau} \left( \frac{\eta L_j - 1}{mK} \sum_{i=1}^{m} \sum_{k = 1}^{K} \mathop{\mathbb{E}} \left[ \| \nabla_j F(\mathbf{x}_k^i) \|^2 \right] \right) \nonumber \\
    & \quad\quad + \sum_{j = 1}^{\tau} \left( \frac{1}{mK} \sum_{i = 1}^{m} \sum_{k = 1}^{K} \mathop{\mathbb{E}} \left[ \| \nabla_j F(\mathbf{u}_k) - \nabla_j F(\mathbf{x}_k^i) \|^2 \right] \right)
\end{align}
If $\eta \leq \frac{1}{L_{max}}$, then
\begin{align}
    \mathop{\mathbb{E}} \left[ \frac{1}{K} \sum_{k = 1}^{K} \left( \sum_{j = 1}^{\tau} \| \nabla_j F(\mathbf{u}_k) \|^2 \right) \right] & \leq \frac{2}{\eta K} \mathop{\mathbb{E}} \left[ F(\mathbf{u}_1) - F(\mathbf{u}_{k+1}) \right] + \sum_{j = 1}^{\tau} \frac{\eta L_j \sigma_j^2}{m} \nonumber \\
    & \quad\quad + \sum_{j = 1}^{\tau} \left( \frac{1}{mK} \sum_{i = 1}^{m} \sum_{k = 1}^{K} \mathop{\mathbb{E}} \left[ \| \nabla_j F(\mathbf{u}_k) - \nabla_j F(\mathbf{x}_k^i) \|^2 \right] \right) \nonumber \\
    & \leq \frac{2}{\eta K} \mathop{\mathbb{E}} \left[ F(\mathbf{u}_1) - F(\mathbf{u}_{k+1}) \right] + \sum_{j = 1}^{\tau} \frac{\eta L_j \sigma_j^2}{m} \nonumber \\
    & \quad\quad + \frac{1}{mK} \sum_{i = 1}^{m} \sum_{k = 1}^{K} \left( \sum_{j = 1}^{\tau} L_j^2 \mathop{\mathbb{E}} \left[ \| \mathbf{u}_{(j,k)} - \mathbf{x}_{(j,k)}^i \|^2 \right] \right) \label{eq:iid_lipscitz_late},
\end{align}
where (\ref{eq:iid_lipscitz_late}) holds based on Assumption 1.
Finally, summing up the gradients of $\tau$ model partitions, we have
\begin{align}
    \mathop{\mathbb{E}} \left[ \frac{1}{K} \sum_{k = 1}^{K} \| \nabla F(\mathbf{u}_k) \|^2 \right] & \leq \frac{2}{\eta K} \mathop{\mathbb{E}} \left[ F(\mathbf{u}_1) - F(\mathbf{u}_{k+1}) \right] + \sum_{j = 1}^{\tau} \frac{\eta L_j \sigma_j^2}{m} + \frac{1}{mK} \sum_{i = 1}^{m} \sum_{k = 1}^{K} \sum_{j = 1}^{\tau} L_j^2 \mathop{\mathbb{E}} \left[ \| \mathbf{u}_{(j,k)} - \mathbf{x}_{(j,k)}^i \|^2 \right] \nonumber \\
    & \leq \frac{2}{\eta K} \mathop{\mathbb{E}} \left[ F(\mathbf{u}_1) - F(\mathbf{u}_{inf}) \right] + \frac{\eta}{m} \sum_{j = 1}^{\tau} L_j \sigma_j^2 + \frac{1}{mK} \sum_{i = 1}^{m} \sum_{k = 1}^{K}  \sum_{j = 1}^{\tau} L_j^2 \mathop{\mathbb{E}} \left[ \| \mathbf{u}_{(j,k)} - \mathbf{x}_{(j,k)}^i \|^2 \right] \nonumber
\end{align}
We complete the proof.

\end{proof}

\subsection {Proof of Lemma 2}

\textbf{Lemma 2.}
\textit{(model discrepancy) Under Assumption $1 \sim 3$, Algorithm 1 ensures}
\begin{align}
    & \frac{1}{mK}\sum_{k=1}^{K}\sum_{i=1}^{m} \sum_{j=1}^{\tau} L_j^2 \mathop{\mathbb{E}} \left[\| \mathbf{u}_{(j,k)} - \mathbf{x}_{(j,k)}^i \|^2 \right] \nonumber \\
    & \quad\quad \leq \sum_{j=1}^{\tau} L_j^2 \left( \eta^2 (\tau - 1) \sigma_j^2 + \frac{\eta^2 \tau (\tau - 1)}{mK} \sum_{i=1}^{m} \sum_{k=1}^{K} \mathop{\mathbb{E}}\left[\left\| \nabla_j F(\mathbf{x}_k^i) \right\|^2\right] \right).\nonumber
\end{align}
\begin{proof}
The averaged distance of each partition $j$ can be re-written using the vectorized form of the parameters as follows.
\begin{align}
    \sum_{i=1}^{m} \left\| \mathbf{u}_{(j,k)}- \mathbf{x}_{(j,k)}^{i} \right\|^2 = \left\| \mathbf{J}_j \mathbf{x}_{(j,k)} - \mathbf{x}_{(j,k)} \right\|^2 = \left\| (\mathbf{J}_j - \mathbf{I}_j) \mathbf{x}_{(j,k)} \right\|^2.
\end{align}
According to the parameter update rule, we have
\begin{align}
    (\mathbf{J}_j - \mathbf{I}_j)\mathbf{x}_{(j,k)} & = (\mathbf{J}_j - \mathbf{I}_j)\mathbf{P}_{(j,k-1)} (\mathbf{x}_{(j,k-1)} - \eta \mathbf{g}_{(j,k-1)}) \nonumber \\
    & = (\mathbf{J}_j - \mathbf{I}_j)\mathbf{P}_{(j,k-1)} \mathbf{x}_{(j,k-1)} - (\mathbf{J}_j - \mathbf{P}_{(j,k-1)})\eta \mathbf{g}_{(j,k-1)},
\end{align}
where the second equality holds because $\mathbf{J}_j \mathbf{P}_j = \mathbf{J}_j$ and $\mathbf{I}_j \mathbf{P}_j = \mathbf{P}_j$.

Then, expanding the expression of $\mathbf{x}_{(j,k-1)}$, we have
\begin{align}
    (\mathbf{J}_j - \mathbf{I}_j)\mathbf{x}_{(j,k)} & = (\mathbf{J}_j - \mathbf{I}_j) \mathbf{P}_{(j,k-1)}(\mathbf{P}_{(j,k-2)}(\mathbf{x}_{(j,k-2)} - \eta \mathbf{g}_{(j,k-2)})) - (\mathbf{J}_j - \mathbf{P}_{(j,k-1)}) \eta \mathbf{g}_{(j,k-1)} \nonumber \\
    &= (\mathbf{J}_j - \mathbf{I}_j) \mathbf{P}_{(j,k-1)}\mathbf{P}_{(j,k-2)}\mathbf{x}_{(j,k-2)} - (\mathbf{J}_j - \mathbf{P}_{(j,k-1)}\mathbf{P}_{(j,k-2)})\mu \mathbf{g}_{(j,k-2)} \\
    & \quad - (\mathbf{J}_j - \mathbf{P}_{(j,k-1)})\mu \mathbf{g}_{(j,k-1)}.\nonumber
\end{align}

Repeating the same procedure for $\mathbf{x}_{(j,k-2)}, \mathbf{x}_{(j,k-3)}, \cdots, \mathbf{x}_{(j,2)}$, we have
\begin{align}
    (\mathbf{J}_j - \mathbf{I}_j)\mathbf{x}_{(j,k)} &= (\mathbf{J}_j - \mathbf{I}_j) \prod_{s=1}^{k-1}\mathbf{P}_{(j,s)}\mathbf{x}_{(j,1)} - \eta \sum_{s=1}^{k-1}(\mathbf{J}_j - \prod_{l=s}^{k-1}\mathbf{P}_{(j,l)})\mathbf{g}_{(j,s)} \nonumber\\
    &= - \eta \sum_{s=1}^{k-1}(\mathbf{J}_j - \prod_{l=s}^{k-1}\mathbf{P}_{(j,l)})\mathbf{g}_{(j,s)},
\end{align}
where the second equality holds since $\mathbf{x}_{(j,1)}^{i}$ is all the same among $m$ workers and thus $(\mathbf{J}_j - \mathbf{I}_j)\mathbf{x}_{(j,1)}$ is $0$.

Then, we have
\begin{align}
    & \frac{1}{mK}\sum_{k=1}^{K}\sum_{i=1}^{m} \sum_{j=1}^{\tau} L_j^2 \mathop{\mathbb{E}}[\| \mathbf{u}_{(j,k)} - \mathbf{x}_{(j,k)}^i \|^2] \nonumber \\
    & = \frac{1}{mK} \sum_{k=1}^{K} \sum_{j=1}^{\tau} L_j^2 \mathop{\mathbb{E}}[\| (\mathbf{J}_j - \mathbf{I}_j)\mathbf{x}_{(j,k)} \|^2] \nonumber \\ 
    & = \frac{\eta^2}{mK}\sum_{k=1}^{K} \sum_{j=1}^{\tau}L_j^2 \mathop{\mathbb{E}}[\| \sum_{s=1}^{k-1} (\mathbf{J}_j - \prod_{l=s}^{k-1}\mathbf{P}_{(j,l)}) \mathbf{g}_{(j,s)} \|^2] \nonumber \\
    & = \frac{\eta^2}{mK}\sum_{k=1}^{K} \sum_{j=1}^{\tau} L_j^2 \mathop{\mathbb{E}}\left[\left\| \sum_{s=1}^{k-1} (\mathbf{J}_j - \prod_{l=s}^{k-1}\mathbf{P}_{(j,l)})(\mathbf{g}_{(j,s)} - \mathbf{f}_{(j,s)}) + \sum_{s=1}^{k-1} (\mathbf{J}_j - \prod_{l=s}^{k-1}\mathbf{P}_{(j,l)}) \mathbf{f}_{(j,s)} \right\|^2\right] \nonumber \\
    & \leq \frac{2\eta^2}{mK} \sum_{j=1}^{\tau} L_j^2 \left( \underset{T_1}{\underbrace{ \sum_{k=1}^{K} \mathop{\mathbb{E}}\left[\left\| \sum_{s=1}^{k-1} (\mathbf{J}_j - \prod_{l=s}^{k-1}\mathbf{P}_{(j,l)}) (\mathbf{g}_{(j,s)} - \mathbf{f}_{(j,s)}) \right\|^2\right]}} +  \underset{T_2}{\underbrace{ \sum_{k=1}^{K} \mathop{\mathbb{E}}\left[\left\| \sum_{s=1}^{k-1} (\mathbf{J}_j - \prod_{l=s}^{k-1}\mathbf{P}_{(j,l)}) \mathbf{f}_{(j,s)} \right\|^2\right]}} \right), \label{eq:t1t2}
\end{align}
where the last inequality holds based on a basic inequality: $\|a+b\|^2 \leq 2\|a\|^2 + 2\|b\|^2$.
Now we focus on bounding the above two terms, $T_1$ and $T_2$, separately.

\textbf{Bounding $T_1$} --
We first partition $K$ iterations into three subsets: the first $j$ iterations, the next $K - \tau$ iterations, and the final $\tau - j$ iterations.
We bound the first $j$ iterations as follows.
\begin{align}
    \sum_{k=1}^{j} \mathop{\mathbb{E}}\left[\left\| \sum_{s=1}^{k-1} \left( \mathbf{J}_j - \prod_{l=s}^{k-1} \mathbf{P}_{(j,l)} \right) \left( \mathbf{g}_{(j,s)} - \mathbf{f}_{(j,s)} \right) \right\|^2\right] & = \sum_{k=1}^{j} \sum_{s=1}^{k-1} \mathop{\mathbb{E}}\left[\left\| \left( \mathbf{J}_j - \prod_{l=s}^{k-1} \mathbf{P}_{(j,l)} \right) \left( \mathbf{g}_{(j,s)} - \mathbf{f}_{(j,s)} \right) \right\|^2\right] \label{eq:j_jensen} \\
    & = \sum_{k=1}^{j} \sum_{s=1}^{k-1} \mathop{\mathbb{E}}\left[\left\| \left( \mathbf{J}_j - \mathbf{I}_{j} \right) \left(\mathbf{g}_{(j,s)} - \mathbf{f}_{(j,s)} \right) \right\|^2\right] \label{eq:j_reduce} \\
    & \leq \sum_{k=1}^{j} \sum_{s=1}^{k-1} \mathop{\mathbb{E}}\left[\left\| \mathbf{g}_{(j,s)} - \mathbf{f}_{(j,s)} \right\|^2 \right] \label{eq:zeroout}\\
    & = \sum_{k=1}^{j} \sum_{s=1}^{k-1} \sum_{i=1}^{m} \mathop{\mathbb{E}}\left[\left\|  \mathbf{g}_{(j,s)}^{i} - \nabla_j F_i (\mathbf{x}_{s}^i ) \right\|^2 \right] \nonumber \\
    & \leq \sum_{k=1}^{j} \sum_{s=1}^{k-1} m \sigma_j^2 = \frac{j (j - 1)}{2} m \sigma_j^2, \label{eq:j_bound}
\end{align}
where (\ref{eq:j_jensen}) holds because $\mathbf{g}_{(j,s)} - \nabla_j F(\mathbf{x}_{s})$ has a mean of $0$ and independent across $s$; (\ref{eq:j_reduce}) holds because $\prod_{l=s}^{k-1} \mathbf{P}_{(j,l)}$ is $\mathbf{I}_j$ when $k < j$; (\ref{eq:zeroout}) holds based on Lemma \ref{lemma:zeroout}.

Then, the next $K - \tau$ iterations of $T_1$ are bounded as follows.
\begin{align}
    & \sum_{k = j + 1}^{K - \tau + j} \mathop{\mathbb{E}}\left[\left\| \sum_{s=1}^{k-1} \left( \mathbf{J}_j - \prod_{l=s}^{k-1} \mathbf{P}_{(j,l)} \right) \left( \mathbf{g}_{(j,s)} - \mathbf{f}_{(j,s)} \right) \right\|^2\right] & = \sum_{k = j + 1}^{K - \tau + j} \sum_{s=1}^{k-1} \mathop{\mathbb{E}}\left[\left\| \left(\mathbf{J}_j - \prod_{l=s}^{k-1} \mathbf{P}_{(j,l)} \right) \left( \mathbf{g}_{(j,s)} - \mathbf{f}_{(j,s)} \right) \right\|^2\right]. \label{eq:middle_op}
\end{align}
Without loss of generality, we replace $k$ with $a\tau + b + j$ where $a$ is the communication round and $b$ is the local update step.
Note that these iterations are $\frac{K}{\tau} - 1$ full communication rounds.
So, (\ref{eq:middle_op}) can be re-written and bounded as follows.
\begin{align}
    & \sum_{a = 0}^{K / \tau - 2} \sum_{b = 1}^{\tau} \sum_{s = 1}^{a\tau + b + j - 1} \mathop{\mathbb{E}}\left[\left\| \left(\mathbf{J}_j - \prod_{l=s}^{k-1} \mathbf{P}_{(j,l)} \right) \left( \mathbf{g}_{(j,s)} - \mathbf{f}_{(j,s)} \right) \right\|^2 \right] \nonumber\\
    & \quad\quad = \sum_{a = 0}^{K / \tau - 2} \sum_{b = 1}^{\tau} \sum_{s = 1}^{a\tau + j} \mathop{\mathbb{E}}\left[\left\| \left(\mathbf{J}_j - \prod_{l=s}^{k-1} \mathbf{P}_{(j,l)} \right) \left( \mathbf{g}_{(j,s)} - \mathbf{f}_{(j,s)} \right) \right\|^2 \right] \nonumber \\
    & \quad\quad\quad\quad + \sum_{a = 0}^{K / \tau - 2} \sum_{b = 1}^{\tau} \sum_{s = a\tau + j + 1}^{a\tau + b + j - 1} \mathop{\mathbb{E}}\left[\left\| \left(\mathbf{J}_j - \prod_{l=s}^{k-1} \mathbf{P}_{(j,l)} \right) \left( \mathbf{g}_{(j,s)} - \mathbf{f}_{(j,s)} \right\|^2 \right) \right] \nonumber\\
    &  \quad\quad = \sum_{a = 0}^{K / \tau - 2} \sum_{b = 1}^{\tau} \sum_{s = a\tau + j + 1}^{a\tau + b + j - 1} \mathop{\mathbb{E}}\left[\left\| \left(\mathbf{J}_j - \prod_{l=s}^{k-1} \mathbf{P}_{(j,l)} \right) \left( \mathbf{g}_{(j,s)} - \mathbf{f}_{(j,s)} \right\|^2 \right) \right] \label{eq:middle_zeroout} \\
    & \quad\quad = \sum_{a = 0}^{K / \tau - 2} \sum_{b = 1}^{\tau} \sum_{s = a\tau + j + 1}^{a\tau + b + j - 1} \mathop{\mathbb{E}}\left[\left\| \mathbf{g}_{(j,s)} - \mathbf{f}_{(j,s)} \right\|^2 \right] \label{eq:zeroout_2} \\
    & \quad\quad \leq \sum_{a = 0}^{K / \tau - 2} \sum_{b = 1}^{\tau} \sum_{s = a\tau + j + 1}^{a\tau + b + j - 1} m \sigma_j^2 \nonumber\\
    & \quad\quad = \sum_{a = 0}^{K / \tau - 2} \sum_{b = 1}^{\tau} (b - 1) m \sigma_j^2 \nonumber\\
    & \quad\quad = \sum_{a = 0}^{K / \tau - 2} \frac{\tau (\tau - 1)}{2} m \sigma_j^2 = (\frac{K}{\tau} - 1)\frac{\tau (\tau - 1)}{2} m \sigma_j^2, \label{eq:middle_bound}
\end{align}
where (\ref{eq:middle_zeroout}) holds because $\mathbf{J}_j - \prod_{l=s}^{a\tau + b + j - 1} \mathbf{P}_{(j,l)}$ becomes $0$ when $s \leq a\tau + j$; (\ref{eq:zeroout_2}) holds based on Lemma \ref{lemma:zeroout}.

Finally, the last $\tau - j$ iterations are bounded as follows.
\begin{align}
    & \sum_{k = K - \tau + j + 1}^{K} \mathop{\mathbb{E}}\left[\left\| \sum_{s=1}^{k-1} \left( \mathbf{J}_j - \prod_{l=s}^{k-1} \mathbf{P}_{(j,l)} \right) \left( \mathbf{g}_{(j,s)} - \mathbf{f}_{(j,s)} \right) \right\|^2\right] \nonumber\\
    & = \sum_{k = K - \tau + j + 1}^{K} \sum_{s=1}^{k-1} \mathop{\mathbb{E}}\left[\left\| \left( \mathbf{J}_j - \prod_{l=s}^{k-1} \mathbf{P}_{(j,l)} \right) \left( \mathbf{g}_{(j,s)} - \mathbf{f}_{(j,s)} \right) \right\|^2\right] \label{eq:last_jensen} \\
\end{align}
where (\ref{eq:last_jensen}) holds because $\mathbf{g}_{(j,s)} - \nabla_j F(\mathbf{x}_{s})$ has $0$ mean and independent across $s$.
Replacing $k$ with $a\tau + b + j$, we have
\begin{align}
    & \sum_{a = K / \tau - 1}^{K / \tau - 1} \sum_{b = 1}^{\tau - j} \left(\sum_{s = 1}^{a\tau + b + j - 1} \mathop{\mathbb{E}}\left[\left\| \left( \mathbf{J}_j - \prod_{l=s}^{k-1} \mathbf{P}_{(j,l)} \right) \left( \mathbf{g}_{(j,s)} - \mathbf{f}_{(j,s)} \right) \right\|^2 \right] \right) \nonumber \\
    & \quad\quad = \sum_{a = K / \tau - 1}^{K / \tau - 1} \sum_{b = 1}^{\tau - j} \left( \sum_{s = 1}^{a\tau + j} \mathop{\mathbb{E}}\left[\left\| \left( \mathbf{J}_j - \prod_{l=s}^{k-1} \mathbf{P}_{(j,l)} \right) \left( \mathbf{g}_{(j,s)} - \mathbf{f}_{(j,s)} \right) \right\|^2 \right] \right) \nonumber \\
    & \quad\quad\quad\quad + \sum_{a = K / \tau - 1}^{K / \tau - 1} \sum_{b = 1}^{\tau - j} \left( \sum_{s = a\tau + j + 1}^{a\tau + b + j - 1} \mathop{\mathbb{E}}\left[\left\|  \left( \mathbf{J}_j - \prod_{l=s}^{k-1} \mathbf{P}_{(j,l)} \right) \left( \mathbf{g}_{(j,s)} - \mathbf{f}_{(j,s)} \right) \right\|^2 \right] \right) \nonumber\\
    & \quad\quad = \sum_{a = K / \tau - 1}^{K / \tau - 1} \sum_{b = 1}^{\tau - j} \left( \sum_{s = a\tau + j + 1}^{a\tau + b + j - 1} \mathop{\mathbb{E}}\left[\left\| \left( \mathbf{J}_j - \prod_{l=s}^{k-1} \mathbf{P}_{(j,l)} \right) \left( \mathbf{g}_{(j,s)} - \mathbf{f}_{(j,s)} \right) \right\|^2 \right] \right) \label{eq:last_short}\\
    & \quad\quad = \sum_{a = K / \tau - 1}^{K / \tau - 1} \sum_{b = 1}^{\tau - j} \left( \sum_{s = a\tau + j + 1}^{a\tau + b + j - 1} \mathop{\mathbb{E}}\left[\left\| \mathbf{g}_{(j,s)} - \mathbf{f}_{(j,s)} \right\|^2 \right] \right) \label{eq:last_zeroout} \\
    & \quad\quad \leq \sum_{a = K / \tau - 1}^{K / \tau - 1} \sum_{b = 1}^{\tau - j} \left( \sum_{s = a\tau + j + 1}^{a\tau + b + j - 1} \sigma_j^2 \right) = \sum_{b = 1}^{\tau - j} (b - 1) m \sigma_j^2 = \frac{(\tau - j)(\tau - j - 1)}{2} m \sigma_j^2, \label{eq:last_bound}
\end{align}
where (\ref{eq:last_short}) holds because $\mathbf{J}_j - \prod_{l=s}^{K - \tau + b + j - 1} \mathbf{P}_{(j,l)}$ becomes $0$ when $s \leq K - \tau + j$; (\ref{eq:last_zeroout}) holds based on Lemma \ref{lemma:zeroout}.

Based on (\ref{eq:j_bound}), (\ref{eq:middle_bound}), and (\ref{eq:last_bound}), $T_1$ is bounded as follows.
\begin{align}
    T_1 & \leq \frac{j (j - 1)}{2} m \sigma_j^2 + (\frac{K}{\tau} - 1)\frac{\tau (\tau - 1)}{2} m \sigma_j^2 + \frac{(\tau - j)(\tau - j - 1)}{2} m \sigma_j^2 \nonumber \\
    & = \frac{j (j - 1) + (\tau - j)(\tau - j - 1)}{2} m \sigma_j^2 + (\frac{K}{\tau} - 1)\frac{\tau (\tau - 1)}{2} m \sigma_j^2 \nonumber \\
    & \leq \frac{\tau (\tau - 1)}{2} m \sigma_j^2 + (\frac{K}{\tau} - 1)\frac{\tau (\tau - 1)}{2} m \sigma_j^2 \label{eq:tau}\\
    & = mK \frac{(\tau - 1)}{2} \sigma_j^2, \label{eq:t1_bound}
\end{align}
where (\ref{eq:tau}) holds because $0 < j \leq \tau$.
Here, we finish bounding $T_1$.

\textbf{Bounding $T_2$} --
Likely to $T_1$, we partition $T_2$ to three subsets and bound them separately.
We begin with the first $j$ iterations.
\begin{align}
    & \sum_{k=1}^{j} \mathop{\mathbb{E}}\left[\left\| \sum_{s=1}^{k-1} \left( \mathbf{J}_j - \prod_{l=s}^{k-1}\mathbf{P}_{(j,l)} \right) \mathbf{f}_{(j,s)} \right\|^2\right] \nonumber\\
    & \leq \sum_{k=1}^{j} (k - 1) \sum_{s=1}^{k-1} \mathop{\mathbb{E}}\left[\left\| \left( \mathbf{J}_j - \prod_{l=s}^{k-1}\mathbf{P}_{(j,l)} \right) \mathbf{f}_{(j,s)} \right\|^2\right] \label{eq:t2_j_jensen}\\
    & = \sum_{k=1}^{j} (k - 1) \sum_{s = 1}^{k - 1} \mathop{\mathbb{E}}\left[\left\| \mathbf{f}_{(j,s)} \right\|^2\right] \label{eq:t2_j_zeroout} \\
    & \leq \frac{j (j - 1)}{2} \sum_{k = 1}^{j - 1} \mathop{\mathbb{E}}\left[\left\| \mathbf{f}_{(j,k)} \right\|^2\right], \label{eq:t2_j_bound}
\end{align}
where (\ref{eq:t2_j_jensen}) holds based on the convexity of $\ell_2$ norm and Jensen's inequality; (\ref{eq:t2_j_zeroout}) holds based on Lemma \ref{lemma:zeroout}.

Then, the next $K - \tau$ iterations of $T_2$ are bounded as follows.
\begin{align}
    & \sum_{k = j + 1}^{K - \tau} \mathop{\mathbb{E}}\left[\left\| \sum_{s=1}^{k-1} (\mathbf{J}_j - \prod_{l=s}^{k-1}\mathbf{P}_{(j,l)}) \mathbf{f}_{(j,s)} \right\|^2\right] \nonumber\\
    & = \sum_{a = 0}^{K / \tau - 2} \sum_{b = 1}^{\tau} \mathop{\mathbb{E}}\left[\left\| \sum_{s=1}^{a \tau + b + j - 1} (\mathbf{J}_j - \prod_{l=s}^{a\tau + b + j - 1}\mathbf{P}_{(j,l)}) \mathbf{f}_{(j,s)} \right\|^2\right] \nonumber\\
    & = \sum_{a = 0}^{K / \tau - 2} \sum_{b = 1}^{\tau} \mathop{\mathbb{E}}\left[\left\| \sum_{s = a\tau + j + 1}^{a\tau + b + j - 1} (\mathbf{J}_j - \prod_{l=s}^{a\tau + b + j - 1}\mathbf{P}_{(j,l)}) \mathbf{f}_{(j,s)} \right\|^2\right] \label{eq:t2_middle_zeroout}\\
    & = \sum_{a = 0}^{K / \tau - 2} \sum_{b = 1}^{\tau} \left( (b - 1) \sum_{s = a\tau + j + 1}^{a\tau + b + j - 1} \mathop{\mathbb{E}}\left[\left\| (\mathbf{J}_j - \prod_{l=s}^{a\tau + b + j - 1}\mathbf{P}_{(j,l)}) \mathbf{f}_{(j,s)} \right\|^2\right] \right) \label{eq:t2_middle_jensen}\\
    & \leq \sum_{a = 0}^{K / \tau - 2} \sum_{b = 1}^{\tau} \left( (b - 1) \sum_{s = a\tau + j + 1}^{a\tau + b + j - 1} \mathop{\mathbb{E}}\left[\left\| \mathbf{f}_{(j,s)} \right\|^2 \left\| (\mathbf{J}_j - \prod_{l=s}^{a\tau + b + j - 1}\mathbf{P}_{(j,l)}) \right\|_{op}^2 \right] \right) \label{eq:t2_middle_op}\\
    & \leq \sum_{a = 0}^{K / \tau - 2} \sum_{b = 1}^{\tau} \left( (b - 1) \sum_{s = a\tau + j + 1}^{a\tau + b + j - 1} \mathop{\mathbb{E}}\left[\left\| \mathbf{f}_{(j,s)} \right\|^2 \right] \right) \label{eq:t2_middle_opout} \\
    & \leq \frac{\tau (\tau - 1)}{2} \sum_{a = 0}^{K / \tau - 2} \left( \sum_{s = a\tau + j + 1}^{a\tau + \tau + j - 1} \mathop{\mathbb{E}}\left[\left\| \mathbf{f}_{(j,s)} \right\|^2 \right] \right) \nonumber \\
    & \leq \frac{\tau (\tau - 1)}{2} \sum_{k = j + 1}^{K - \tau + j - 1} \mathop{\mathbb{E}}\left[\left\| \mathbf{f}_{(j,k)} \right\|^2 \right], \label{eq:t2_middle_bound}
\end{align}
where (\ref{eq:t2_middle_zeroout}) holds because $\mathbf{J}_j - \prod_{l=s}^{a\tau + b + j - 1}\mathbf{P}_{(j,l)}$ becomes 0 when $s \leq a\tau + j$.
(\ref{eq:t2_middle_jensen}) holds based on the convexity of $\ell_2$ norm and Jensen's inequality.
(\ref{eq:t2_middle_op}) holds based on Lemma \ref{lemma:operator}.
(\ref{eq:t2_middle_opout}) holds based on Lemma \ref{lemma:zeroout}.

Finally, the last $\tau - j$ iterations of $T_2$ are bounded as follows.
\begin{align}
    & \sum_{k = K - \tau + j + 1}^{K} \mathop{\mathbb{E}}\left[\left\| \sum_{s=1}^{k-1} \left(\mathbf{J}_j - \prod_{l=s}^{k-1}\mathbf{P}_{(j,l)} \right) \mathbf{f}_{(j,s)} \right\|^2\right] \nonumber\\
    & \quad\quad = \sum_{a = K / \tau - 1}^{K / \tau - 1} \sum_{b = 1}^{\tau - j} \left( \mathop{\mathbb{E}}\left[\left\| \sum_{s=1}^{a\tau + b + j - 1} \left( \mathbf{J}_j - \prod_{l=s}^{a\tau + b + j - 1}\mathbf{P}_{(j,l)} \right) \mathbf{f}_{(j,s)} \right\|^2 \right] \right) \nonumber \\
    & \quad\quad = \sum_{a = K / \tau - 1}^{K / \tau - 1} \sum_{b = 1}^{\tau - j} \left( \mathop{\mathbb{E}}\left[\left\| \sum_{s = a\tau + j + 1}^{a\tau + b + j - 1} \left( \mathbf{J}_j - \prod_{l=s}^{a\tau + b + j - 1}\mathbf{P}_{(j,l)} \right) \mathbf{f}_{(j,s)} \right\|^2\right] \right) \label{eq:t2_last_zeroout} \\
    & \quad\quad = \sum_{a = K / \tau - 1}^{K / \tau - 1} \sum_{b = 1}^{\tau - j} \left( (b - 1) \sum_{s = a\tau + j + 1}^{a\tau + b + j - 1} \mathop{\mathbb{E}}\left[\left\| \left( \mathbf{J}_j - \prod_{l=s}^{a\tau + b + j - 1}\mathbf{P}_{(j,l)} \right) \mathbf{f}_{(j,s)} \right\|^2\right] \right) \label{eq:t2_last_jensen} \\
    & \quad\quad \leq \sum_{a = K / \tau - 1}^{K / \tau - 1} \sum_{b = 1}^{\tau - j} \left( (b - 1) \sum_{s = a\tau + j + 1}^{a\tau + b + j - 1} \mathop{\mathbb{E}}\left[\left\| \mathbf{f}_{(j,s)} \right\|^2 \right] \right) \label{eq:t2_last_opout} \\
    & \quad\quad \leq \frac{(\tau - j)(\tau - j - 1)}{2} \sum_{a = K / \tau - 1}^{K / \tau - 1} \left( \sum_{s = a\tau + j + 1}^{a\tau + \tau - 1} \mathop{\mathbb{E}}\left[\left\| \mathbf{f}_{(j,s)} \right\|^2 \right] \right) \nonumber \\
    & \quad\quad = \frac{(\tau - j)(\tau - j - 1)}{2} \sum_{k = K - \tau + j + 1}^{K - 1} \mathop{\mathbb{E}}\left[\left\| \mathbf{f}_{(j,k)} \right\|^2 \right], \label{eq:t2_last_bound}
\end{align}
where (\ref{eq:t2_last_zeroout}) holds because $\mathbf{J}_j - \prod_{l = s}^{a\tau + b + j - 1}$ becomes $0$ when $s \leq a\tau + j$; (\ref{eq:t2_last_jensen}) holds based on the convexity of $\ell_2$ norm and Jensen's inequality; (\ref{eq:t2_last_opout}) holds based on Lemma \ref{lemma:zeroout}.

Based on (\ref{eq:t2_j_bound}), (\ref{eq:t2_middle_bound}), and (\ref{eq:t2_last_bound}), $T_2$ is bounded as follows.
\begin{align}
    T_2 & \leq \frac{j (j - 1)}{2} \sum_{k = 1}^{j - 1} \mathop{\mathbb{E}}\left[\left\| \mathbf{f}_{(j,k)} \right\|^2\right] \nonumber \\
    & \quad + \frac{\tau (\tau - 1)}{2} \sum_{k = j + 1}^{K - \tau} \mathop{\mathbb{E}}\left[\left\| \mathbf{f}_{(j,k)} \right\|^2 \right] \nonumber \\
    & \quad + \frac{(\tau - j)(\tau - j - 1)}{2} \sum_{k = K - \tau + j + 1}^{K - 1} \mathop{\mathbb{E}}\left[\left\| \mathbf{f}_{(j,k)} \right\|^2 \right] \nonumber \\
    & \leq \frac{\tau (\tau - 1)}{2} \left( \sum_{k = 1}^{K} \mathop{\mathbb{E}}\left[\left\| \mathbf{f}_{(j,k)} \right\|^2 \right] \right) \label{eq:t2_bound}
\end{align}
Here, we finish bounding $T_2$.

\textbf{Final result} --
By plugging in (\ref{eq:t1_bound}) and (\ref{eq:t2_bound}) into (\ref{eq:t1t2}), we have
\begin{align}
    & \frac{1}{mK}\sum_{k=1}^{K}\sum_{i=1}^{m} \sum_{j=1}^{\tau} L_j^2 \mathop{\mathbb{E}}[\| \mathbf{u}_{(j,k)} - \mathbf{x}_{(j,k)}^i \|^2] \nonumber \\
    & \leq \frac{2 \eta^2}{mK} \sum_{j = 1}^{\tau} L_j^2 \left( mK \frac{(\tau - 1)}{2} \sigma_j^2 + \frac{\tau (\tau - 1)}{2} \left( \sum_{k = 1}^{K} \mathop{\mathbb{E}}\left[\left\| \mathbf{f}_{(j,k)} \right\|^2 \right] \right) \right) \nonumber \\
    & = \sum_{j = 1}^{\tau} L_j^2 \left( \eta^2 (\tau - 1) \sigma_j^2 + \frac{\eta^2 \tau (\tau - 1)}{mK} \left( \sum_{k = 1}^{K} \mathop{\mathbb{E}}\left[\left\| \mathbf{f}_{(j,k)} \right\|^2 \right] \right) \right) \nonumber \\
    & = \sum_{j = 1}^{\tau} L_j^2 \left( \eta^2 (\tau - 1) \sigma_j^2 + \frac{\eta^2 \tau (\tau - 1)}{mK} \left( \sum_{i = 1}^{m} \sum_{k = 1}^{K} \mathop{\mathbb{E}}\left[\left\| \nabla_j F(\mathbf{x}_k^i)\right\|^2 \right] \right) \right)
\end{align}
Here, we complete the proof.
\end{proof}

\subsection {Proof of Other Lemmas}

\begin{lemma}
\label{lemma:operator}
Consider a real matrix $\mathbf{A} \in \mathbb{R}^{md_j \times md_j}$ and a real vector $\mathbf{b} \in \mathbb{R}^{md_j}$. If $A$ is symmetric and $\mathbf{b} \neq \mathbf{0}_{md_j}$, we have
\begin{equation}
    \| \mathbf{Ab} \| \leq \| \mathbf{A} \|_{op} \| \mathbf{b} \|
\end{equation}
\end{lemma}
\begin{proof}
\begin{align}
    \| \mathbf{Ab} \|^2 & = \frac{\| \mathbf{Ab} \|^2}{\|\mathbf{b} \|^2} \|\mathbf{b}\|^2 \nonumber \\
    & \leq \|\mathbf{A}\|_{op}^{2}\|\mathbf{b}\|^2 \label{eq:op}
\end{align}
where (\ref{eq:op}) holds based on the definition of operator norm.
\end{proof}

\begin{lemma}
\label{lemma:zeroout}
Given an identity matrix $\mathbf{I}_j \in \mathbb{R}^{md_j \times md_j}$ and a full-averaging matrix $\mathbf{J}_j \in \mathbb{R}^{md_j \times md_j}$,
\begin{align}
    \left\| \left(\mathbf{I}_{j} - \mathbf{J}_j \right) \mathbf{x} \right\|^2 \leq \left\| \mathbf{x} \right\|^2.
\end{align}
\end{lemma}
\begin{proof}
Since $\mathbf{J}_j$ is a real symmetric matrix, it can be decomposed into $\mathbf{J}_j = \mathbf{Q \Lambda Q^{\top}}$, where $\mathbf{Q}$ is a orthogonal eigenvector matrix and $\mathbf{\Lambda}$ is a diagonal eigenvalue matrix.
By the definition, $\mathbf{J}_j$ the sum of every column is $1$, and thus its eigenvalue is either $1$ or $0$.
Because $\mathbf{J}_j$ has only two different columns, $\mathbf{\Lambda} = \textrm{diag}\{ 1, 1, 0, \cdots, 0 \}$.
By the definition of the identity matrix $\mathbf{I}_j$, it can be decomposed into $\mathbf{Q \Lambda_i Q^{\top}}$, where $\mathbf{\Lambda}_i = \textrm{diag}\{1, 1, \cdots, 1 \}$.
Then, we have
\begin{align}
    \mathbf{I}_j - \mathbf{J}_j = \mathbf{Q (\Lambda_i - \Lambda) Q^{\top}}.
\end{align}
Thus, the eigenvalue matrix of $\mathbf{I}_j - \mathbf{J}_j$ is $\mathbf{\Lambda_i - \Lambda} = \textrm{diag}\{0, 0, 1, \cdots, 1 \}$.
By the definition of operator norm, 
\begin{align}
    \left\| \left( \mathbf{I}_j - \mathbf{J}_j \right) \right\|_{op} = \sqrt{ \lambda_{max}\left( (\mathbf{I}_j - \mathbf{J}_j)^{\top}(\mathbf{I}_j - \mathbf{J}_j) \right) } = \sqrt{ \lambda_{max}\left( \mathbf{I}_j - \mathbf{J}_j \right) } = 1
\end{align}
Finally, based on Lemma \ref{lemma:operator}, it follows
\begin{align}
    \left\| \left( \mathbf{I}_j - \mathbf{J}_j \right) \mathbf{x} \right\|^2 \leq \left\| \mathbf{I}_j - \mathbf{J}_j \right\|^2_{op} \left\| \mathbf{x} \right\|^2 = \left\| \mathbf{x} \right\|^2
\end{align}

\end{proof}

%% file: supplement_3_noniid.tex
\section {Convergence Analysis for Non-IID Data}

We provide the proofs of Theorems and Lemmas under non-IID data settings.

\subsection {Preliminaries}
In addition to the conventional assumptions including the smoothness of each local objective function, the unbiased local stochastic gradients, and the bounded local variance, we highlight the following assumption on the bounded dissimilarity of the gradients across clients for non-IID analysis.

\textbf{Assumption 4.}
\textit{(Bounded Dissimilarity). There exist constants $\beta_j^2 \geq 1$ and $\kappa_j^2 \geq 0$ such that $\frac{1}{m} \sum_{i = 1}^{m} \| \nabla_j F_i (\mathbf{x}) \|^2 \leq \beta_j^2 \| \frac{1}{m} \sum_{i = 1}^{m} \nabla_j F_i (\mathbf{x}) \|^2 + \kappa_j^2$.
If the data is IID, $\beta_j^2 = 1$ and $\kappa_j^2 = 0$.}

\subsection {Proof of Theorem 2}
\textbf{Theorem 2.} \textit{Suppose all $m$ local models are initialized to the same point $\mathbf{u}_1$. Under Assumption $1 \sim 4$, if Algorithm 1 runs for $K$ iterations and the learning rate satisfies $\eta \leq \frac{1}{L_{\max} } \min \left\{1, \frac{1}{ \sqrt{2\tau (\tau - 1) (2\beta^2 + 1)}} \right\}$, the average-squared gradient norm of $\mathbf{u}_k$ is bounded as follows}
\begin{align}
    \mathop{\mathbb{E}}\left[\frac{1}{K}\sum_{i=1}^{K} \| \nabla F(\mathbf{u}_k) \|^2\right] & \leq \frac{4}{\eta K}\left( \mathop{\mathbb{E}}\left[ F(\mathbf{u}_{1}) - F(\mathbf{u}_{inf}) \right] \right) + \frac{2\eta}{m} \sum_{j = 1}^{\tau} L_j \sigma_j^2 \nonumber \\
    & \quad\quad + 3 \eta^2 (\tau - 1) \sum_{j = 1}^{\tau} L_j^2 \sigma_j^2 + 6 \eta^2 \tau (\tau - 1) \sum_{j = 1}^{\tau} L_j^2 \kappa_j^2.
\end{align}

\begin{proof}
Based on Lemma \ref{lemma:framework} and \ref{lemma:niid_discrepancy}, we have
\begin{align}
    \frac{1}{K} \sum_{k=1}^{K} \mathop{\mathbb{E}}\left[ \left\| \nabla F(\mathbf{u}_k) \right\|^2 \right] & \leq \frac{2}{\eta K}\left( \mathop{\mathbb{E}}\left[ F(\mathbf{u}_{1}) - F(\mathbf{u}_{inf}) \right] \right) + \frac{\eta}{m} \sum_{j = 1}^{\tau} L_j \sigma_j^2 \nonumber \\
    & \quad\quad + \sum_{j = 1}^{\tau} L_j^2 \left( \frac{\eta^2 (\tau - 1) \sigma_j^2}{1 - A_j} + \frac{A_j \beta_j^2}{K L_j^2 (1 - A_j)} \sum_{k=1}^{K} \mathop{\mathbb{E}} \left[ \left\| \nabla_j F(\mathbf{u}_{k}) \right\|^2 \right] + \frac{A_j \kappa_j^2}{L_j^2 (1 - A_j)} \right), \nonumber
\end{align}
where $A_j = 2 \eta^2 \tau (\tau - 1) L_j^2$.
After re-writing the left-hand side and a minor rearrangement, we have
\begin{align}
    \frac{1}{K} \sum_{k=1}^{K} \sum_{j = 1}^{\tau} \mathop{\mathbb{E}}\left[ \left\| \nabla_j F(\mathbf{u}_k) \right\|^2 \right] & \leq \frac{2}{\eta K}\left( \mathop{\mathbb{E}}\left[ F(\mathbf{u}_{1}) - F(\mathbf{u}_{inf}) \right] \right) + \frac{\eta}{m} \sum_{j = 1}^{\tau} L_j \sigma_j^2 \nonumber \\
    & \quad\quad + \frac{1}{K} \sum_{k = 1}^{K} \sum_{j = 1}^{\tau} \frac{A_j \beta_j^2}{1 - A_j} \mathop{\mathbb{E}} \left[ \left\| \nabla_j F(\mathbf{u}_{k}) \right\|^2 \right] \nonumber \\
    & \quad\quad + \sum_{j = 1}^{\tau} L_j^2 \left( \frac{\eta^2 (\tau - 1) \sigma_j^2}{1 - A_j} + \frac{A_j \kappa_j^2}{L_j^2 (1 - A_j)} \right). \nonumber
\end{align}
By moving the third term on the right-hand side to the left-hand side, we have
\begin{align}
    \frac{1}{K} \sum_{k=1}^{K} \sum_{j = 1}^{\tau} \left(1 - \frac{A_j \beta_j^2}{1 - A_j} \right) \mathop{\mathbb{E}}\left[ \left\| \nabla_j F(\mathbf{u}_k) \right\|^2 \right] & \leq \frac{2}{\eta K}\left( \mathop{\mathbb{E}}\left[ F(\mathbf{u}_{1}) - F(\mathbf{u}_{inf}) \right] \right) + \frac{\eta}{m} \sum_{j = 1}^{\tau} L_j \sigma_j^2 \nonumber \\
    & \quad\quad + \sum_{j = 1}^{\tau} L_j^2 \left( \frac{\eta^2 (\tau - 1) \sigma_j^2}{1 - A_j} + \frac{A_j \kappa_j^2}{L_j^2 (1 - A_j)} \right). \label{eq:niid_almost}
\end{align}
If $A_j \leq \frac{1}{2 \beta_j^2 + 1}$, then $\frac{A_j \beta_j^2}{1 - A_j} \leq \frac{1}{2}$.
Therefore, (\ref{eq:niid_almost}) can be simplified as follows.
\begin{align}
    \frac{1}{K} \sum_{k=1}^{K} \sum_{j = 1}^{\tau} \mathop{\mathbb{E}}\left[ \left\| \nabla_j F(\mathbf{u}_k) \right\|^2 \right] & \leq \frac{4}{\eta K}\left( \mathop{\mathbb{E}}\left[ F(\mathbf{u}_{1}) - F(\mathbf{u}_{inf}) \right] \right) + \frac{2\eta}{m} \sum_{j = 1}^{\tau} L_j \sigma_j^2 \nonumber \\
    & \quad\quad + 2 \sum_{j = 1}^{\tau} L_j^2 \left( \frac{\eta^2 (\tau - 1) \sigma_j^2}{1 - A_j} \right) + 2 \sum_{j = 1}^{\tau} \frac{A_j \kappa_j^2}{1 - A_j}. \nonumber
\end{align}

The learning rate condition $A_j \leq \frac{1}{2 \beta_j^2 + 1}$ also ensures that $\frac{1}{1 - A_j} \leq 1 + \frac{1}{2 \beta_j^2}$.
Based on Assumption 4, $\frac{1}{2 \beta_j^2} \leq \frac{2}{3}$, and thus $\frac{1}{1 - A_j} \leq \frac{2}{3}$.
Therefore, we have
\begin{align}
    \frac{1}{K} \sum_{k=1}^{K} \sum_{j = 1}^{\tau} \mathop{\mathbb{E}}\left[ \left\| \nabla_j F(\mathbf{u}_k) \right\|^2 \right] & \leq \frac{4}{\eta K}\left( \mathop{\mathbb{E}}\left[ F(\mathbf{u}_{1}) - F(\mathbf{u}_{inf}) \right] \right) + \frac{2\eta}{m} \sum_{j = 1}^{\tau} L_j \sigma_j^2 \nonumber \\
    & \quad\quad + 3 \sum_{j = 1}^{\tau} L_j^2 \eta^2 (\tau - 1) \sigma_j^2 + 6 \sum_{j = 1}^{\tau} \eta^2 \tau (\tau - 1) L_j^2 \kappa_j^2 \nonumber \\
    & = \frac{4}{\eta K}\left( \mathop{\mathbb{E}}\left[ F(\mathbf{u}_{1}) - F(\mathbf{u}_{inf}) \right] \right) + \frac{2\eta}{m} \sum_{j = 1}^{\tau} L_j \sigma_j^2 \nonumber \\
    & \quad\quad + 3 \eta^2 (\tau - 1) \sum_{j = 1}^{\tau} L_j^2 \sigma_j^2 + 6 \eta^2 \tau (\tau - 1) \sum_{j = 1}^{\tau} L_j^2 \kappa_j^2. \nonumber
\end{align}
We complete the proof.
\end{proof}

\textbf{Learning rate constraints} --
We have two learning rate constraints as follows.
\begin{align*}
    \eta &\leq \frac{1}{L_{\max}} \quad\quad &\text{Lemma \ref{lemma:framework}}\\
    2 \eta^2 \tau (\tau - 1) L_j^2 &\leq \frac{1}{2\beta_j^2 + 1} & \text{Theorem 2} \\
\end{align*}
By merging the two constraints, we can have a single learning rate constraint as follows.
\begin{equation}
    \eta \leq \frac{1}{L_{\max} } \min \left\{1, \frac{1}{ \sqrt{2\tau (\tau - 1) (2\beta^2 + 1)}} \right\}
\end{equation}

\subsection {Proof of Lemma 3}

\textbf{Lemma 3.}
\textit{(model discrepancy) Under Assumption $1 \sim 4$, local SGD with the partial model averaging ensures}
\begin{align}
    & \frac{1}{mK} \sum_{k = 1}^{K} \sum_{i=1}^{m} \mathop{\mathbb{E}} \left[ \left\| \mathbf{u}_{(j,k)} - \mathbf{x}_{(j,k)}^{i} \right\|^2 \right] \nonumber \\
    & \quad\quad \leq \frac{\eta^2 (\tau - 1) \sigma_j^2}{1 - A_j} + \frac{A_j \beta_j^2}{K L_j^2 (1 - A_j)} \sum_{k=1}^{K} \mathop{\mathbb{E}} \left[ \left\| \nabla_j F(\mathbf{u}_{k}) \right\|^2 \right] + \frac{A_j \kappa_j^2}{L_j^2 (1 - A_j)}
\end{align}
where $A_j = 2 \eta^2 \tau (\tau - 1) L_j^2$.

\begin{proof}
We begin with re-writing the weighted average of the squared distance using the vectorized form of the local models as follows.
\begin{align}
    \frac{1}{m} \sum_{i=1}^{m} \left\| \mathbf{u}_{(j,k)} - \mathbf{x}_{(j,k)}^{i} \right\|^2 = \left\| \mathbf{J}_j \mathbf{x}_{(j,k)} - \mathbf{x}_{(j,k)} \right\|^2 = \left\| (\mathbf{J}_j - \mathbf{I}_j)\mathbf{x}_{(j,k)} \right\|^2 \label{eq:niid_discrepancy}
\end{align}

Then, according to the parameter update rule, we have
\begin{align}
    (\mathbf{J}_j - \mathbf{I}_j)\mathbf{x}_{(j,k)} & = (\mathbf{J}_j - \mathbf{I}_j)\mathbf{P}_{(j,k-1)}(\mathbf{x}_{(j,k-1)} - \eta \mathbf{g}_{(j,k-1)})\\
    & = (\mathbf{J}_j - \mathbf{I}_j) \mathbf{P}_{(j,k-1)} \mathbf{x}_{(j,k-1)} - (\mathbf{J}_j - \mathbf{P}_{(j,k-1)}) \eta \mathbf{g}_{(j,k-1)},
\end{align}
where the second equality holds because $\mathbf{J}_j \mathbf{P}_j = \mathbf{J}_j$ and $\mathbf{I}_j \mathbf{P}_j = \mathbf{P}_j$.

Then, expanding the expression of $\mathbf{x}_{(j,k-1)}$, we have
\begin{align}
    (\mathbf{J}_j - \mathbf{I}_j) \mathbf{x}_{(j,k)} & = (\mathbf{J}_j - \mathbf{I}_j) \mathbf{P}_{(j,k-1)} (\mathbf{P}_{(j,k-2)} (\mathbf{x}_{(j,k-2)} - \eta \mathbf{g}_{(j,k-2)})) - (\mathbf{J}_j - \mathbf{P}_{(j,k-1)}) \eta \mathbf{g}_{(j,k-1)} \nonumber \\
    & = (\mathbf{J}_j - \mathbf{I}_j) \mathbf{P}_{(j,k-1)} \mathbf{P}_{(j,k-2)} \mathbf{x}_{(j,k-2)} - (\mathbf{J}_j - \mathbf{P}_{(j,k-1)} \mathbf{P}_{(j,k-2)}) \eta \mathbf{g}_{(j,k-2)} - (\mathbf{J}_j - \mathbf{P}_{(j,k-1)}) \eta \mathbf{g}_{(j,k-1)}. \nonumber
\end{align}

Repeating the same procedure for $\mathbf{x}_{(j,k-2)}$, $\mathbf{x}_{(j,k-3)}$, $\cdots$, $\mathbf{x}_{(j,2)}$, we have
\begin{align}
    (\mathbf{J}_j - \mathbf{I}_j) \mathbf{\hat{x}}_{(j,k)} & = (\mathbf{J}_j - \mathbf{I}_j) \prod_{s=1}^{k-1} \mathbf{P}_{(j,s)} \mathbf{x}_{(j,1)} - \eta \sum_{s=1}^{k-1} (\mathbf{J}_j - \prod_{l=s}^{k-1}\mathbf{P}_{(j,l)}) \mathbf{g}_{(j,s)} \nonumber \\
    & = - \eta \sum_{s=1}^{k-1} (\mathbf{J}_j - \prod_{l=s}^{k-1} \mathbf{P}_{(j,l)}) \mathbf{g}_{(j,s)}, \label{eq:niid_distance}
\end{align}
where (\ref{eq:niid_distance}) holds because $\mathbf{x}_{(j,1)}^i$ is the same across all the workers and thus $(\mathbf{J}_j - \mathbf{I}_j)\mathbf{x}_{(j,1)} = 0$.

Based on (\ref{eq:niid_distance}), we have
\begin{align}
    & \frac{1}{mK} \sum_{k = 1}^{K} \sum_{i=1}^{m} \mathop{\mathbb{E}} \left[ \left\| \mathbf{u}_{(j,k)} - \mathbf{x}_{(j,k)}^{i} \right\|^2 \right] \nonumber \\
    & \quad\quad = \frac{1}{K} \sum_{k = 1}^{K} \left( \mathop{\mathbb{E}} \left[ \left\| \left( \mathbf{J}_j - \mathbf{I}_j \right) \mathbf{x}_{(j,k)} \right\|^2 \right] \right) \nonumber \\
    & \quad\quad = \frac{1}{K} \sum_{k = 1}^{K} \left( \eta^2 \mathop{\mathbb{E}} \left[ \left\| \sum_{s=1}^{k-1} \left( \mathbf{J}_j - \prod_{l=s}^{k-1} \mathbf{P}_{(j,l)} \right) \mathbf{g}_{(j,s)} \right\|^2 \right] \right) \nonumber \\
    & \quad\quad = \frac{1}{K} \sum_{k = 1}^{K} \left( \eta^2 \mathop{\mathbb{E}} \left[ \left\| \sum_{s=1}^{k-1} \left( \mathbf{J}_j - \prod_{l=s}^{k-1} \mathbf{P}_{(j,l)} \right) \left( \mathbf{g}_{(j,s)} - \mathbf{f}_{(j,s)} \right) + \sum_{s=1}^{k-1} \left( \mathbf{J}_j - \prod_{l=s}^{k-1} \mathbf{P}_{(j,l)} \right) \mathbf{f}_{(j,s)} \right\|^2 \right] \right) \nonumber \\
    & \quad\quad \leq \frac{2 \eta^2}{K} \left( \underset{T_3}{\underbrace{ \sum_{k = 1}^{K} \mathop{\mathbb{E}} \left[ \left\| \sum_{s=1}^{k-1} \left( \mathbf{J}_j - \prod_{l=s}^{k-1} \mathbf{P}_{(j,l)} \right) \left( \mathbf{g}_{(j,s)} - \mathbf{f}_{(j,s)} \right) \right\|^2 \right] }} + \underset{T_4}{\underbrace{ \sum_{k = 1}^{K} \mathop{\mathbb{E}} \left[ \left\| \sum_{s=1}^{k-1} \left( \mathbf{J}_j - \prod_{l=s}^{k-1} \mathbf{P}_{(j,l)} \right) \mathbf{f}_{(j,s)} \right\|^2 \right] }} \right) \label{eq:niid_t3t4}
\end{align}
where (\ref{eq:niid_t3t4}) holds based on the convexity of $\ell_2$ norm and Jensen's inequality.
Now, we focus on bounding $T_3$ and $T_4$, separately.

\textbf{Bounding $T_3$} --
We first partition $K$ iterations into three subsets: the first $j$ iterations, the next $K - \tau$ iterations, and the final $\tau - j$ iterations.
We bound the first $j$ iterations of $T_3$ as follows.
\begin{align}
    \sum_{k = 1}^{j} \mathop{\mathbb{E}} \left[ \left\| \sum_{s=1}^{k-1} \left(\mathbf{J}_j - \prod_{l=s}^{k-1} \mathbf{P}_{(j,l)} \right) \left( \mathbf{g}_{(j,s)} - \mathbf{f}_{(j,s)} \right) \right\|^2 \right] & \leq \sum_{k = 1}^{j} \sum_{s = 1}^{k - 1} \mathop{\mathbb{E}} \left[ \left\| \left( \mathbf{J}_j - \prod_{l=s}^{k-1} \mathbf{P}_{(j,l)} \right) \left( \mathbf{g}_{(j,s)} - \mathbf{f}_{(j,s)} \right) \right\|^2 \right] \label{eq:niid_t3_jensen} \\
    & \leq \sum_{k = 1}^{j} \sum_{s = 1}^{k - 1} \mathop{\mathbb{E}} \left[ \left\| \left( \mathbf{J}_j - \mathbf{I}_j \right) \left( \mathbf{g}_{(j,s)} - \mathbf{f}_{(j,s)} \right) \right\|^2 \right] \label{eq:niid_t3_zeroout} \\
    & \leq \sum_{k = 1}^{j} \sum_{s = 1}^{k - 1} \mathop{\mathbb{E}} \left[ \left\| \mathbf{g}_{(j,s)} - \mathbf{f}_{(j,s)} \right\|^2 \right] \label{eq:niid_t3_opout} \\
    & = \sum_{k = 1}^{j} \sum_{s = 1}^{k - 1} \frac{1}{m} \sum_{i = 1}^{m} \mathop{\mathbb{E}} \left[ \left\| \mathbf{g}_{(j,s)}^{i} - \nabla_j F_i(\mathbf{x}_{s}^i) \right\|^2 \right] \nonumber \\
    & \leq \sum_{k = 1}^{j} \sum_{s = 1}^{k - 1} \sigma_j^2 \nonumber \\
    & = \frac{j (j - 1)}{2} \sigma_j^2, \label{eq:niid_t3_j_bound}
\end{align}
where (\ref{eq:niid_t3_zeroout}) holds because $\prod_{l=s}^{k-1} \mathbf{P}_{(j,l)}$ is $\mathbf{I}_j$ when $k < j$;
(\ref{eq:niid_t3_opout}) holds based on Lemma \ref{lemma:zeroout}.

Then, the next $K - \tau$ iterations of $T_3$ can be bounded as follows.
\begin{align}
    & \sum_{k = j + 1}^{K - \tau + j} \mathop{\mathbb{E}} \left[ \left\| \sum_{s=1}^{k-1} \left( \mathbf{J}_j - \prod_{l=s}^{k-1} \mathbf{P}_{(j,l)} \right) \left( \mathbf{g}_{(j,s)} - \mathbf{f}_{(j,s)} \right) \right\|^2 \right] \nonumber \\
    & \quad\quad = \sum_{k = j + 1}^{K - \tau + j} \sum_{s = 1}^{k - 1} \mathop{\mathbb{E}} \left[ \left\| \left( \mathbf{J}_j - \prod_{l=s}^{k-1} \mathbf{P}_{(j,l)} \right) \left( \mathbf{g}_{(j,s)} - \mathbf{f}_{(j,s)} \right) \right\|^2 \right] \label{eq:niid_t3_middle_jensen}
\end{align}

Replacing $k$ with $a\tau + b + j$, (\ref{eq:niid_t3_middle_jensen}) can be re-written and bounded as follows.
\begin{align}
    & \sum_{a = 0}^{K / \tau - 2} \sum_{b = 1}^{\tau} \sum_{s = 1}^{a\tau + b + j - 1} \mathop{\mathbb{E}} \left[ \left\| \left( \mathbf{J}_j - \prod_{l=s}^{k-1} \mathbf{P}_{(j,l)} \right) \left(\mathbf{g}_{(j,s)} - \mathbf{f}_{(j,s)} \right) \right\|^2 \right] \nonumber\\
    & \quad\quad = \sum_{a = 0}^{K / \tau - 2} \sum_{b = 1}^{\tau} \sum_{s = 1}^{a\tau + j} \mathop{\mathbb{E}}\left[ \left\| \left( \mathbf{J}_j - \prod_{l=s}^{k-1} \mathbf{P}_{(j,l)} \right) \left( \mathbf{g}_{(j,s)} - \mathbf{f}_{(j,s)} \right) \right\|^2 \right] \nonumber \\
    & \quad\quad\quad\quad + \sum_{a = 0}^{K / \tau - 2} \sum_{b = 1}^{\tau} \sum_{s = a\tau + j + 1}^{a\tau + b + j - 1} \mathop{\mathbb{E}}\left[ \left\| \left( \mathbf{J}_j - \prod_{l=s}^{k-1} \mathbf{P}_{(j,l)} \right) \left( \mathbf{g}_{(j,s)} - \mathbf{f}_{(j,s)} \right) \right\|^2 \right] \nonumber\\
    & \quad\quad = \sum_{a = 0}^{K / \tau - 2} \sum_{b = 1}^{\tau} \sum_{s = a\tau + j + 1}^{a\tau + b + j - 1} \mathop{\mathbb{E}}\left[ \left\| \left( \mathbf{J}_j - \prod_{l=s}^{k-1} \mathbf{P}_{(j,l)} \right) \left( \mathbf{g}_{(j,s)} - \mathbf{f}_{(j,s)} \right) \right\|^2 \right] \label{eq:niid_t3_middle_zeroout} \\
    & \quad\quad \leq \sum_{a = 0}^{K / \tau - 2} \sum_{b = 1}^{\tau} \sum_{s = a\tau + j + 1}^{a\tau + b + j - 1} \mathop{\mathbb{E}}\left[ \left\| \mathbf{g}_{(j,s)} - \mathbf{f}_{(j,s)} \right\|^2 \right] \label{eq:niid_t3_middle_opout} \\
    & \quad\quad = \sum_{a = 0}^{K / \tau - 2} \sum_{b = 1}^{\tau} \sum_{s = a\tau + j + 1}^{a\tau + b + j - 1} \frac{1}{m} \sum_{i = 1}^{m} \mathop{\mathbb{E}}\left[ \left\| \mathbf{g}_{(j,s)}^{i} - \nabla_j F_i (\mathbf{x}_{s}^{i}) \right\|^2 \right] \nonumber\\
    & \quad\quad \leq \sum_{a = 0}^{K / \tau - 2} \sum_{b = 1}^{\tau} \sum_{s = a\tau + j + 1}^{a\tau + b + j - 1} \sigma_j^2 \label{eq:niid_t3_middle_sigma} \\
    & \quad\quad = \sum_{a = 0}^{K / \tau - 2} \sum_{b = 1}^{\tau} (b - 1) \sigma_j^2 \nonumber \\
    & \quad\quad = \sum_{a = 0}^{K / \tau - 2} \frac{\tau (\tau - 1)}{2} \sigma_j^2 \nonumber \\
    & \quad\quad = \left( \frac{K}{\tau} - 1 \right) \frac{\tau (\tau - 1)}{2} \sigma_j^2, \label{eq:niid_t3_middle_bound}
\end{align}
where (\ref{eq:middle_zeroout}) holds because $\mathbf{J}_j - \prod_{l=s}^{a\tau + b + j - 1} \mathbf{P}_{(j,l)}$ becomes $0$ when $s \leq a\tau + j$;
(\ref{eq:niid_t3_middle_opout}) holds based on Lemma \ref{lemma:zeroout}.
(\ref{eq:niid_t3_middle_sigma}) holds based on Assumption 6.

Finally, the last $\tau - j$ iterations are bounded as follows.
\begin{align}
    & \sum_{k = K - \tau + j + 1}^{K} \mathop{\mathbb{E}} \left[ \left\| \sum_{s=1}^{k-1} \left( \mathbf{J}_j - \prod_{l=s}^{k-1} \mathbf{P}_{(j,l)} \right) \left( \mathbf{g}_{(j,s)} - \mathbf{f}_{(j,s)} \right) \right\|^2 \right] \nonumber\\
    & \quad\quad = \sum_{k = K - \tau + j + 1}^{K} \sum_{s=1}^{k-1} \mathop{\mathbb{E}}\left[\left\| \left( \mathbf{J}_j - \prod_{l=s}^{k-1} \mathbf{P}_{(j,l)} \right) \left( \mathbf{g}_{(j,s)} - \mathbf{f}_{(j,s)} \right) \right\|^2\right] \label{eq:niid_t3_last_jensen}
\end{align}
where (\ref{eq:niid_t3_last_jensen}) holds because $\mathbf{g}_{(j,s)} - \mathbf{f}_{(j,s)}$ has a mean of 0 and independent across $s$.

By replacing $k$ with $a\tau + b + j$ in (\ref{eq:niid_t3_last_jensen}), we have
\begin{align}
    & \sum_{a = K / \tau - 1}^{K / \tau - 1} \sum_{b = 1}^{\tau - j} \left(\sum_{s = 1}^{a\tau + b + j - 1} \mathop{\mathbb{E}}\left[\left\| \left( \mathbf{J}_j - \prod_{l=s}^{k-1} \mathbf{P}_{(j,l)} \right) \left( \mathbf{g}_{(j,s)} - \mathbf{f}_{(j,s)} \right) \right\|^2 \right] \right) \nonumber \\
    & \quad\quad = \sum_{a = K / \tau - 1}^{K / \tau - 1} \sum_{b = 1}^{\tau - j} \left( \sum_{s = 1}^{a\tau + j} \mathop{\mathbb{E}}\left[\left\| \left( \mathbf{J}_j - \prod_{l=s}^{k-1} \mathbf{P}_{(j,l)} \right) \left( \mathbf{g}_{(j,s)} - \mathbf{f}_{(j,s)} \right) \right\|^2 \right] \right) \nonumber \\
    & \quad\quad\quad\quad + \sum_{a = K / \tau - 1}^{K / \tau - 1} \sum_{b = 1}^{\tau - j} \left( \sum_{s = a\tau + j + 1}^{a\tau + b + j - 1} \mathop{\mathbb{E}}\left[\left\| \left( \mathbf{J}_j - \prod_{l=s}^{k-1} \mathbf{P}_{(j,l)} \right) \left( \mathbf{g}_{(j,s)} - \mathbf{f}_{(j,s)} \right) \right\|^2 \right] \right) \nonumber \\
    & \quad\quad = \sum_{a = K / \tau - 1}^{K / \tau - 1} \sum_{b = 1}^{\tau - j} \left( \sum_{s = a\tau + j + 1}^{a\tau + b + j - 1} \mathop{\mathbb{E}}\left[\left\| \left( \mathbf{J}_j - \prod_{l=s}^{k-1} \mathbf{P}_{(j,l)} \right) \left( \mathbf{g}_{(j,s)} - \mathbf{f}_{(j,s)} \right) \right\|^2 \right] \right) \label{eq:niid_t3_last_zeroout} \\
    & \quad\quad = \sum_{a = K / \tau - 1}^{K / \tau - 1} \sum_{b = 1}^{\tau - j} \left( \sum_{s = a\tau + j + 1}^{a\tau + b + j - 1} \mathop{\mathbb{E}}\left[\left\| \mathbf{g}_{(j,s)} - \mathbf{f}_{(j,s)} \right\|^2 \right] \right) \label{eq:niid_t3_last_opout}\\
    & \quad\quad = \sum_{a = K / \tau - 1}^{K / \tau - 1} \sum_{b = 1}^{\tau - j} \left( \sum_{s = a\tau + j + 1}^{a\tau + b + j - 1} \frac{1}{m} \sum_{i = 1}^{m} \mathop{\mathbb{E}}\left[\left\| (\mathbf{g}_{(j,s)}^{i} - \nabla_j F_i(\mathbf{x}_{(j,s)}^i) \right\|^2 \right] \right) \nonumber\\
    & \quad\quad \leq \sum_{a = K / \tau - 1}^{K / \tau - 1} \sum_{b = 1}^{\tau - j} \sum_{s = a\tau + j + 1}^{a\tau + b + j - 1} \sigma_j^2 \label{eq:niid_t3_last_sigma} \\
    & \quad\quad = \sum_{b = 1}^{\tau - j} (b - 1) \sigma_j^2 = \frac{(\tau - j)(\tau - j - 1)}{2} \sigma_j^2, \label{eq:niid_t3_last_bound}
\end{align}
where (\ref{eq:niid_t3_last_zeroout}) holds because $\mathbf{J}_j - \prod_{l=s}^{K - \tau + b + j - 1} \mathbf{P}_{(j,l)}$ becomes $0$ when $s \leq K - \tau + j$;
(\ref{eq:niid_t3_last_opout}) holds based on Lemma \ref{lemma:zeroout};
(\ref{eq:niid_t3_last_sigma}) holds based on Assumption 6.

Summing up (\ref{eq:niid_t3_j_bound}), (\ref{eq:niid_t3_middle_bound}), and (\ref{eq:niid_t3_last_bound}), we have
\begin{align}
    T_3 & \leq \frac{j (j - 1)}{2} \sigma_j^2 + (\frac{K}{\tau} - 1)\frac{\tau (\tau - 1)}{2} \sigma_j^2 + \frac{(\tau - j)(\tau - j - 1)}{2} \sigma_j^2 \nonumber \\
    & = \frac{j (j - 1) + (\tau - j)(\tau - j - 1)}{2} \sigma_j^2 + (\frac{K}{\tau} - 1)\frac{\tau (\tau - 1)}{2} \sigma_j^2 \nonumber \\
    & \leq \frac{\tau (\tau - 1)}{2} m \sigma_j^2 + (\frac{K}{\tau} - 1)\frac{\tau (\tau - 1)}{2} \sigma_j^2 \label{eq:niid_t3_tau}\\
    & = K \frac{(\tau - 1)}{2} \sigma_j^2, \label{eq:niid_t3_bound}
\end{align}
where (\ref{eq:niid_t3_tau}) holds because $0 < j \leq \tau$.
Here, we finish bounding $T_3$.

\textbf{Bounding $T_4$} --
Likely to $T_3$, we partition $T_4$ to three subsets and bound them separately.
The $T_4$ at the first $j$ iterations can be bounded as follows.
\begin{align}
    \sum_{k = 1}^{j} \mathop{\mathbb{E}} \left[ \left\| \sum_{s=1}^{k-1} \left( \mathbf{J}_j - \prod_{l=s}^{k-1} \mathbf{P}_{(j,l)} \right) \mathbf{f}_{(j,s)} \right\|^2 \right] & \leq \sum_{k = 1}^{j} (k - 1) \sum_{s=1}^{k-1} \mathop{\mathbb{E}} \left[ \left\| \left( \mathbf{J}_j - \prod_{l=s}^{k-1} \mathbf{P}_{(j,l)} \right) \mathbf{f}_{(j,s)} \right\|^2 \right] \nonumber \\
    & = \sum_{k = 1}^{j} (k - 1) \sum_{s=1}^{k-1} \mathop{\mathbb{E}} \left[ \left\| \left( \mathbf{J}_j - \mathbf{I}_j \right) \mathbf{f}_{(j,s)} \right\|^2 \right] \label{eq:niid_t4_j_op} \\
    & \leq \sum_{k = 1}^{j} (k - 1) \sum_{s=1}^{k-1} \mathop{\mathbb{E}} \left[ \left\| \mathbf{f}_{(j,s)} \right\|^2 \right] \label{eq:niid_t4_j_opout} \\
    & = \frac{j (j - 1)}{2} \sum_{k=1}^{j-1} \mathop{\mathbb{E}} \left[ \left\| \mathbf{f}_{(j,k)} \right\|^2 \right], \label{eq:niid_t4_j_bound}
\end{align}
where (\ref{eq:niid_t4_j_op}) holds because $\prod_{l=1}^{k-1} \mathbf{P}_{(j,l)}$ is $\mathbf{I}_j$ when $k < j$; (\ref{eq:niid_t4_j_opout}) holds based on Lemma \ref{lemma:zeroout}.

Then, the next $K - \tau$ iterations of $T_2$ are bounded as follows.
\begin{align}
    \sum_{k = j + 1}^{K - \tau} \mathop{\mathbb{E}}\left[\left\| \sum_{s=1}^{k-1} \left( \mathbf{J}_j - \prod_{l=s}^{k-1}\mathbf{P}_{(j,l)} \right) \mathbf{f}_{(j,s)} \right\|^2\right] & = \sum_{a = 0}^{K / \tau - 2} \sum_{b = 1}^{\tau} \mathop{\mathbb{E}}\left[\left\| \sum_{s=1}^{a \tau + b + j - 1} \left( \mathbf{J}_j - \prod_{l=s}^{a\tau + b + j - 1}\mathbf{P}_{(j,l)} \right) \mathbf{f}_{(j,s)} \right\|^2\right] \nonumber\\
    & = \sum_{a = 0}^{K / \tau - 2} \sum_{b = 1}^{\tau} \mathop{\mathbb{E}}\left[\left\| \sum_{s = a\tau + j + 1}^{a\tau + b + j - 1} \left( \mathbf{J}_j - \prod_{l=s}^{a\tau + b + j - 1}\mathbf{P}_{(j,l)} \right) \mathbf{f}_{(j,s)} \right\|^2\right] \label{eq:t4_middle_zeroout}\\
    & = \sum_{a = 0}^{K / \tau - 2} \sum_{b = 1}^{\tau} \left( (b - 1) \sum_{s = a\tau + j + 1}^{a\tau + b + j - 1} \mathop{\mathbb{E}}\left[ \left\| \left( \mathbf{J}_j - \prod_{l=s}^{a\tau + b + j - 1}\mathbf{P}_{(j,l)} \right) \mathbf{f}_{(j,s)} \right\|^2\right] \right) \label{eq:t4_middle_jensen}\\
    & = \sum_{a = 0}^{K / \tau - 2} \sum_{b = 1}^{\tau} \left( (b - 1) \sum_{s = a\tau + j + 1}^{a\tau + b + j - 1} \mathop{\mathbb{E}}\left[ \left\| \left( \mathbf{J}_j - \mathbf{I}_j \right) \mathbf{f}_{(j,s)} \right\|^2\right] \right) \nonumber \\
    & \leq \sum_{a = 0}^{K / \tau - 2} \sum_{b = 1}^{\tau} \left( (b - 1) \sum_{s = a\tau + j + 1}^{a\tau + b + j - 1} \mathop{\mathbb{E}}\left[ \left\| \mathbf{f}_{(j,s)} \right\|^2\right] \right) \label{eq:t4_middle_short}\\
    & \leq \frac{\tau (\tau - 1)}{2} \sum_{a = 0}^{K / \tau - 2} \left( \sum_{s = a\tau + j + 1}^{a\tau + \tau + j - 1} \mathop{\mathbb{E}}\left[\left\| \mathbf{f}_{(j,s)} \right\|^2 \right] \right) \nonumber \\
    & \leq \frac{\tau (\tau - 1)}{2} \sum_{k = j + 1}^{K - \tau + j - 1} \mathop{\mathbb{E}}\left[\left\| \mathbf{f}_{(j,k)} \right\|^2 \right], \label{eq:t4_middle_bound}
\end{align}
where (\ref{eq:t4_middle_zeroout}) holds because $\mathbf{J}_j - \prod_{l=s}^{a\tau + b + j - 1}\mathbf{P}_{(j,l)}$ becomes 0 when $s \leq a\tau + j$;
(\ref{eq:t4_middle_jensen}) holds based on the convexity of $\ell_2$ norm and Jensen's inequality;
(\ref{eq:t4_middle_short}) holds based on Lemma \ref{lemma:zeroout}.

Finally, the last $\tau - j$ iterations of $T_4$ are bounded as follows.
\begin{align}
    \sum_{k = K - \tau + j + 1}^{K} \mathop{\mathbb{E}}\left[\left\| \sum_{s=1}^{k-1} \left( \mathbf{J}_j - \prod_{l=s}^{k-1}\mathbf{P}_{(j,l)} \right) \mathbf{f}_{(j,s)} \right\|^2\right] & = \sum_{a = K / \tau - 1}^{K / \tau - 1} \sum_{b = 1}^{\tau - j} \left( \mathop{\mathbb{E}}\left[\left\| \sum_{s=1}^{a\tau + b + j - 1} \left( \mathbf{J}_j - \prod_{l=s}^{a\tau + b + j - 1}\mathbf{P}_{(j,l)} \right) \mathbf{f}_{(j,s)} \right\|^2\right] \right) \nonumber \\
    & = \sum_{a = K / \tau - 1}^{K / \tau - 1} \sum_{b = 1}^{\tau - j} \left( \mathop{\mathbb{E}}\left[\left\| \sum_{s = a\tau + j + 1}^{a\tau + b + j - 1} \left( \mathbf{J}_j - \prod_{l=s}^{a\tau + b + j - 1}\mathbf{P}_{(j,l)} \right) \mathbf{f}_{(j,s)} \right\|^2\right] \right) \label{eq:t4_last_zeroout} \\
    & = \sum_{a = K / \tau - 1}^{K / \tau - 1} \sum_{b = 1}^{\tau - j} \left( (b - 1) \sum_{s = a\tau + j + 1}^{a\tau + b + j - 1} \mathop{\mathbb{E}}\left[\left\| \left( \mathbf{J}_j - \prod_{l=s}^{a\tau + b + j - 1}\mathbf{P}_{(j,l)} \right) \mathbf{f}_{(j,s)} \right\|^2\right] \right) \label{eq:t4_last_jensen} \\
    & = \sum_{a = K / \tau - 1}^{K / \tau - 1} \sum_{b = 1}^{\tau - j} \left( (b - 1) \sum_{s = a\tau + j + 1}^{a\tau + b + j - 1} \mathop{\mathbb{E}}\left[\left\| \left( \mathbf{J}_j - \mathbf{I}_j \right) \mathbf{f}_{(j,s)} \right\|^2 \right] \right) \nonumber \\
    & \leq \sum_{a = K / \tau - 1}^{K / \tau - 1} \sum_{b = 1}^{\tau - j} \left( (b - 1) \sum_{s = a\tau + j + 1}^{a\tau + b + j - 1} \mathop{\mathbb{E}}\left[\left\| \mathbf{f}_{(j,s)} \right\|^2 \right] \right) \label{eq:t4_last_opout} \\
    & \leq \frac{(\tau - j)(\tau - j - 1)}{2} \sum_{a = K / \tau - 1}^{K / \tau - 1} \left( \sum_{s = a\tau + j + 1}^{a\tau + \tau - 1} \mathop{\mathbb{E}}\left[\left\| \mathbf{f}_{(j,s)} \right\|^2 \right] \right) \nonumber \\
    & = \frac{(\tau - j)(\tau - j - 1)}{2} \sum_{k = K - \tau + j + 1}^{K - 1} \mathop{\mathbb{E}}\left[\left\| \mathbf{f}_{(j,k)} \right\|^2 \right], \label{eq:t4_last_bound}
\end{align}
where (\ref{eq:t4_last_zeroout}) holds because $\mathbf{J}_j - \prod_{l = s}^{a\tau + b + j - 1}$ becomes $0$ when $s \leq a\tau + j$;
(\ref{eq:t4_last_jensen}) holds based on the convexity of $\ell_2$ norm and Jensen's inequality;
(\ref{eq:t4_last_opout}) holds based on Lemma \ref{lemma:zeroout}.

Based on (\ref{eq:niid_t4_j_bound}), (\ref{eq:t4_middle_bound}), and (\ref{eq:t4_last_bound}), $T_4$ is bounded as follows.
\begin{align}
    T_4 & \leq \frac{j (j - 1)}{2} \sum_{k=1}^{j-1} \mathop{\mathbb{E}} \left[ \left\| \mathbf{f}_{(j,k)} \right\|^2 \right] + \frac{\tau (\tau - 1)}{2} \sum_{k = j + 1}^{K - \tau + j - 1} \mathop{\mathbb{E}}\left[\left\| \mathbf{f}_{(j,k)} \right\|^2 \right] \nonumber \\
    & \quad\quad + \frac{(\tau - j)(\tau - j - 1)}{2} \sum_{k = K - \tau + j + 1}^{K - 1} \mathop{\mathbb{E}}\left[\left\| \mathbf{f}_{(j,k)} \right\|^2 \right] \nonumber \\
    & \leq \frac{\tau (\tau - 1)}{2} \left( \sum_{k=1}^{j-1} \mathop{\mathbb{E}} \left[ \left\| \mathbf{f}_{(j,k)} \right\|^2 \right] + \sum_{k = j + 1}^{K - \tau + j - 1} \mathop{\mathbb{E}}\left[\left\| \mathbf{f}_{(j,k)} \right\|^2 \right] + \sum_{k = K - \tau + j + 1}^{K - 1} \mathop{\mathbb{E}}\left[\left\| \mathbf{f}_{(j,k)} \right\|^2 \right] \right) \label{eq:niid_tau} \\
    & \leq \frac{\tau (\tau - 1)}{2} \left( \sum_{k=1}^{K} \mathop{\mathbb{E}} \left[ \left\| \mathbf{f}_{(j,k)} \right\|^2 \right] \right) \nonumber \\
    & = \frac{\tau (\tau - 1)}{2m} \left( \sum_{k=1}^{K} \sum_{i = 1}^{m} \mathop{\mathbb{E}} \left[ \left\| \nabla_j F_i(\mathbf{x}_{k}^{i}) \right\|^2 \right] \right), \label{eq:niid_t4_bound}
\end{align}
where (\ref{eq:niid_tau}) holds because $0 < j \leq \tau$.
Here, we finish bounding $T_4$.

By plugging in (\ref{eq:niid_t3_bound}) and (\ref{eq:niid_t4_bound}) into (\ref{eq:niid_t3t4}), we have
\begin{align}
    & \frac{1}{mK} \sum_{k = 1}^{K} \sum_{i=1}^{m} \mathop{\mathbb{E}} \left[ \left\| \mathbf{u}_{(j,k)} - \mathbf{x}_{(j,k)}^{i} \right\|^2 \right] \nonumber \\
    & \quad\quad \leq \frac{2 \eta^2}{K} \left( K\frac{(\tau - 1)}{2} \sigma_j^2 + \frac{\tau (\tau - 1)}{2m} \left( \sum_{k=1}^{K} \sum_{i = 1}^{m} \mathop{\mathbb{E}} \left[ \left\| \nabla_j F_i(\mathbf{x}_{k}^i) \right\|^2 \right] \right) \right) \nonumber \\
    & \quad\quad = \eta^2 (\tau - 1) \sigma_j^2 + \frac{\eta^2 \tau (\tau - 1)}{mK} \left( \sum_{k=1}^{K} \sum_{i = 1}^{m} \mathop{\mathbb{E}} \left[ \left\| \nabla_j F_i(\mathbf{x}_{k}^{i}) \right\|^2 \right] \right) \label{eq:niid_before}
\end{align}

The local gradient term on the right-hand side in (\ref{eq:niid_before}) can be rewritten using the following inequality.
\begin{align}
    \mathop{\mathbb{E}} \left[ \left\| \nabla_j F_i(\mathbf{x}_{k}^{i}) \right\|^2 \right] & = \mathop{\mathbb{E}} \left[ \left\| \nabla_j F_i(\mathbf{x}_{k}^{i}) - \nabla_j F_i(\mathbf{u}_{k}) + \nabla_j F_i(\mathbf{u}_{k}) \right\|^2 \right] \nonumber \\
    & \leq 2 \mathop{\mathbb{E}} \left[ \left\| \nabla_j F_i(\mathbf{x}_{k}^{i}) - \nabla_j F_i(\mathbf{u}_{k}) \right\|^2 \right] + 2 \mathop{\mathbb{E}} \left[ \left\| \nabla_j F_i(\mathbf{u}_{k}) \right\|^2 \right] \label{eq:niid_localgrad_jensen} \\
    & \leq 2 L_j^2 \mathop{\mathbb{E}} \left[ \left\| \mathbf{u}_{(j,k)} - \mathbf{x}_{(j,k)}^{i} \right\|^2 \right] + 2 \mathop{\mathbb{E}} \left[ \left\| \nabla_j F_i(\mathbf{u}_{k}) \right\|^2 \right], \label{eq:niid_localgrad}
\end{align}
where (\ref{eq:niid_localgrad_jensen}) holds based on the convexity of $\ell_2$ norm and Jensen's inequality.

Plugging in (\ref{eq:niid_localgrad}) into (\ref{eq:niid_before}), we have
\begin{align}
    & \frac{1}{mK} \sum_{k = 1}^{K} \sum_{i=1}^{m} \mathop{\mathbb{E}} \left[ \left\| \mathbf{u}_{(j,k)} - \mathbf{x}_{(j,k)}^{i} \right\|^2 \right] \nonumber \\
    & \quad\quad \leq \eta^2 (\tau - 1) \sigma_j^2 + \frac{2 \eta^2 \tau (\tau - 1) L_j^2}{mK} \sum_{k=1}^{K} \sum_{i = 1}^{m} \mathop{\mathbb{E}} \left[ \left\| \mathbf{u}_{(j,k)} - \mathbf{x}_{(j,k)}^{i} \right\|^2 \right] \nonumber \\
    & \quad\quad\quad\quad + \frac{2 \eta^2 \tau (\tau - 1)}{mK} \sum_{k=1}^{K} \sum_{i = 1}^{m} \mathop{\mathbb{E}} \left[ \left\| \nabla_j F_i(\mathbf{u}_{k}) \right\|^2 \right]
\end{align}
After a minor rearranging, we have
\begin{align}
    & \frac{1}{mK} \sum_{k = 1}^{K} \sum_{i=1}^{m} \mathop{\mathbb{E}} \left[ \left\| \mathbf{u}_{(j,k)} - \mathbf{x}_{(j,k)}^{i} \right\|^2 \right] \nonumber \\
    & \quad\quad \leq \frac{\eta^2 (\tau - 1) \sigma_j^2}{1 - 2\eta^2 \tau (\tau - 1) L_j^2} + \frac{2 \eta^2 \tau (\tau - 1)}{mK (1 - 2\eta^2 \tau (\tau - 1) L_j^2)} \sum_{k=1}^{K} \sum_{i = 1}^{m} \mathop{\mathbb{E}} \left[ \left\| \nabla_j F_i(\mathbf{u}_{k}) \right\|^2 \right] \label{eq:niid_rearrange}
\end{align}
Let us define $A_j = 2\eta^2 \tau (\tau - 1) L_j^2$. Then (\ref{eq:niid_rearrange}) is simplified as follows.
\begin{align}
    & \frac{1}{mK} \sum_{k = 1}^{K} \sum_{i=1}^{m} \mathop{\mathbb{E}} \left[ \left\| \mathbf{u}_{(j,k)} - \mathbf{x}_{(j,k)}^{i} \right\|^2 \right] \nonumber \\
    & \quad\quad \leq \frac{\eta^2 (\tau - 1) \sigma_j^2}{1 - A_j} + \frac{A_j}{mK L_j^2 (1 - A_j)} \sum_{k=1}^{K} \sum_{i = 1}^{m} \mathop{\mathbb{E}} \left[ \left\| \nabla_j F_i(\mathbf{u}_{k}) \right\|^2 \right] \nonumber
\end{align}
Based on Assumption 4, we have
\begin{align}
    & \frac{1}{mK} \sum_{k = 1}^{K} \sum_{i=1}^{m} \mathop{\mathbb{E}} \left[ \left\| \mathbf{u}_{(j,k)} - \mathbf{x}_{(j,k)}^{i} \right\|^2 \right] \nonumber \\
    & \quad\quad \leq \frac{\eta^2 (\tau - 1) \sigma_j^2}{1 - A_j} + \frac{A_j \beta_j^2}{K L_j^2 (1 - A_j)} \sum_{k=1}^{K} \mathop{\mathbb{E}} \left[ \left\| \nabla_j F(\mathbf{u}_{k}) \right\|^2 \right] + \frac{A_j \kappa_j^2}{L_j^2 (1 - A_j)} \nonumber
\end{align}
Here, we complete the proof.
\end{proof}

%% file: supplement_4_exp.tex
\section {Additional Experimental Results}
In this section, we provide detailed experimental settings and additional experimental results that support our proposed algorithm.

\subsection {Datasets and Models}

\textbf{CIFAR-10 and CIFAR-100} -- CIFAR-10 and CIFAR-100 is benchmark image datasets for classification. Both datasets have 50K training samples and 10K validation samples. Each sample is a $32 \times 32$ RGB image. We use ResNet-20 and Wide-ResNet-28-10 for CIFAR-10 and CIFAR-100 classification, respectively.
We apply weight decay using a parameter of $0.0001$ for ResNet-20 and $0.0005$ for Wide-ResNet-28-10. 

\textbf{SVHN} -- SVHN is an image dataset that consists of 73K training samples and 26K test samples. It also has additional 530K training samples. Each sample is a $32 \times 32$ RGB image. We use Wide-ResNet-16-8 for classification experiments.
We apply weight decay using a parameter of $0.0001$. 

\textbf{Fashion-MNIST} -- Fashion-MNIST is an image dataset that has 50K training samples and 10K test samples. Each sample is a $28 \times 28$ gray image. We use VGG-11 for classification experiments.
We apply weight decay using a parameter of $0.0001$. 

\textbf{IMDB review} -- IMDB consists of 50K movie reviews for natural language processing. For IMDB sentiment analysis experiments, we use a LSTM model that consists of one embedding layer followed by one bidirectional LSTM layer of size $256$. We also applied dropout with a probability of $0.3$ to both layers. The maximum number of words in the embedding layer is $10,000$ and the output dimension is $256$.
We do not apply weight decay for LSTM training.

\textbf{Federated Extended MNIST} -- FEMNIST consists of $805,263$ pictures of hand-written digits and characters.
The data is intrinsically heterogeneous such that $3,550$ writers provide different numbers of pictures.
We use a CNN that consists of 2 convolution layers and 2 fully-connected layers.
We provide the reference to the model architecture in the main paper.
When training the model, we use a random $10\%$ of the writer's samples only.

\subsection {Experimental Results}

\subsubsection {Local Model Re-distribution}
When re-distributing the models to a new set of active workers, there are two available design options.
First, the aggregated local models can be fully averaged and then distributed to other active workers.
This option slightly sacrifices the statistical efficiency due to the full averaging while the local data privacy is better protected.
Second, the aggregated local models can be re-distributed to other active workers without averaging.
This option provides a good statistical efficiency while potentially having a privacy issue.
In our experiments, we found that both options outperforms the periodic averaging.
All the performance results reported in the main paper are obtained using the second option.

Here we compare the performance of these two design options in Table \ref{tab:design_options}.
We set $\tau$ to 11 and re-distribute the local models to new active workers after every 110 iterations so that the total communication cost is the same as the periodic averaging setting.
First, interestingly, the second design option provides the better accuracy than the first design option in most of the settings.
This result demonstrates that the full model averaging likely harms the statistical efficiency regardless of the frequency.
Second, both design options consistently outperforms the periodic averaging.
In this work, we simply choose a random subset of workers as the new active workers.
Studying the impact of different device selection schemes on the convergence properties and the accuracy can be an interesting future work.

\subsubsection {Model Partitioning} \label{sec:partition}

When synchronizing a part of model in Algorithm 1, the model can be partitioned in many different ways.
Table \ref{tab:partition_options_iid} and \ref{tab:partition_options_niid} show the CIFAR-10 classification performance comparison between \textit{channel-partition} and \textit{layer-partition} for IID and non-IID data settings, respectively.
We do not see a large difference between the two different partitioning methods.
This result demonstrates that, because every parameter is guaranteed to be averaged after every $\tau$ iterations, the order of synchronizations does not strongly affect the performance.

\subsubsection{Learning Curves of IID Data Experiments}
We present the training loss and validation accuracy curves collected in our experiments.

\textbf{CIFAR-10} --
Figure \ref{fig:cifar10} shows the learning curves of ResNet-20 training on CIFAR-10.
The averaging interval $\tau$ is set to 2, 4, and 8 (\textbf{a}, \textbf{b}, \textbf{c}).
The learning rate is decayed by a factor of 10 after 150 and 225 epochs.
We clearly see that the partial averaging makes the training loss converge faster.
In addition, the partial averaging achieves a higher validation accuracy than the periodic averaging after the same number of training epochs.
The performance gap between the periodic averaging and the partial averaging becomes more significant as $\tau$ increases.

\textbf{CIFAR-100} --
Figure \ref{fig:cifar100} shows the learning curves of WideResNet-28-10 training on CIFAR-100.
The hyper-parameter settings are shown in Table 1.
The averaging interval $\tau$ is set to 2, 4, and 8 (\textbf{a}, \textbf{b}, \textbf{c}).
The learning rate is decayed by a factor of 10 after 120 and 185 epochs.
Overall, the two different averaging methods show a significant performance difference.
When the averaging interval is large (8), both averaging methods show a significant drop of accuracy but the partial averaging still outperforms the periodic averaging.

\begin{table}[t]
\scriptsize
\centering
\caption{
    CIFAR-10 (ResNet20) classification accuracy comparison between two design options: \textbf{(1)}: average the aggregated local models before re-distributing to new active workers, \textbf{(2)}: re-distribute the local models to new active workers without averaging.
}
\begin{tabular}{|c|c|c|c|c|c||c|c|} \hline
dataset & batch size (LR) & workers & avg interval & active ratio & Dir($\alpha$) & design \textbf{(1)} & design \textbf{(2)} \\ \hline \hline
\multirow{9}{*}{\shortstack{CIFAR-10\\(ResNet20)}} & \multirow{8}{*}{32 (0.4)} & \multirow{9}{*}{128} & \multirow{9}{*}{11} & \multirow{3}{*}{$100\%$} & 1 & $90.74\pm 0.1\%$ & \textbf{91.54} $\pm 0.1\%$ \\
\cline{6-8}
 & & & & & 0.5 & $90.53\pm 0.1\%$ & \textbf{91.43}$\pm 0.1\%$ \\
\cline{6-8}
 & & & & & 0.1 & $90.39\pm 0.2\%$ & \textbf{91.08} $\pm 0.1\%$ \\
\cline{5-8}
 & & & & \multirow{3}{*}{$50\%$} & 1 & \textbf{90.69}$\pm 0.1\%$ & 90.64 $\pm 0.2\%$ \\
\cline{6-8}
 & & & & & 0.5 & $90.23\pm 0.2\%$ & \textbf{91.02}$\pm 0.3\%$ \\
\cline{6-8}
 & & & & & 0.1 & $89.89\pm 0.2\%$ & \textbf{90.17}$\pm 0.2\%$ \\
\cline{5-8}
 & & & & \multirow{3}{*}{$25\%$} & 1 & $89.64\pm 0.3\%$ & \textbf{91.00} $\pm 0.2\%$ \\
\cline{6-8}
 & & & & & 0.5 & $89.39\pm 0.3\%$ & \textbf{90.16}$\pm 0.3\%$ \\
\cline{2-2}\cline{6-8}
 & 32 (0.2) & & & & 0.1 & $88.32\pm 0.2\%$ & \textbf{88.95} $\pm 0.3\%$ \\
\hline
\end{tabular}
\label{tab:design_options}
\end{table}

\begin{table}[t]
\scriptsize
\centering
\caption{
    CIFAR-10 classification performance comparison between \textit{channel-partition} and \textit{layer-partition} (IID data).
    We do not see any meaningful difference between the two partitioning methods.
    \vspace{0.2cm}
}
\begin{tabular}{|c|c|c|c|c||c|c|} \hline
dataset & model & \# of workers & epochs & avg interval & channel-partition & layer-partition \\ \hline \hline
\multirow{3}{*}{CIFAR-10} & \multirow{3}{*}{ResNet20} & \multirow{3}{*}{128} & \multirow{3}{*}{300} & 2 & \textbf{91.89}$\pm 0.1\%$ & $91.81\pm 0.1\%$ \\
 &  &  &  & 4 & $90.56\pm 0.2\%$ & \textbf{90.58}$\pm 0.2\%$ \\
 &  &  &  & 8 & \textbf{87.13}$\pm 0.1\%$ & $87.10\pm 0.1\%$ \\ \hline
\end{tabular}
\label{tab:partition_options_iid}
\end{table}

\begin{table}[t]
\scriptsize
\centering
\caption{
    CIFAR-10 (ResNet20) classification accuracy comparison between two model partitioning options: \textbf{(1)}: \textit{channel-partitioning}, \textbf{(2)}: \textit{layer-partitioning}.
    \vspace{0.2cm}
}
\begin{tabular}{|c|c|c|c|c|c||c|c|} \hline
dataset & batch size (LR) & workers & avg interval & active ratio & Dir($\alpha$) & \textit{layer-partitioning} & \textit{channel-partitioning} \\ \hline \hline
\multirow{9}{*}{\shortstack{CIFAR-10\\(ResNet20)}} & \multirow{8}{*}{32 (0.4)} & \multirow{9}{*}{128} & \multirow{9}{*}{11} & \multirow{3}{*}{$100\%$} & 1 & $91.21\pm 0.1\%$ & \textbf{91.54} $\pm 0.1\%$ \\
\cline{6-8}
 & & & & & 0.5 & \textbf{91.56}$\pm 0.2\%$ & $91.43\pm 0.1\%$ \\
\cline{6-8}
 & & & & & 0.1 & \textbf{91.31}$\pm 0.1\%$ & $91.08\pm 0.1\%$ \\
\cline{5-8}
 & & & & \multirow{3}{*}{$50\%$} & 1 & \textbf{90.66}$\pm 0.1\%$ & $90.61\pm 0.2\%$ \\
\cline{6-8}
 & & & & & 0.5 & $90.97\pm 0.2\%$ & \textbf{91.02}$\pm 0.3\%$ \\
\cline{6-8}
 & & & & & 0.1 & $90.09\pm 0.3\%$ & \textbf{90.64}$\pm 0.2\%$ \\
\cline{5-8}
 & & & & \multirow{3}{*}{$25\%$} & 1 & $90.48\pm 0.3\%$ & \textbf{91.00}$\pm 0.2\%$ \\
\cline{6-8}
 & & & & & 0.5 & $89.39\pm 0.3\%$ & \textbf{90.02}$\pm 0.2\%$ \\
\cline{2-2}\cline{6-8}
 & 32 (0.2) & & & & 0.1 & $88.32\pm 0.2\%$ & \textbf{88.92}$\pm 0.1\%$ \\
\hline
\end{tabular}
\label{tab:partition_options_niid}
\end{table}

\textbf{SVHN} --
Figure \ref{fig:svhn} shows the learning curves of WideResNet-16-8 training on SVHN.
The hyper-parameter settings are shown in Table 1.
The averaging interval $\tau$ is set to 4, 16, and 64 (\textbf{a}, \textbf{b}, \textbf{c}).
The learning rate is decayed by a factor of 10 after 80 and 120 epochs.
We could use a relatively longer averaging interval than the other experiments without much losing the performance due to the large number of training samples.
The partial averaging slightly outperforms the periodic averaging in all the settings.

\begin{figure}[t]
\centering
\includegraphics[width=\columnwidth]{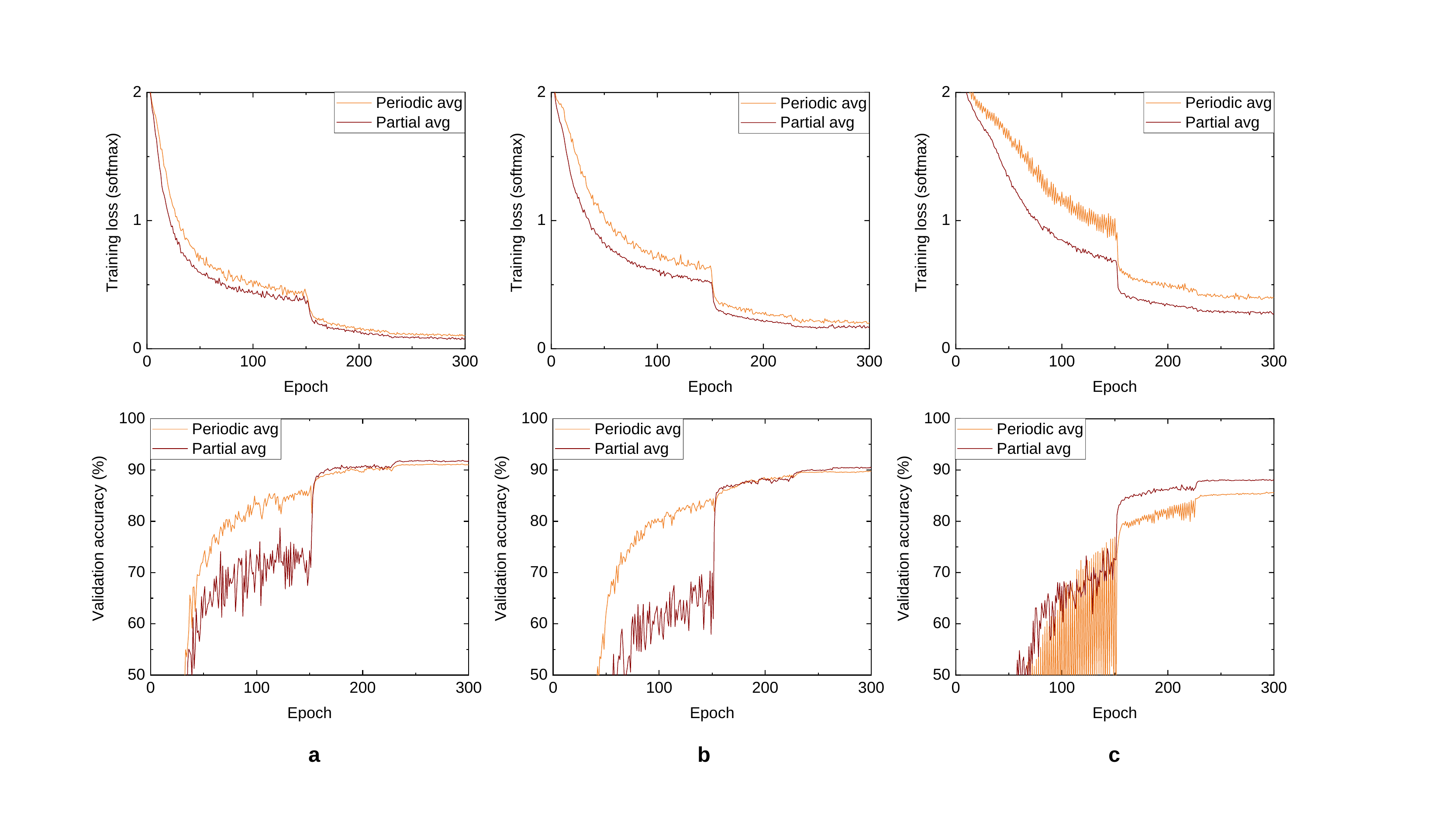}
\caption{
    The learning curves of ResNet-20 (CIFAR-10) training.
    The number of workers is 128 and the hyper-parameters are shown in Table 1.
    The top charts are the training loss and the bottom charts are the validation accuracy.
    The averaging interval $\tau$ is set to 2, 4, and 8 (\textbf{a}, \textbf{b}, and \textbf{c}).
}
\label{fig:cifar10}
\end{figure}

\begin{figure}[t]
\centering
\includegraphics[width=\columnwidth]{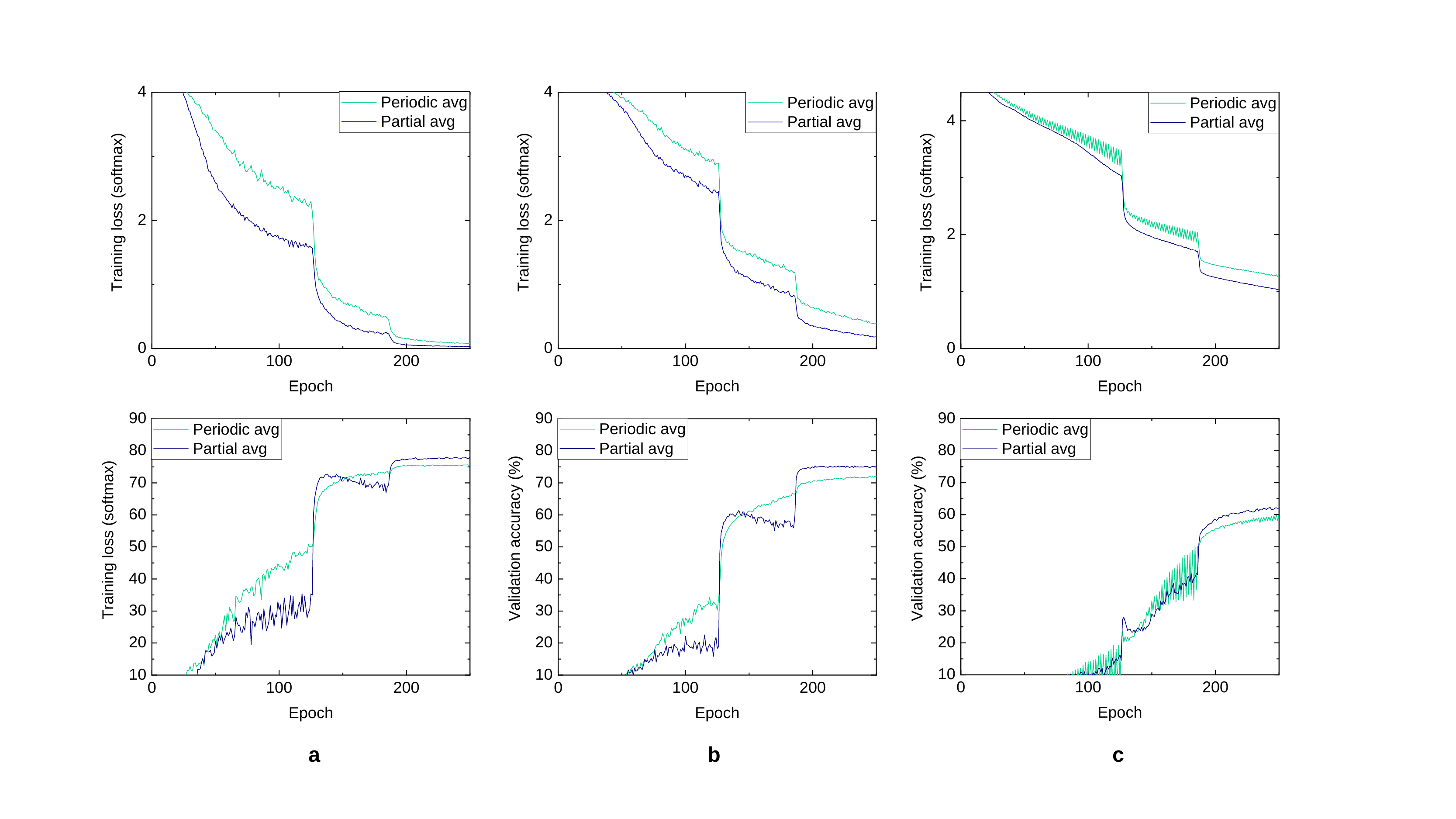}
\caption{
    The learning curves of WideResNet-28-10 (CIFAR-100) training.
    The number of workers is 128 and the hyper-parameters are shown in Table 1.
    The top charts are the training loss and the bottom charts are the validation accuracy.
    The averaging interval $\tau$ is set to 2, 4, and 8 (\textbf{a}, \textbf{b}, and \textbf{c}).
}
\label{fig:cifar100}
\end{figure}

\begin{figure}[t]
\centering
\includegraphics[width=\columnwidth]{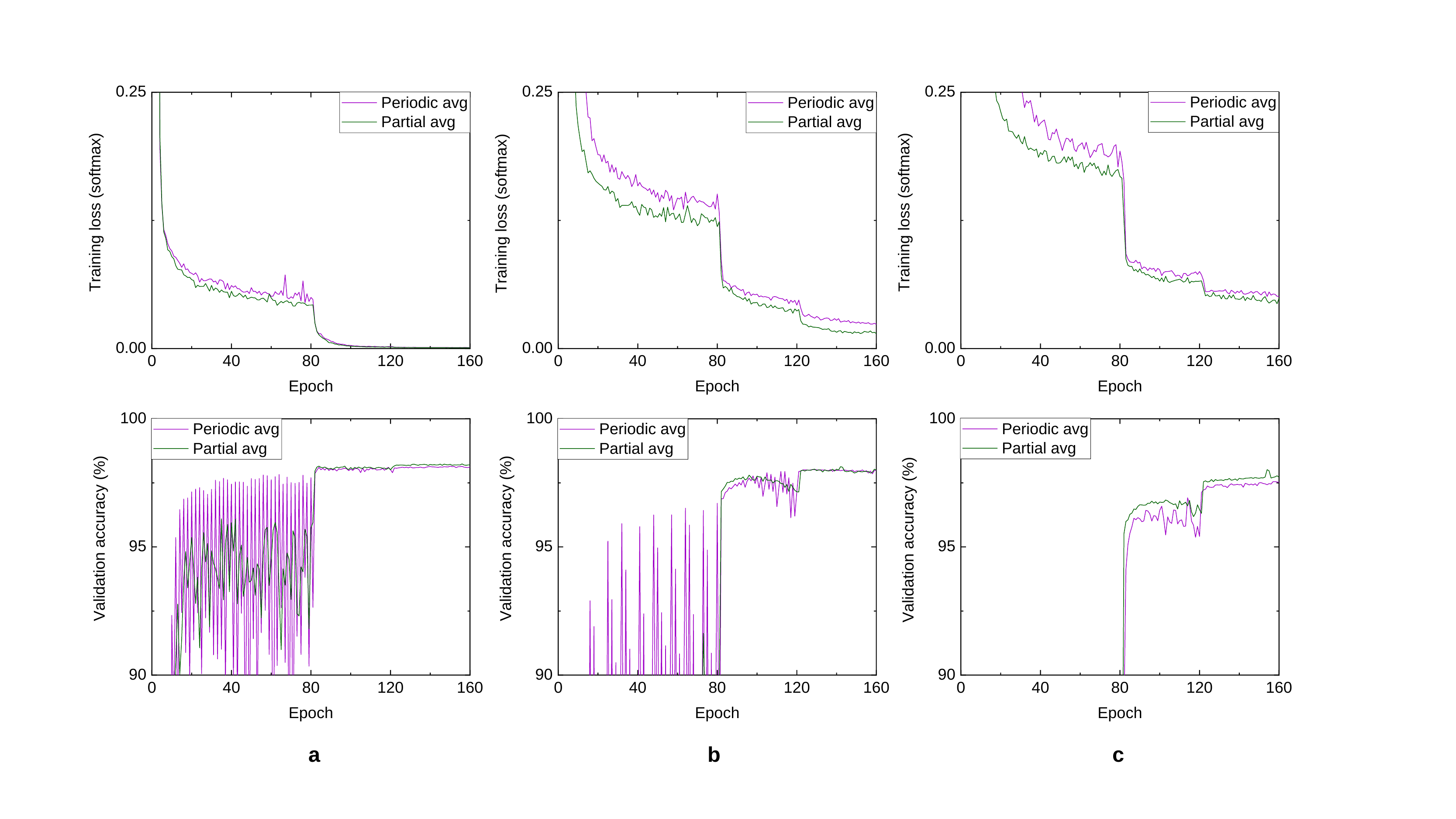}
\caption{
    The learning curves of WideResNet-16-8 (SVHN) training.
    The number of workers is 128 and the hyper-parameters are shown in Table 1.
    The top charts are the training loss and the bottom charts are the validation accuracy.
    The averaging interval $\tau$ is set to 4, 16, and 64 (\textbf{a}, \textbf{b}, and \textbf{c}).
}
\label{fig:svhn}
\end{figure}

\begin{figure}[t]
\centering
\includegraphics[width=\columnwidth]{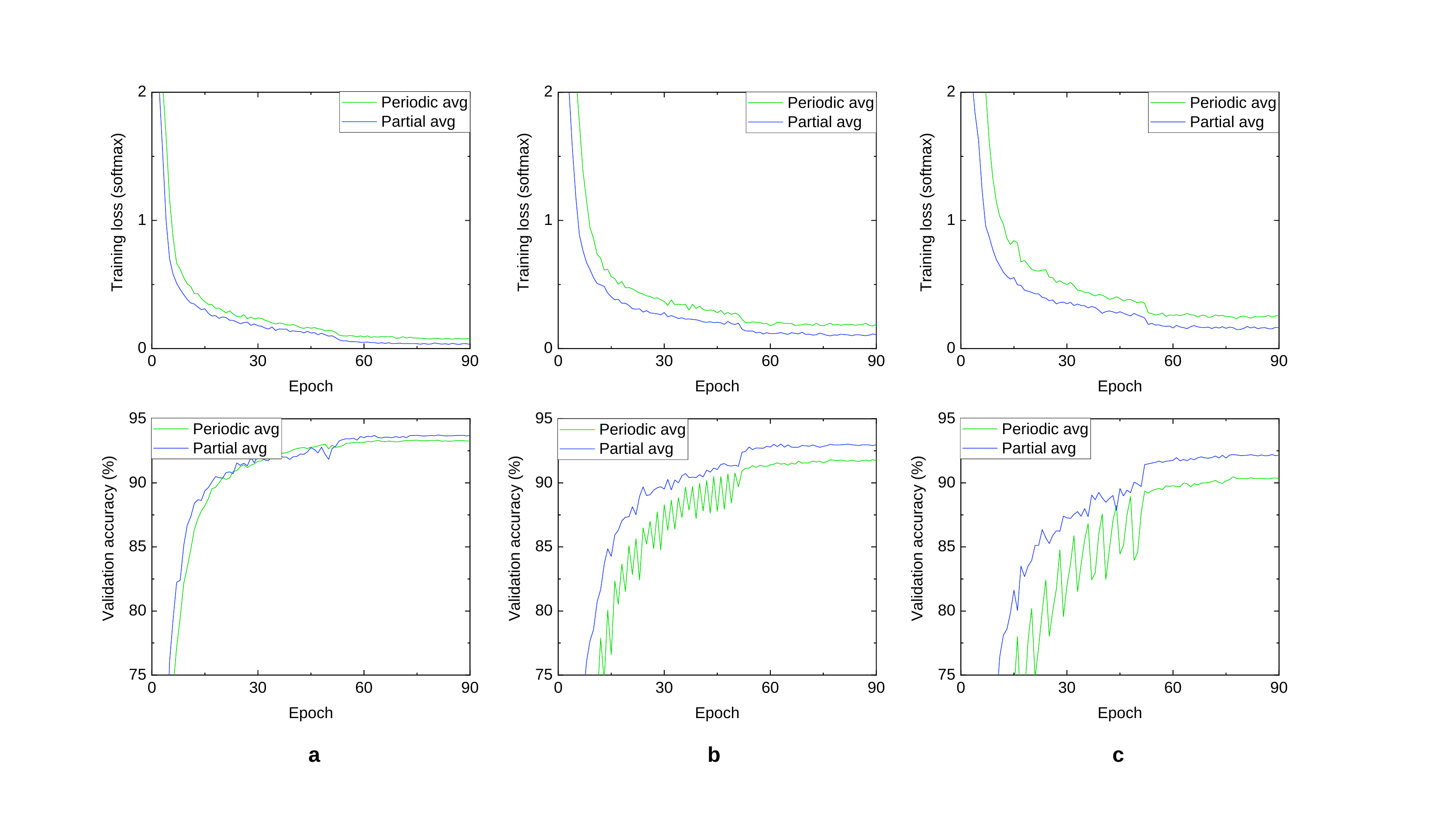}
\caption{
    The learning curves of VGG-11 (Fashion-MNIST) training.
    The number of workers is 128 and the hyper-parameters are shown in Table 1.
    The top charts are the training loss and the bottom charts are the validation accuracy.
    The averaging interval $\tau$ is set to 2, 4, and 8 (\textbf{a}, \textbf{b}, and \textbf{c}).
}
\label{fig:fmnist}
\end{figure}

\textbf{Fashion-MNIST} --
Figure \ref{fig:fmnist} shows the learning curves of VGG-11 training on Fashion-MNIST.
The hyper-parameter settings are shown in Table 1.
The averaging interval $\tau$ is set to 2, 4, and 8 (\textbf{a}, \textbf{b}, \textbf{c}).
The learning rate is decayed by a factor of 10 after 50 and 75 epochs.
The partial averaging consistently outperforms the periodic averaging for all the different averaging interval settings.

\textbf{IMDB reviews} --
Figure \ref{fig:imdb} shows the learning curves of LSTM training on IMDB.
The hyper-parameter settings are shown in Table 1.
The averaging interval $\tau$ is set to 2, 4, and 8 (\textbf{a}, \textbf{b}, \textbf{c}).
The learning rate is decayed by a factor of 10 after 60 and 80 epochs.
Although the final accuracy is not significantly different between the two averaging methods, the partial averaging accuracy is still consistently higher than that of the periodic averaging.

\begin{figure}[t]
\centering
\includegraphics[width=\columnwidth]{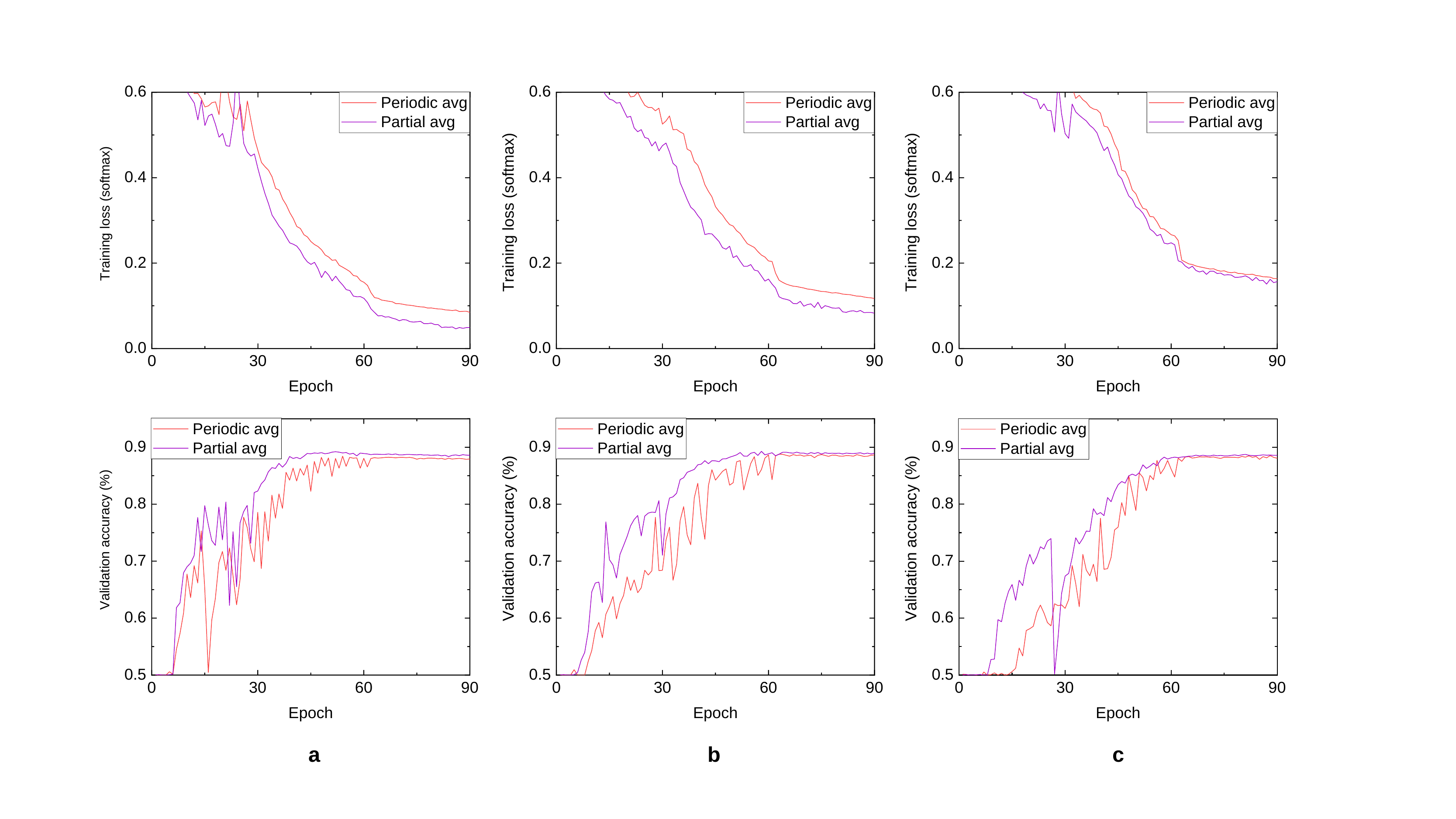}
\caption{
    The learning curves of LSTM (IMDB) training.
    The number of workers is 128 and the hyper-parameters are shown in Table 1.
    The top charts are the training loss and the bottom charts are the validation accuracy.
    The averaging interval $\tau$ is set to 2, 4, and 8 (\textbf{a}, \textbf{b}, and \textbf{c}).
}
\label{fig:imdb}
\end{figure}

\subsubsection {Learning Curves of non-IID Data Experiments}
Here, we present the learning curves for non-IID data experiments.
We summarize two key observations on the learning curves as follows.
First, the partial averaging provides smoother training loss curves than the periodic full averaging as well as a faster convergence.
Especially when the averaging interval is large ($\tau=8$), the periodic averaging curves fluctuate significantly while the partial averaging curves are stable.
Second, the validation curves show noticeable differences.
The partial averaging shows a sharp increase of validation curves when the learning rate is decayed.
It has been known that the high degree of noise in the early training can improve the generalization performance.
This pattern of validation curves is well aligned with the presented final accuracy.

\textbf{CIFAR-10} --
Figure \ref{fig:cifar10_noniid} shows the learning curves of ResNet20 training on CIFAR-10 under more realistic Federated Learning settings.
\textbf{a, b, c}: Training loss curves with activation ratio of $25\%$, $50\%$, and $100\%$, respectively.
\textbf{d, e, f}: Validation accuracy curves with activation ratio of $25\%$, $50\%$, and $100\%$, respectively.
The three columns correspond to Dirichlet's concentration parameters of $0.1$, $0.5$, and $1.0$, respectively.
The learning rate is decayed by a factor of 10 after 5000 and 7500 iterations.
We see that the partial averaging shows a faster convergence of training loss as well as a higher accuracy in all the experiment.

\begin{figure}[t]
\centering
\includegraphics[width=0.9\columnwidth]{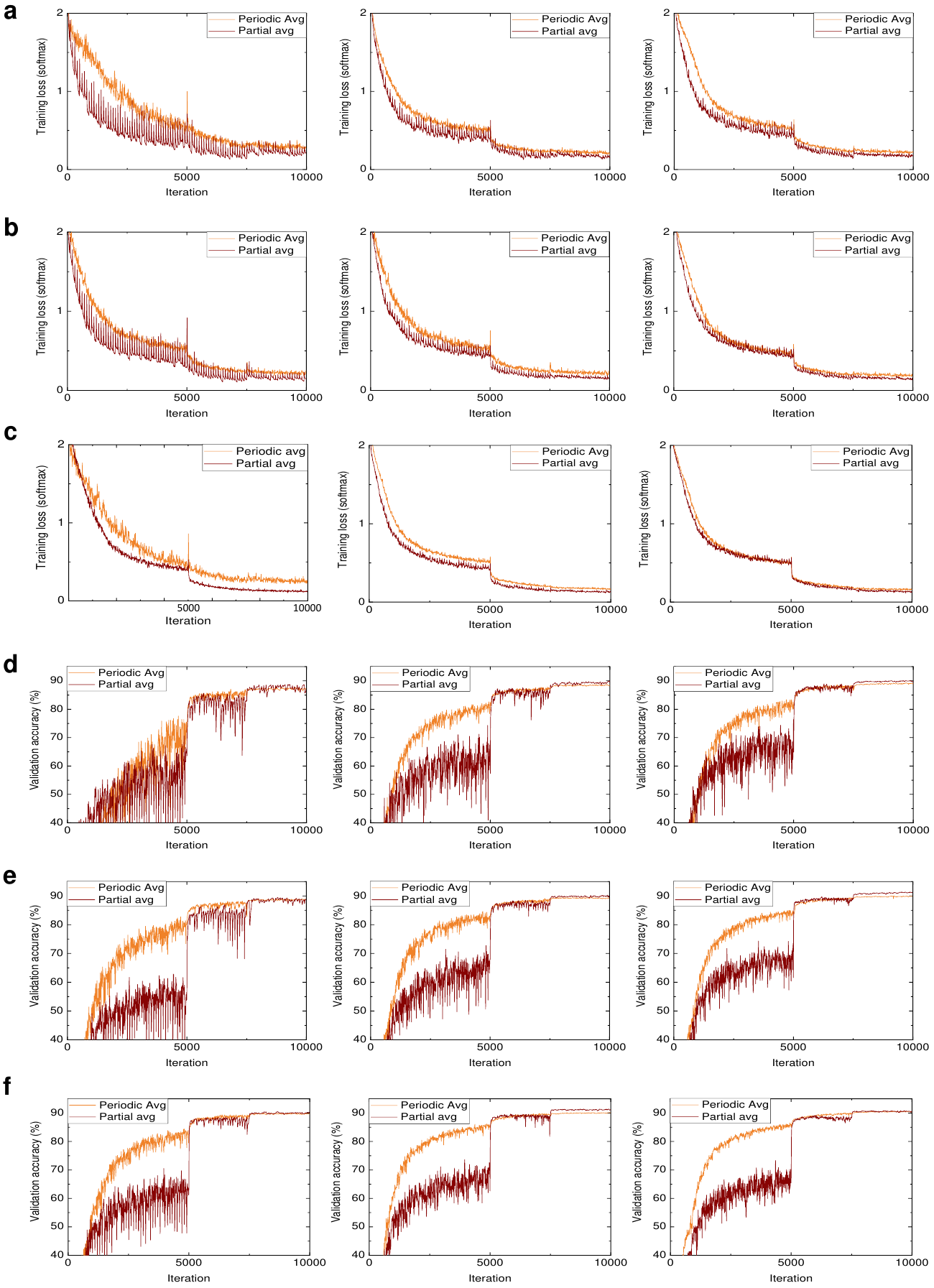}
\caption{
    The learning curves of CIFAR-10 with various degrees of data heterogeneity and ratios of the active workers. \textbf{a, b, c}: Training loss curves with activation ratio of $25\%$, $50\%$, and $100\%$, respectively. \textbf{d, e, f}: Validation accuracy curves with activation ratio of $25\%$, $50\%$, and $100\%$, respectively. The three columns correspond to Dirichlet's concentration parameters of $0.1$, $0.5$, and $1.0$, respectively.
}
\label{fig:cifar10_noniid}
\end{figure}

\textbf{IMDB} --
Figure \ref{fig:imdb_noniid} shows the learning curves of LSTM training on IMDB under more realistic Federated Learning settings.
\textbf{a, b, c}: Training loss curves with activation ratio of $25\%$, $50\%$, and $100\%$, respectively.
\textbf{d, e, f}: Validation accuracy curves with activation ratio of $25\%$, $50\%$, and $100\%$, respectively.
The two columns correspond to Dirichlet's concentration parameters of $0.5$ and $1.0$, respectively.
The learning rate is decayed by a factor of 10 after 1500 and 1800 iterations.
Likely to CIFAR-10, the partial averaging shows superior classification performance than the periodic averaging.
The performance gap is even larger than that of the same IMDB sentiment analysis with IID settings.

\textbf{FEMNIST} --
Figure \ref{fig:femnist} shows the learning curves of CNN training on FEMNIST.
Because the data distribution is already heterogeneous across the workers, we adjust the ratio of device activation.
\textbf{a, b, c} show the learning curves with $25\%$, $50\%$, and $100\%$ activation ratios, respectively.
The partial model averaging achieves the higher accuracy in all the settings.
Especially, the training loss curves show a significant difference between the two model averaging methods.

\begin{figure}[t]
\centering
\includegraphics[width=0.7\columnwidth]{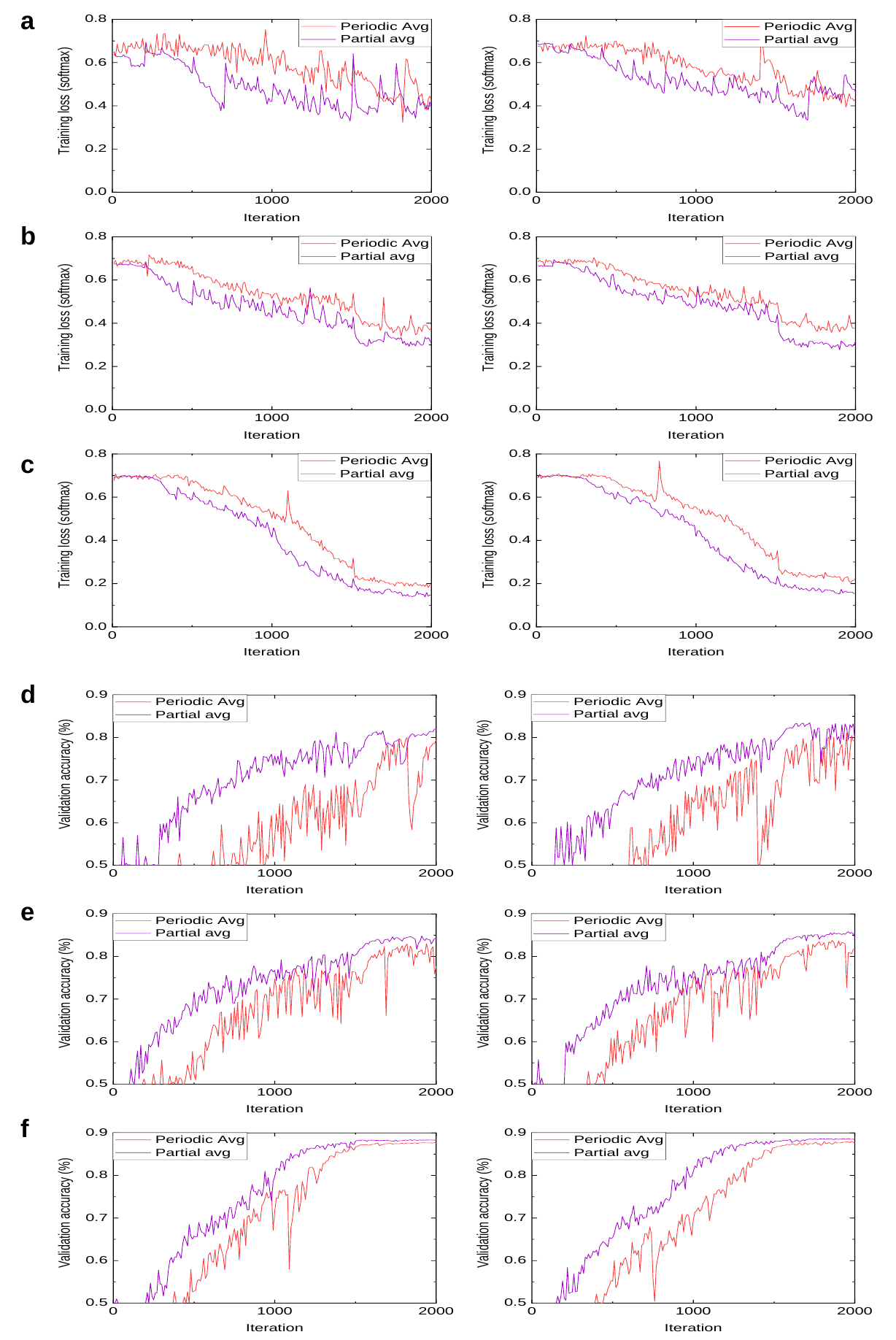}
\caption{
    The learning curves of IMDB with various degrees of data heterogeneity and ratios of the active workers. \textbf{a, b, c}: Training loss curves with activation ratio of $25\%$, $50\%$, and $100\%$, respectively. \textbf{d, e, f}: Validation accuracy curves with activation ratio of $25\%$, $50\%$, and $100\%$, respectively. The two columns correspond to Dirichlet's concentration parameters of $0.5$ and $1.0$, respectively.
}
\label{fig:imdb_noniid}
\end{figure}

\begin{figure}[t]
\centering
\includegraphics[width=0.7\columnwidth]{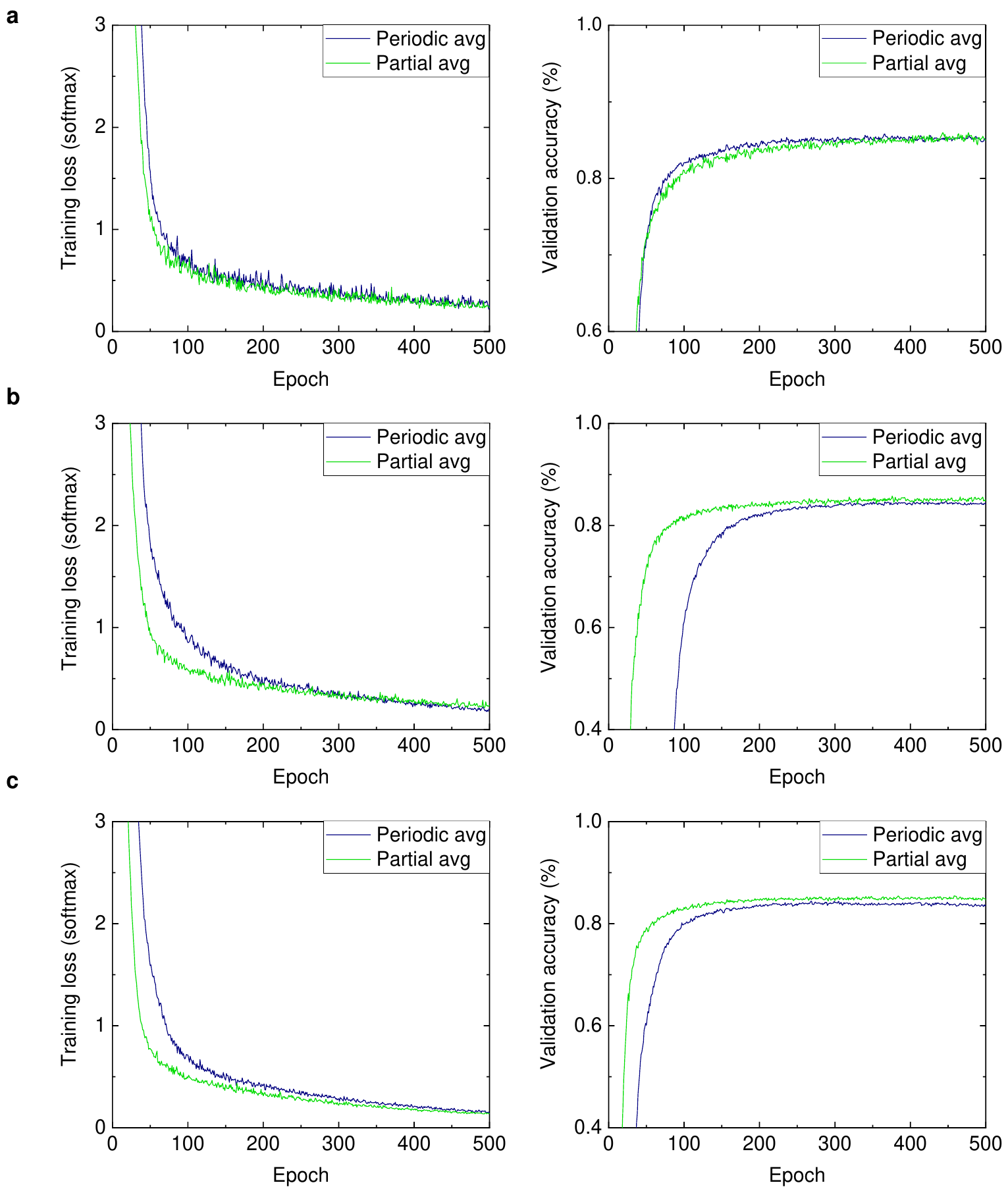}
\caption{
    The learning curves of FEMNIST with different ratios of the active workers.
    \textbf{a, b, c}: learning curves with $25\%$, $50\%$, and $100\%$ of random devices, respectively.
    The left charts are the training loss curves and the right charts are the validation accuracy curves.
}
\label{fig:femnist}
\end{figure}

%% file: main.bbl
\begin{thebibliography}{26}
\providecommand{\natexlab}[1]{#1}

\bibitem[{Abadi et~al.(2015)Abadi, Agarwal, Barham, Brevdo, Chen, Citro,
  Corrado, Davis, Dean, Devin, Ghemawat, Goodfellow, Harp, Irving, Isard, Jia,
  Jozefowicz, Kaiser, Kudlur, Levenberg, Man\'{e}, Monga, Moore, Murray, Olah,
  Schuster, Shlens, Steiner, Sutskever, Talwar, Tucker, Vanhoucke, Vasudevan,
  Vi\'{e}gas, Vinyals, Warden, Wattenberg, Wicke, Yu, and
  Zheng}]{tensorflow2015-whitepaper}
Abadi, M.; Agarwal, A.; Barham, P.; Brevdo, E.; Chen, Z.; Citro, C.; Corrado,
  G.~S.; Davis, A.; Dean, J.; Devin, M.; Ghemawat, S.; Goodfellow, I.; Harp,
  A.; Irving, G.; Isard, M.; Jia, Y.; Jozefowicz, R.; Kaiser, L.; Kudlur, M.;
  Levenberg, J.; Man\'{e}, D.; Monga, R.; Moore, S.; Murray, D.; Olah, C.;
  Schuster, M.; Shlens, J.; Steiner, B.; Sutskever, I.; Talwar, K.; Tucker, P.;
  Vanhoucke, V.; Vasudevan, V.; Vi\'{e}gas, F.; Vinyals, O.; Warden, P.;
  Wattenberg, M.; Wicke, M.; Yu, Y.; and Zheng, X. 2015.
\newblock {TensorFlow}: Large-Scale Machine Learning on Heterogeneous Systems.
\newblock Software available from tensorflow.org.

\bibitem[{Caldas et~al.(2018)Caldas, Duddu, Wu, Li, Kone{\v{c}}n{\`y}, McMahan,
  Smith, and Talwalkar}]{caldas2018leaf}
Caldas, S.; Duddu, S. M.~K.; Wu, P.; Li, T.; Kone{\v{c}}n{\`y}, J.; McMahan,
  H.~B.; Smith, V.; and Talwalkar, A. 2018.
\newblock Leaf: A benchmark for federated settings.
\newblock \emph{arXiv preprint arXiv:1812.01097}.

\bibitem[{Defazio, Bach, and Lacoste-Julien(2014)}]{defazio2014saga}
Defazio, A.; Bach, F.; and Lacoste-Julien, S. 2014.
\newblock SAGA: A fast incremental gradient method with support for
  non-strongly convex composite objectives.
\newblock \emph{arXiv preprint arXiv:1407.0202}.

\bibitem[{Defazio and Bottou(2018)}]{defazio2018ineffectiveness}
Defazio, A.; and Bottou, L. 2018.
\newblock On the ineffectiveness of variance reduced optimization for deep
  learning.
\newblock \emph{arXiv preprint arXiv:1812.04529}.

\bibitem[{Goyal et~al.(2017)Goyal, Doll{\'a}r, Girshick, Noordhuis, Wesolowski,
  Kyrola, Tulloch, Jia, and He}]{goyal2017accurate}
Goyal, P.; Doll{\'a}r, P.; Girshick, R.; Noordhuis, P.; Wesolowski, L.; Kyrola,
  A.; Tulloch, A.; Jia, Y.; and He, K. 2017.
\newblock Accurate, large minibatch sgd: Training imagenet in 1 hour.
\newblock \emph{arXiv preprint arXiv:1706.02677}.

\bibitem[{Haddadpour et~al.(2019)Haddadpour, Kamani, Mahdavi, and
  Cadambe}]{haddadpour2019local}
Haddadpour, F.; Kamani, M.~M.; Mahdavi, M.; and Cadambe, V.~R. 2019.
\newblock Local sgd with periodic averaging: Tighter analysis and adaptive
  synchronization.
\newblock \emph{arXiv preprint arXiv:1910.13598}.

\bibitem[{Johnson and Zhang(2013)}]{johnson2013accelerating}
Johnson, R.; and Zhang, T. 2013.
\newblock Accelerating stochastic gradient descent using predictive variance
  reduction.
\newblock \emph{Advances in neural information processing systems}, 26:
  315--323.

\bibitem[{Karimireddy et~al.(2020)Karimireddy, Kale, Mohri, Reddi, Stich, and
  Suresh}]{karimireddy2020scaffold}
Karimireddy, S.~P.; Kale, S.; Mohri, M.; Reddi, S.; Stich, S.; and Suresh,
  A.~T. 2020.
\newblock SCAFFOLD: Stochastic controlled averaging for federated learning.
\newblock In \emph{International Conference on Machine Learning}, 5132--5143.
  PMLR.

\bibitem[{Krizhevsky, Hinton et~al.(2009)}]{krizhevsky2009learning}
Krizhevsky, A.; Hinton, G.; et~al. 2009.
\newblock Learning multiple layers of features from tiny images.

\bibitem[{Lewkowycz et~al.(2020)Lewkowycz, Bahri, Dyer, Sohl-Dickstein, and
  Gur-Ari}]{lewkowycz2020large}
Lewkowycz, A.; Bahri, Y.; Dyer, E.; Sohl-Dickstein, J.; and Gur-Ari, G. 2020.
\newblock The large learning rate phase of deep learning: the catapult
  mechanism.
\newblock \emph{arXiv preprint arXiv:2003.02218}.

\bibitem[{Li et~al.(2020)Li, Sahu, Zaheer, Sanjabi, Talwalkar, and
  Smith}]{li2018federated}
Li, T.; Sahu, A.~K.; Zaheer, M.; Sanjabi, M.; Talwalkar, A.; and Smith, V.
  2020.
\newblock Federated Optimization in Heterogeneous Networks.
\newblock In \emph{Proceedings of Machine Learning and Systems}, volume~2,
  429--450.

\bibitem[{Li, Wei, and Ma(2019)}]{li2019towards}
Li, Y.; Wei, C.; and Ma, T. 2019.
\newblock Towards explaining the regularization effect of initial large
  learning rate in training neural networks.
\newblock \emph{arXiv preprint arXiv:1907.04595}.

\bibitem[{Liang et~al.(2019)Liang, Shen, Liu, Pan, Chen, and
  Cheng}]{liang2019variance}
Liang, X.; Shen, S.; Liu, J.; Pan, Z.; Chen, E.; and Cheng, Y. 2019.
\newblock Variance reduced local SGD with lower communication complexity.
\newblock \emph{arXiv preprint arXiv:1912.12844}.

\bibitem[{Lin et~al.(2018)Lin, Stich, Patel, and Jaggi}]{lin2018don}
Lin, T.; Stich, S.~U.; Patel, K.~K.; and Jaggi, M. 2018.
\newblock Don't Use Large Mini-Batches, Use Local SGD.
\newblock \emph{arXiv preprint arXiv:1808.07217}.

\bibitem[{Maas et~al.(2011)Maas, Daly, Pham, Huang, Ng, and
  Potts}]{maas-EtAl:2011:ACL-HLT2011}
Maas, A.~L.; Daly, R.~E.; Pham, P.~T.; Huang, D.; Ng, A.~Y.; and Potts, C.
  2011.
\newblock Learning Word Vectors for Sentiment Analysis.
\newblock In \emph{Proceedings of the 49th Annual Meeting of the Association
  for Computational Linguistics: Human Language Technologies}, 142--150.
  Portland, Oregon, USA: Association for Computational Linguistics.

\bibitem[{McMahan et~al.(2017)McMahan, Moore, Ramage, Hampson, and
  y~Arcas}]{mcmahan2017communication}
McMahan, B.; Moore, E.; Ramage, D.; Hampson, S.; and y~Arcas, B.~A. 2017.
\newblock Communication-efficient learning of deep networks from decentralized
  data.
\newblock In \emph{Artificial Intelligence and Statistics}, 1273--1282. PMLR.

\bibitem[{Netzer et~al.(2011)Netzer, Wang, Coates, Bissacco, Wu, and
  Ng}]{netzer2011reading}
Netzer, Y.; Wang, T.; Coates, A.; Bissacco, A.; Wu, B.; and Ng, A.~Y. 2011.
\newblock Reading digits in natural images with unsupervised feature learning.

\bibitem[{Robbins and Monro(1951)}]{robbins1951stochastic}
Robbins, H.; and Monro, S. 1951.
\newblock A stochastic approximation method.
\newblock \emph{The annals of mathematical statistics}, 400--407.

\bibitem[{Stich(2018)}]{stich2018local}
Stich, S.~U. 2018.
\newblock Local SGD converges fast and communicates little.
\newblock \emph{arXiv preprint arXiv:1805.09767}.

\bibitem[{Wang and Joshi(2018{\natexlab{a}})}]{wang2018adaptive}
Wang, J.; and Joshi, G. 2018{\natexlab{a}}.
\newblock Adaptive communication strategies to achieve the best error-runtime
  trade-off in local-update SGD.
\newblock \emph{arXiv preprint arXiv:1810.08313}.

\bibitem[{Wang and Joshi(2018{\natexlab{b}})}]{wang2018cooperative}
Wang, J.; and Joshi, G. 2018{\natexlab{b}}.
\newblock Cooperative SGD: A unified framework for the design and analysis of
  communication-efficient SGD algorithms.
\newblock \emph{arXiv preprint arXiv:1808.07576}.

\bibitem[{Wang et~al.(2020)Wang, Liu, Liang, Joshi, and
  Poor}]{wang2020tackling}
Wang, J.; Liu, Q.; Liang, H.; Joshi, G.; and Poor, H.~V. 2020.
\newblock Tackling the objective inconsistency problem in heterogeneous
  federated optimization.
\newblock \emph{arXiv preprint arXiv:2007.07481}.

\bibitem[{Xiao, Rasul, and Vollgraf(2017)}]{xiao2017fashion}
Xiao, H.; Rasul, K.; and Vollgraf, R. 2017.
\newblock Fashion-mnist: a novel image dataset for benchmarking machine
  learning algorithms.
\newblock \emph{arXiv preprint arXiv:1708.07747}.

\bibitem[{You et~al.(2019)You, Li, Reddi, Hseu, Kumar, Bhojanapalli, Song,
  Demmel, Keutzer, and Hsieh}]{you2019large}
You, Y.; Li, J.; Reddi, S.; Hseu, J.; Kumar, S.; Bhojanapalli, S.; Song, X.;
  Demmel, J.; Keutzer, K.; and Hsieh, C.-J. 2019.
\newblock Large batch optimization for deep learning: Training bert in 76
  minutes.
\newblock \emph{arXiv preprint arXiv:1904.00962}.

\bibitem[{Yu, Jin, and Yang(2019)}]{yu2019linear}
Yu, H.; Jin, R.; and Yang, S. 2019.
\newblock On the linear speedup analysis of communication efficient momentum
  SGD for distributed non-convex optimization.
\newblock In \emph{International Conference on Machine Learning}, 7184--7193.
  PMLR.

\bibitem[{Yu, Yang, and Zhu(2019)}]{yu2019parallel}
Yu, H.; Yang, S.; and Zhu, S. 2019.
\newblock Parallel restarted SGD with faster convergence and less
  communication: Demystifying why model averaging works for deep learning.
\newblock In \emph{Proceedings of the AAAI Conference on Artificial
  Intelligence}, volume~33, 5693--5700.

\end{thebibliography}
